\documentclass[a4paper,fleqn]{cas-sc}

\usepackage[authoryear,longnamesfirst]{natbib}
\usepackage{relsize}

\usepackage{xfrac}
\usepackage{amsfonts,amssymb}
\usepackage{bm}
\usepackage{caption}
\usepackage{subcaption}
\usepackage[ruled,noend,vlined]{algorithm2e} 
\usepackage{tikz,tikz-3dplot}

\newdefinition{definition}{Definition}
\newdefinition{example}{Example}
\newtheorem{lemma}{Lemma}
\newtheorem{assumption}{Assumption}
\newtheorem{theorem}{Theorem}
\newtheorem{corollary}{Corollary}
\newproof{proof}{Proof}

\def\tsc#1{\csdef{#1}{\textsc{\lowercase{#1}}\xspace}}
\tsc{WGM}
\tsc{QE}
\tsc{EP}
\tsc{PMS}
\tsc{BEC}
\tsc{DE}

\DeclareMathOperator*{\argmin}{argmin}

\renewcommand{\vec}[1]{\bm{#1}}

\newcommand{\ClassNP}{\ensuremath{\textsf{NP}}\xspace}

\newcommand{\ClassPP}{\ensuremath{\textsf{PP}}\xspace}

\newcommand{\ClassNPPP}{\ensuremath{\smash{\textsf{NP}^\textsf{PP}}\xspace}}
\newcommand{\ClassSigmaTwoP}{\ensuremath{\smash{\mathsf{\Sigma}_2^{\textsf{p}}}}\xspace}
\newcommand{\ClassWTwo}{\ensuremath{\mathsf{W}[2]}\xspace}

\newcommand{\MIPRelax}{MIP$_{\text{relax}}$}

\begin{document}
\let\WriteBookmarks\relax
\def\floatpagepagefraction{1}
\def\textpagefraction{.001}
\shorttitle{Probabilistic Linear Explanations}
\shortauthors{F. Koriche et~al.}

\title[mode = title]{Probabilistic Linear Explanations}                      
\tnotemark[1]
\tnotetext[1]{This manuscript is an extended version of a paper presented at the 41st Conference on Uncertainty in Artificial Intelligence (UAI 2025), which received a Best Paper Award.}

\author[1]{Frédéric Koriche}[]
\cormark[1]
\fnmark[1]
\ead{koriche@cril.fr}

\affiliation[1]{
  organization={Computer science Research Institute of Lens (CRIL), UMR CNRS 8188, University of Artois},
  addressline={Rue Jean Souvraz SP 18}, 
  city={Lens},
  postcode={F-62307}, 
  country={France}
}

\author[1]{Jean-Marie Lagniez}[]
\fnmark[2]
\ead{lagniez@cril.fr}

\author[1]{Chi Tran}[]
\fnmark[2]
\ead{tran@cril.fr}

\cortext[cor1]{Corresponding author}

\begin{abstract}
Formal explainability provides mathematically grounded justifications for individual predictions. However, abductive explanations often exceed human cognitive limits by involving too many features, while probabilistic relaxations have remained largely limited to categorical classification. We present a unified framework for probabilistic explainability based on sparse, anchored linear models, applicable to both binary classification and continuous regression. By mapping instances to the Boolean hypercube, our linear explanations strictly generalize subset-based approaches: they capture both the magnitude and direction of feature contributions while enforcing a prescribed sparsity budget $k$. We show that minimizing the relevance error for such explanations is \ClassNPPP-hard when the underlying model is a neural network, and we relate this intractable objective to a tractable surrogate---the fidelity error. For a parameterized family of local distributions, the relevance error of any $k$-sparse explanation is bounded by its fidelity error up to a multiplicative factor that remains small locally. We address the resulting empirical problem using two complementary approaches: a Mixed Integer Programming (MIP) formulation that yields provably optimal empirical solutions while maintaining polynomial sample complexity, and a polynomial-time Iterative Hard Thresholding (IHT) algorithm with provable approximation guarantees. Empirical evaluations show that, unlike state-of-the-art baselines such as LIME and MAPLE, our explanations satisfy both the anchoring and sparsity constraints by construction, while consistently achieving lower relevance error.
\end{abstract}

\begin{keywords}
Explainable Artificial Intelligence \sep Probabilistic Explanations \sep
Sparse Linear Functions \sep Iterative Hard Thresholding \sep
Mixed Integer Programming 
\end{keywords}

\maketitle


\section{Introduction}
\label{sec:introduction}

As machine learning models increasingly influence high-stakes domains---such as criminal justice, medical diagnosis, and social scoring---the need for ethics, fairness, and safety has become more pressing than ever. Explainable Artificial Intelligence (XAI) confronts this challenge by providing methods that enable users to interpret model behavior without requiring deep technical expertise \citep{Burkart.JAIR.2021,Molnar.Book.2025}. More recently, the subfield of \emph{formal explainability} has emerged to offer rigorous mathematical guarantees regarding explanation quality, size, and semantics. The primary aim is to establish robust theoretical foundations for explaining predictions, thereby systematically fostering trust and confidence in model capabilities \citep{Ignatiev.IJCAI.2020, MarquesSilva.AAAI.2022}.

When a machine learning model makes a consequential decision, the most natural question is, ``Why?'' Formal explainability seeks to translate this question into a clear, interpretable answer with mathematical guarantees. For a given data instance $\vec{x}$ and classifier $f$, the aim is to identify a rule that justifies the output $f(\vec{x})$. Such a rule can be represented by a feature subset $S$, so that changing features outside $S$ does not alter the outcome $f(\vec{x})$. The restriction of $\vec{x}$ to $S$, denoted $\vec{x}_S$, thus contains all necessary information to determine $f(\vec{x})$. Accordingly, $S$ is often called a (weak) \emph{abductive explanation} \citep{Cooper.AI.2023,Ignatiev.AAAI.2019} or a \emph{sufficient reason} \citep{Barcelo.NeurIPS.2020,Darwiche.ECAI.2020}.

Despite their logical rigor, abductive explanations are often too complex for practical use \citep{Ignatiev.CP.2020, Ignatiev.AAAI.2019}. This limitation fundamentally stems from human cognitive constraints. Classic psychological studies show that while people can retain about seven distinct pieces of information in short-term memory \citep{Miller.PR.1956}, our ability to process interacting variables is even more limited---typically capped at four or five elements before cognitive overload \citep{Halford.BBS.1998, JohnsonLaird.PNAS.2010}. As a result, empirical XAI research emphasizes the need for concise explanations to achieve true interpretability \citep{Lage.HCOMP.2019, Narayanan.arXiv.2018}.

To overcome these cognitive bottlenecks, recent research has shifted toward \emph{probabilistic explanations}, which seek to balance certainty with interpretability \citep{Blanc.NeurIPS.2021, Izza.JAR.2023, Waldchen.JAIR.2021}. In this framework, a probabilistic explanation for a classifier $f$ and instance $\vec{x}$ is a small feature subset $S$ such that $\vec{x}_S$ determines $f(\vec{x})$ with high probability. The effectiveness of $S$ is measured by its \emph{relevance error}: the probability that $f$ assigns a different label to a random instance $\vec{z}$ than to $\vec{x}$, even when $\vec{z}$ and $\vec{x}$ are indistinguishable on $S$. Evaluated with respect to a distribution $\mathcal{D}$---such as the uniform distribution or a local neighborhood---finding a probabilistic explanation becomes a constrained stochastic optimization problem. For instance, seeking an explanation with the lowest relevance error, subject to a user-defined size bound $k$, leads to the following formulation:
\begin{equation}
\label{pb:classification}
\tag{P1}
\begin{aligned}
\text{minimize} \quad & \mathbb{P}_{\vec{z} \sim \mathcal{D}}[f(\vec{z}) \neq f(\vec{x}) \mid \vec{z}_S = \vec{x}_S] \\
\text{subject to} \quad & |S| \leq k
\end{aligned}
\end{equation}

\begin{example}
\label{ex:loan-classification}
For illustration, consider a binary classifier $f$ that evaluates loan applications, where the applicant’s profile $\vec{x}$ is represented by a set of discrete, Boolean propositions. Suppose applicant $\vec{x}$ is approved ($f(\vec{x}) = +1$), even though their credit history is shorter than $10$ years. A strict abductive explanation might require fixing several features---such as 
$[\text{Account Age} \geq 5\text{ yrs}]$, $[\text{Citizen} = \text{Yes}]$, $[\text{Income} \geq 65\text{k}]$, $[\text{Debt-to-Income Ratio (DTI)} \leq 30\%]$, 
$[\text{Credit History} < 10\text{ yrs}]$, and $[\text{Employed} \geq 2\text{ yrs}]$---to mathematically guarantee that $f$ outputs $+1$. However, presenting such a lengthy rule risks the very cognitive overload that XAI seeks to prevent. By instead framing the explanation as the stochastic optimization problem in \eqref{pb:classification} 
and imposing a cognitive size limit of $k=2$, we might identify a smaller subset $S = \{[\text{Income} \geq 65\text{k}], [\text{DTI} \leq 30\%]\}$. In this case, the relevance error is the probability that a random applicant $\vec{z} \sim \mathcal{D}$ matching $\vec{x}$ on these two specific conditions would be denied the loan ($f(\vec{z}) = -1$). 
If this error is $0.02$, then among all applicants agreeing with $\vec{x}$ on these two conditions, only $2\%$ receive a different decision. The user receives a clear, digestible answer, while the probabilistic framework rigorously quantifies the rare edge cases where the explanation does not hold.
\end{example}

While probabilistic explanations successfully address cognitive limitations, their use has largely been restricted to categorical classification. This prompts a key question: \emph{how can these guarantees be extended to a broader class of predictive models?} In this paper, we present a unified framework that represents data instances as vectors in the Boolean hypercube $\{-1,+1\}^d$. 
Continuous or categorical attributes are mapped to Boolean features in advance, following standard practice in model-agnostic explainability \citep{Ribeiro.KDD.2016}. Our approach encompasses both binary classification models---where $f$ takes values in $\{-1,+1\}$---and continuous regression models, where $f$ spans an interval such as $[-1, +1]$.\footnote{Standardizing the codomain to $[-1,+1]$ simplifies notation and highlights the mathematical parallel with the zero-one loss in binary classification. All theoretical results extend to any bounded continuous range $[-c, +c]$ for constant $c > 0$ by adjusting the normalization factor.} Importantly, we impose no assumptions on the internal structure of $f$: whether $f$ is a tree ensemble, support vector machine, or deep neural network, it is treated entirely as a black box, accessible only through value queries.

When explaining a continuous prediction $f(\vec{x})$, selecting a feature subset $S$ alone often fails to capture both the magnitude and direction of each feature's influence. $S$ acts as a binary mask---a support of a weight vector that indicates only which features are present, without conveying the strength or sign of their contributions. To address this shortcoming, we define explanations as \emph{sparse linear functions} $\vec{w} \in \mathbb{R}^d$ that satisfy the anchoring hyperplane condition $\vec{w} \cdot \vec{x} = f(\vec{x})$. This constraint, related to the \emph{local accuracy} property \citep{Lundberg.NeurIPS.2017}, ensures that any linear explanation $\vec{w}$ remains perfectly consistent with the black-box model at $\vec{x}$. The sparsity of $\vec{w}$---its number of nonzero coefficients, $\|\vec{w}\|_0$---reflects its conciseness. By allowing weights to take continuous values, $\vec{w}$ captures both the direction and magnitude of each feature’s contribution, thereby improving interpretability and expressiveness.

Crucially, sparse linear explanations provide a strict generalization of subset-based explanations. Since $\vec{w}$ is zero outside its support $S$, any instance $\vec{z}$ that matches $\vec{x}$ on $S$ will also satisfy $\vec{w} \cdot \vec{z} = \vec{w} \cdot \vec{x}$, so the set of instances described by $\vec{w}$ always contains those described by $S$. If the coefficients of $\vec{w}$ are \emph{strictly non-compensatory} (that is, no combination of feature weights cancels another), the two sets coincide and we recover the original subset explanation. Otherwise, $\vec{w}$ describes a broader set of instances and encodes more expressive rules: instead of requiring every feature in $S$ to be fixed, it only constrains a weighted combination of features, allowing the explanation to capture compensatory effects between conditions.

To assess how well a sparse linear model approximates the black-box model, we replace the conditional zero-one loss in \eqref{pb:classification} with a normalized conditional quadratic loss: $\ell(y, y') = \tfrac{1}{4}(y - y')^2$. In our unified framework, this loss is applied consistently to both continuous regression and binary classification tasks. The factor of $1/4$ ensures that for binary labels $\{-1, +1\}$, the quadratic loss matches the zero-one classification error exactly. The relevance error for an explanation $\vec{w}$ is defined as the conditional expectation of $\ell(f(\vec{z}), f(\vec{x}))$, given $\vec{w} \cdot \vec{z} = \vec{w} \cdot \vec{x}$. With respect to the chosen distribution $\mathcal{D}$, the central optimization problem addressed in this paper is:
\begin{equation}
\label{pb:relevance}
\tag{P2}
\begin{aligned}
\text{minimize} \quad & \mathbb{E}_{\vec{z} \sim \mathcal{D}}\left[\ell(f(\vec{z}),f(\vec{x})) \mid \vec{w} \cdot \vec{z} = \vec{w} \cdot \vec{x}\right] \\
\text{subject to} \quad & \vec{w} \cdot \vec{x} = f(\vec{x}) \quad \text{and} \quad \|\vec{w}\|_0 \leq k
\end{aligned}
\end{equation}

\begin{example}
\label{ex:loan-regression}
Consider now the same applicant in a continuous regression task: assigning a personalized loan interest rate, where lower values are preferable. Suppose $\vec{x}$ receives a rate of $8\%$ ($f(\vec{x}) = 0.08$). Reporting the feature subset $S = \{[\text{DTI} \leq 30\%], [\text{Credit History} \geq 10\text{ yrs}]\}$ shows which variables were important, but not how they influenced the outcome. 
Solving \eqref{pb:relevance} with $k = 2$ yields $w_{[\text{DTI} \leq 30\%]} = -0.03$ and $w_{[\text{Credit History} \geq 10\text{ yrs}]} = -0.11$. Since the contribution of a feature is $w_i x_i$, and $\vec{x}$ satisfies the first condition ($x_i = +1$) but not the second ($x_i = -1$), the rate decomposes as $-0.03 + 0.11 = 0.08$. The user thus learns that the low debt ratio contributes $-3$ points to the rate while the short credit history contributes $+11$, the two effects partially compensating---an insight that no feature subset can convey.
\end{example}

While \eqref{pb:relevance} provides a unified framework for generating linear explanations, solving it exactly presents severe computational hurdles. In Section~\ref{sec:problem}, we demonstrate that for neural networks, \eqref{pb:relevance} is \ClassNPPP-hard---placing it well beyond the capabilities of modern exact solvers. However, this hardness does not preclude the existence of algorithms with approximation guarantees relative to a tractable ``surrogate'' error. We focus on the \emph{fidelity error}, a widely used metric in model-agnostic explainability \citep{Lakkaraju.AIES.2019, Li.ICLR.2021, Ribeiro.KDD.2016}, which measures the normalized quadratic loss incurred by $\vec{w}$ on random instances labeled by $f$. By substituting the conditional objective in \eqref{pb:relevance} with the fidelity error, we obtain:
\begin{equation}
\label{pb:fidelity}
\tag{P3}
\begin{aligned}
\text{minimize} \quad & \mathbb{E}_{\vec{z} \sim \mathcal{D}}\left[\ell(\vec{w} \cdot \vec{z}, f(\vec{z}))\right] \\
\text{subject to} \quad & \vec{w} \cdot \vec{x} = f(\vec{x}) \quad \text{and} \quad \|\vec{w}\|_0 \leq k
\end{aligned}
\end{equation}

The connection between \eqref{pb:relevance} and \eqref{pb:fidelity} relies on the choice of probability distribution $\mathcal{D}$. Drawing on distance-based models for discrete spaces, such as the Mallows model \citep{Mallows.Biometrika.1957} and its extensions \citep{Fligner.Book.1993, Marden.Book.1996}, we focus on the following family of distributions:
\begin{align*}
\mathcal D(\vec z) \propto \exp\left(-\tfrac{\sigma}{2} \| \vec z - \vec x \|_1\right)
\end{align*}
Here, $\tfrac{1}{2}\|\vec{z} - \vec{x}\|_1$ is the Hamming distance (the number of features on which $\vec{x}$ and $\vec{z}$ differ), and $\sigma$ is a concentration parameter controlling how tightly instances $\vec{z}$ cluster around $\vec{x}$. Notably, setting $\sigma = 0$ recovers the uniform distribution over the instance space. 
Throughout, $\sigma$ is a user-supplied parameter that specifies how local the explanation is meant to be; Section~\ref{sec:experiments} reports how our results vary with it.

In Section~\ref{sec:pp}, we prove that for this parameterized family, the relevance error of any feasible solution to \eqref{pb:relevance} is at most $(1 + e^{-\sigma})^k$ times its fidelity error. 
Furthermore, since \eqref{pb:fidelity} involves a non-conditional expected loss function, it is highly amenable to empirical sampling. Thus, by approximating the expected fidelity error over $m$ samples, we arrive at the following practical, data-driven optimization problem:
\begin{equation}
\label{pb:empirical_fidelity}
\tag{P4}
\begin{aligned}
\text{minimize} \quad & \tfrac{1}{m} \sum_{i = 1}^m \ell(\vec w \cdot \vec z_i, f(\vec z_i)) \\
\text{subject to} \quad & \vec{w} \cdot \vec{x} = f(\vec{x}) \quad \text{and} \quad \|\vec{w}\|_0 \leq k
\end{aligned}
\end{equation}
This task generalizes the classic \emph{subset selection} \citep{Beale.Biometrika.1967, Hocking.Tech.1967} problem, also known as \emph{sparse regression} \citep{Natarajan.SJC.1995}, and is therefore \ClassNP-hard. 
It can nonetheless be solved exactly by Mixed Integer Programming (MIP), which remains tractable on instances of moderate size. 
We further show that a sample size $m$ polynomial in $k$, $1/\varepsilon$ and $\log d$ suffices: with high probability, the resulting $k$-sparse explanation has relevance error at most $(1 + e^{-\sigma})^k(v^* + \varepsilon)$, where $v^*$ is the optimal value of \eqref{pb:empirical_fidelity}. The quality of MIP-based explanations therefore improves as the concentration $\sigma$ increases and, notably, the sample complexity does not depend on $\sigma$ at all.

As a strictly polynomial-time alternative, Section~\ref{sec:np} adapts the Iterative Hard Thresholding (IHT) algorithm \citep{Blumensath.JFAA.2008, Blumensath.ACHA.2009} to compute approximate solutions to \eqref{pb:empirical_fidelity}. Our variant projects, at each iteration, onto the intersection of the $k$-sparsity constraint and the anchoring hyperplane---an operation that is both exact and computable in $\mathcal{O}(d \log d + k^2)$ time. For our family of distributions, IHT produces $k$-sparse explanations whose relevance error is, with high probability, 
at most $(1 + e^{-\sigma})^k(c \cdot v^* + \varepsilon)$, where $c > 1$ is an explicit multiplicative factor governed by $k$ and $\sigma$. However, the sample complexity for IHT grows with $\cosh^4(\sigma/2)$---that is, as $e^{2\sigma}$ for large $\sigma$. This discrepancy with the MIP bound is not merely analytical: while the MIP guarantee only requires uniform convergence of the empirical fidelity error, the IHT guarantee also depends on restricted strong convexity and smoothness constants, which deteriorate as $\mathcal{D}$ becomes more concentrated around $\vec{x}$. 
As a result, the parameter $\sigma$ governs a meaningful computational and statistical trade-off. In the \emph{localized regime} (large $\sigma$), 
the relevance bound is tight and MIP is preferred, as its sample complexity remains independent of $\sigma$. In contrast, IHT faces a double penalty: its sample complexity grows as $e^{2\sigma}$ and its approximation factor worsens as $e^{\sigma}$. In the \emph{intermediate regime}, where $\sigma$ is large enough for the relevance bound to be informative but small enough to keep both the sampling cost and the approximation factor moderate, IHT becomes a compelling alternative, yielding highly relevant explanations in polynomial time.

Section~\ref{sec:experiments} empirically benchmarks our MIP and IHT methods against the LIME \citep{Ribeiro.KDD.2016} and MAPLE \citep{Plumb.NeurIPS.2018} explainers on both classification and regression tasks. Our results show that MAPLE consistently exceeds the sparsity budget $k$ and frequently violates the anchoring condition as well, placing its explanations outside the intended interpretability regime. While LIME satisfies the sparsity budget, it systematically deviates from anchoring and thus fails to reproduce the black-box prediction $f(\vec{x})$ at the instance $\vec{x}$ being explained. Because LIME operates over a strictly larger hypothesis space, it sometimes achieves a lower empirical fidelity error; our measurements indicate that this apparent advantage results directly from the anchoring deviation. Nevertheless, across variations in $k$, $\sigma$, and $m$, both IHT and MIP consistently outperform LIME in terms of relevance---the objective our framework is designed to optimize. Between our two methods, relevance rarely differs by more than a few hundredths, so the choice becomes essentially computational: MIP certifies optimality within seconds on small and medium benchmarks but loses this certificate on the largest ones, whereas IHT is by far the fastest explainer and remains effective at every scale.

In summary, our main contributions are as follows:
\begin{itemize}
\item We present a unified formal framework for probabilistic explainability, bridging categorical classification and continuous regression through sparse, anchored linear models that strictly generalize subset-based explanations. 
We also prove that computing the optimal relevance error for neural network black-box models is \ClassNPPP-hard.
\item We establish a theoretical link between the intractable relevance error and a tractable surrogate---the fidelity error. 
By introducing a parameterized family of local distributions, we show that the relevance error of a $k$-sparse explanation is bounded by $(1 + e^{-\sigma})^k$ times its fidelity error.
\item We formulate the empirical fidelity problem, encode its exact solution as a MIP, and approximate it with a strictly polynomial-time IHT variant.
Our experiments validate the theoretical guarantees, demonstrating that, in contrast to MAPLE and LIME, our methods consistently satisfy both sparsity and anchoring constraints while achieving superior relevance. 
Furthermore, they offer complementary practical advantages: MIP reliably finds optimal solutions for medium-dimensional problems, while IHT scales efficiently to high-dimensional datasets.
\end{itemize}

The remainder of the paper is organized as follows. Section~\ref{sec:related} reviews related work. Section~\ref{sec:problem} introduces sparse anchored linear explanations and establishes the hardness of \eqref{pb:relevance}. Section~\ref{sec:pp} relates the relevance error to the fidelity error and derives the MIP formulation, while Section~\ref{sec:np} presents our IHT variant and its approximation guarantee. Section~\ref{sec:experiments} reports the experimental evaluation, and Section~\ref{sec:conclusion} concludes.


\section{Related Work}
\label{sec:related}

To contextualize our contributions, we organize the landscape of local explainability into two primary categories: probabilistic explanations, which emphasize rigorous theoretical guarantees, and linear explanations, 
which focus on flexibility and the ability to handle continuous outcomes.

\subsection{Probabilistic Explanations}

Probabilistic explanations, as a natural generalization of abductive explanations, have become increasingly prominent in formal explainability. 
Whereas traditional abductive reasoning seeks a feature subset that logically entails a prediction, probabilistic approaches relax this requirement to address human cognitive limits, aiming for subsets that guarantee the prediction only with high probability.

In categorical classification, a feature subset $S$ is called a \emph{$k$-size $\tau$-relevant} probabilistic explanation for a reference instance $\vec{x}$, classifier $f$, and distribution $\mathcal{D}$ if $|S| \leq k$ and its zero-one relevance error is at most $1 - \tau$. Recall that for binary labels, this zero-one error coincides exactly with the normalized quadratic relevance error of \eqref{pb:relevance}, so the two formulations agree on classification tasks. Within this framework and the setup of \eqref{pb:classification}, one can either fix a cognitive size limit $k$ and seek $S$ that minimizes relevance error \citep{Bounia.UAI.2023, Koriche.ECML.2024}, or set a relevance threshold $\tau$ and find the smallest $S$ achieving relevance error at most $1 - \tau$ \citep{Waldchen.JAIR.2021, Izza.JAR.2023}. Since exactly computing these probabilities typically requires knowledge of the model's internal architecture, both approaches are usually studied in a \emph{model-specific} (white-box) setting.

Determining whether a $k$-size $\tau$-relevant explanation exists for a given $\vec{x}$ is computationally hard: the problem is \ClassNPPP-hard for neural networks~\citep{Waldchen.JAIR.2021} and \ClassNP-hard for decision trees~\citep{Arenas.NeurIPS.2022}. Even at $\tau = 1$---the case of abductive explanations---the task remains \ClassSigmaTwoP-hard for random forests~\citep{Audemard.AAAI.2022} and neural networks~\citep{Barcelo.NeurIPS.2020}. The outlook for fixed-parameter tractability and approximability is similarly negative. Finding cardinality-minimal $\tau$-relevant explanations is \ClassWTwo-hard in $k$, even for monotone neural networks and $\tau = 1$~\citep{Ordyniak.IJCAI.2023}, and is \ClassNP-hard to approximate within any factor of $d^{1 - \alpha}$ for any $\alpha > 0$, for both neural networks~\citep{Waldchen.JAIR.2021} and decision trees~\citep{Kozachinskiy.arXiv.2023}. To the best of our knowledge, the main positive result to date applies to linear functions $f$, where a relaxed problem becomes polynomial-time approximable~\citep{Subercaseaux.AAAI.2025}. All of these results concern subset-based explanations. Whether the same barriers apply to linear explanations is not immediate: the constraint $\vec{w} \cdot \vec{z} = \vec{w} \cdot \vec{x}$ selects a far richer family of regions than the subcube induced by a subset, 
so neither hardness nor tractability transfers mechanically. We settle this question in Section~\ref{sec:problem}.

Shifting to black-box classifiers, \citet{Blanc.NeurIPS.2021} showed that if user-supplied instances $\vec{x}$ are drawn uniformly at random, a $\tau$-relevant explanation $S$ of size $k$ can be constructed, with high probability, from the decision path $T(\vec{x})$ of a depth-$k$ decision tree $T$, provided that $T$ disagrees with $f$ on at most a $1 - \tau$ fraction of the instances. 
If $k$ is polynomial in the ``average certificate complexity'' of $f$, the tree $T$ can be learned in polynomial time. In our terminology, their result links fidelity to relevance: 
a surrogate with low fidelity error yields an explanation with low relevance error. Our bound in Section~\ref{sec:pp} has the same form, but uses a sparse linear surrogate instead of a tree. Notably, their guarantee is \emph{average-case} (over random $\vec{x}$), while ours is \emph{pointwise} (for every fixed $\vec{x}$), which is the setting local explainability calls for. Consequently, the two guarantees are complementary, and neither subsumes the other.

Among model-agnostic approaches, \textsc{Anchors} \citep{Ribeiro.AAAI.2018} deserves special mention. This method, designed for black-box classifiers, outputs a feature subset (or logical decision rule) $S$ whose \emph{precision}---the probability that $f$ agrees with $f(\vec{x})$ on instances satisfying $S$---exceeds a user-specified threshold, estimated via sampling rather than computed from a white-box model. Precision, being a conditional quantity, aligns with relevance rather than fidelity in our terminology. Thus, \textsc{Anchors} sits at the intersection of probabilistic and model-agnostic approaches---incorporating elements of both, but only heuristically. Notably, there are no formal guarantees on the \emph{size} of the returned subset: the anchor simply expands until the precision threshold is met, and achieving that threshold is trivial, since including all features of $\vec{x}$ guarantees a precision of $1$. The main computational challenge---finding a \emph{short} anchor with high precision---is delegated to a greedy beam search algorithm, which lacks any approximation guarantee. Most importantly for our purposes, \textsc{Anchors} returns a subset rather than a linear function, placing it outside the class of explanations considered in our work. Accordingly, our experimental comparisons in Section~\ref{sec:experiments} focus exclusively on explainers that produce linear functions.

A related but distinct line of work concerns \emph{contrastive} (or \emph{counterfactual}) explanations, which identify a subset of features that, when modified, alter the prediction; 
we refer the reader to \citet{Guidotti.DMKD.2024} for a survey. These methods answer ``what would have to change?'', whereas abductive explanations and their probabilistic extensions answer ``why this prediction?''. The two questions have recently been brought together by \citet{Bassan.ICLR.2026}, who show that sufficient and contrastive reasons, in both their local and global forms, can all be characterized as minimizers of a single probabilistic value function, and who tie the complexity of computing them to monotonicity, submodularity, and supermodularity properties of that function. 
Their unification is orthogonal to ours: it ranges over explanation \emph{types} while remaining within feature subsets and categorical classification, whereas we unify classification and regression \emph{tasks} within the class of sparse linear functions. Because no prior probabilistic framework has encompassed continuous regression, we focus exclusively on this abductive trajectory and do not consider contrastive explanations further.

\subsection{Linear Explanations}

To address the limitations of discrete subsets---specifically their inability to explain continuous regression tasks---and to abstract away the internal architectures of target models, 
a separate line of research has focused on \emph{model-agnostic} linear explanations. These methods extract an interpretable linear surrogate from the local neighborhood of a data instance \citep{Ribeiro.KDD.2016, Lundberg.NeurIPS.2017, Plumb.NeurIPS.2018, Agarwal.ICML.2021, Zhao.UAI.2021}.

A common goal across these approaches is to minimize an unconstrained objective of the form:
\begin{equation}
    \label{eq:model_agnostic_objective}
    \tfrac{1}{m} \sum_{i=1}^m \phi_{\vec{x}}(\vec{z}_i) (f(\vec{z}_i) - \vec{w} \cdot \vec{z}_i)^2 + \psi(\vec{w})
\end{equation}
Here, $\{(\vec{z}_i,f(\vec{z}_i))\}_{i=1}^m$ is a set of labeled samples generated from a neighborhood distribution around $\vec{x}$, the weighting function $\phi_{\vec{x}}(\vec{z}_i)$ assesses the importance of $\vec{z}_i$, and the regularization term $\psi(\vec{w})$ penalizes the complexity of the linear model $\vec{w}$. For example, in the widely used LIME method \citep{Ribeiro.KDD.2016}, $\phi_{\vec{x}}(\vec{z}_i)$ is defined as a normalized distance between $\vec{z}_i$ and $\vec{x}$. Similarly, in \textsc{KernelSHAP} \citep{Lundberg.NeurIPS.2017}, $\phi_{\vec{x}}(\vec{z}_i)$ is derived from combinatorial Shapley weights to ensure game-theoretic properties. In the MAPLE method \citep{Plumb.NeurIPS.2018}, $\phi_{\vec{x}}(\vec{z}_i)$ is the fraction of trees, in a random forest trained on $f$, for which $\vec{z}_i$ falls in the same leaf as $\vec{x}$.

Despite their widespread adoption, the reliance on the unconstrained, heuristic objective in \eqref{eq:model_agnostic_objective} introduces significant vulnerabilities. A growing body of literature highlights the local instability of these methods, demonstrating that semantically indistinguishable inputs can yield vastly different linear explanations \citep{Alvarez-Melis.WHI.2018}. Furthermore, because the sampling distribution and weighting functions are not explicitly tied to formal performance bounds, these explainers can be manipulated by measurement biases or adversarial perturbations \citep{Ghorbani.AAAI.2019, Slack.AIES.2020}, and their theoretical convergence heavily depends on the choice of hyperparameters \citep{Garreau.AISTATS.2020}.

Our framework departs from this template in three fundamental ways. First, we rigorously formalize locality, moving beyond heuristic approaches. The family of $\sigma$-parameterized distributions introduced in Section~\ref{sec:introduction} serves a role similar to LIME's exponential kernel, but crucially, it acts as the formal probability law defining the expected loss---and thus the relevance error---instead of an arbitrary heuristic weight in an empirical sum. This distinction enables us to state formal bounds. Second, we impose consistency at the reference instance as a hard constraint: $\vec{w} \cdot \vec{x} = f(\vec{x})$. As noted in Section~\ref{sec:introduction}, this anchoring constraint is not new: it is the \emph{local accuracy} property of \citet{Lundberg.NeurIPS.2017}, or equivalently, the efficiency axiom $\phi_0 + \sum_{i=1}^d \phi_i = f(\vec{x})$ of Shapley-based explainers. Our innovation lies in enforcing this anchoring condition \emph{jointly} with a strict sparsity budget. Third, we enforce sparsity as the hard constraint $\|\vec{w}\|_0 \leq k$, rather than through a soft penalty $\psi$ or a relaxed selection procedure. While LIME targets a sparse support via LASSO selection of $k$ features, it provides no guarantee on the resulting support; MAPLE has no sparsity mechanism, as its local linear model assigns a weight to every feature. Solving \eqref{pb:fidelity} under both constraints yields explanations with theoretically sound relevance and fidelity bounds.

These distinctions clarify why \textsc{KernelSHAP}, though it produces linear explanations that satisfy the anchoring constraint, is not used as a baseline in Section~\ref{sec:experiments}. Its minimization of \eqref{eq:model_agnostic_objective} is instrumental rather than predictive: the Shapley kernel is chosen so that the solution of the weighted least-squares program matches the Shapley values, not to ensure local surrogate accuracy. Furthermore, its samples correspond to coalition maskings rather than to neighbors of $\vec{x}$. Most importantly, Shapley values are inherently dense---every feature receives a value---so \textsc{KernelSHAP} cannot satisfy a sparsity budget $\|\vec{w}\|_0 \leq k$. Truncating to the $k$ largest values breaks the local accuracy property (the anchoring constraint). Any explainer that cannot be made feasible for \eqref{pb:fidelity} without sacrificing its defining property is not a meaningful comparator.


\section{Problem Formulation and Complexity}
\label{sec:problem}

To rigorously study the trade-offs between explanation sparsity, exactness, and computational tractability, we must first cast the intuitions from Section~\ref{sec:introduction} into a formal optimization problem. In this section, we define the mathematical framework for generating sparse linear explanations and formulate our primary optimization objective. We then analyze the computational complexity of solving this problem exactly, demonstrating why empirical approximations are practically necessary.

\subsection{Notation}

\paragraph{Sets, Vectors, and Matrices.}
Plain letters represent functions and scalars, while boldface letters represent vectors and matrices. We denote the all-ones and all-zeros vectors as $\vec{1}$ and $\vec{0}$, respectively. The identity matrix is denoted $\vec{I}$. For a positive integer $d$, we use $[d]$ as shorthand for the index set $\{1, \ldots, d\}$. The standard basis vectors of $\mathbb{R}^d$ are denoted $\vec{e}_1, \ldots, \vec{e}_d$, where $e_{ij} = 1$ if $i = j$ and $0$ otherwise. For a vector $\vec{w} \in \mathbb{R}^d$ and a subset $S \subseteq [d]$, the projection of $\vec{w}$ to the coordinates indexed by $S$ is the vector $\vec{w}_S \in \mathbb{R}^{|S|}$. The \emph{support set} of $\vec{w} \in \mathbb{R}^d$, denoted as $\mathrm{support}(\vec{w})$, is the set of coordinates $j \in [d]$ for which $w_j \neq 0$.

\paragraph{Operations and Norms.}
The transpose of a matrix $\vec{Z}$ is denoted $\vec{Z}^{\top}$. The scalar product of two vectors $\vec{v}$ and $\vec{w}$ is denoted as $\vec{v} \cdot \vec{w}$, and their coordinate-wise (or Hadamard) product as $\vec{v} \odot \vec{w}$. For a scalar $p \in [1, \infty)$, the $L_p$ norm of $\vec{w}$ is denoted as $\|\vec{w}\|_p$. Limit cases include the maximum absolute value $\|\vec{w}\|_{\infty} = \max_{j=1}^{d} |w_j|$. Following common usage, we extend this notation to the sparsity measure $\|\vec{w}\|_0 = |\mathrm{support}(\vec{w})|$---a measure that counts the number of non-zero coordinates.

\paragraph{Balls, Hyperplanes, and Projections.}
For a scalar $p \in [0, \infty]$ and a radius $r \geq 0$, the $L_p$ ball (centered at $\vec{0}$) is defined as $\mathcal{B}_p(r) = \{ \vec{w} \in \mathbb{R}^d : \|\vec{w}\|_p \leq r \}$. For a normal vector $\vec{u} \in \mathbb{R}^d$ and a scalar $r \in \mathbb{R}$, the corresponding hyperplane is defined as $\mathcal{H}(\vec{u}, r) = \{ \vec{w} \in \mathbb{R}^d : \vec{u} \cdot \vec{w} = r \}$. Finally, the Euclidean projection of a vector $\vec{w} \in \mathbb{R}^d$ onto a set $\mathcal{V} \subseteq \mathbb{R}^d$ is given by:\footnote{When the minimizer is not unique, ties are broken arbitrarily.}
\begin{align*}
    \Pi_{\mathcal{V}}(\vec{w}) = \arg\min_{\vec{u} \in \mathcal{V}} \|\vec{w} - \vec{u}\|_2
\end{align*}

\subsection{Problem Formulation}

In this study, we focus on explanation tasks where data instances are represented by a set of \emph{interpretable literals}. Returning to our financial example, consider a bank customer who wants to understand why her loan interest rate was set at $8\%$ (i.e., $f(\vec{x}) = 0.08$). Interpretable literals such as $[\text{Income} \geq 65\text{k}]$, $[\text{DTI} \leq 30\%]$, and $[\text{Credit History} \geq 10\text{ yrs}]$ can be used to construct the explanation. Each literal is assigned the polarity present in the customer's profile, resulting in a clear and succinct if-then rule over weighted features, for example: 
\begin{align*}
-0.03\,[\text{DTI} \leq 30\%] \;+\; 0.11\,[\text{Credit History} < 10\text{ yrs}] \;\rightarrow\; \text{Rate} = 0.08
\end{align*}
where each coefficient is the contribution $w_j x_j$ of the corresponding literal, and the two contributions sum to the predicted rate.

Formally, let $[d]$ be the set of interpretable literals. Treating them as binary features, the prediction models in this study are pseudo-Boolean functions of the form $f: \{-1,+1\}^d \rightarrow [-1,+1]$. Here, $f$ is a (binary) classifier if its codomain is $\{-1,+1\}$, and a regressor if its codomain is $[-1,+1]$. Each input to $f$ is a data instance $\vec{x} \in \{-1,+1\}^d$, where $x_j \in \{+1, -1\}$ indicates whether the $j$-th literal is present positively or negatively in $\vec{x}$. By convention, the first literal is the constant $x_1 = 1$, so that a bias (or intercept) term $w_1$ is always available; it is counted within the sparsity measure $\|\vec{w}\|_0$ like any other coefficient, which avoids unnecessary mathematical friction in what follows.

A \emph{linear explanation} for $f(\vec{x})$ is a vector $\vec{w} \in \mathbb{R}^{d}$ that satisfies the anchoring condition $\vec{w} \cdot \vec{x} = f(\vec{x})$. As in the previous example, such an explanation can be interpreted as an if-then rule over weighted literals: the head is $f(\vec{x})$, and the body consists of the literals $j \in \mathrm{support}(\vec{w})$, each taken with the polarity $x_j$ it has in $\vec{x}$ and the weight $w_j$. An explanation $\vec{w}$ is considered \emph{$k$-sparse} if $\|\vec{w}\|_0 \leq k$. Throughout, the two constraints we impose on explanations are sparsity, which reflects the cognitive limits discussed in Section~\ref{sec:introduction}, and anchoring, which enforces consistency with the model at $\vec{x}$. We formalize the set of all such valid explanations as our primary hypothesis space:
\begin{equation}
    \label{eq:hypothesis_space}
    \mathcal W_{\vec{x}, k} = \mathcal{H}(\vec{x}, f(\vec{x})) \cap \mathcal{B}_0(k)
\end{equation}

\begin{figure}[pos=htbp]
    \begin{subfigure}{0.48\textwidth}
        \centering
        \begin{tikzpicture}[scale=0.6,>=stealth]
        \draw[gray,densely dotted] (-4,-4) grid (4,4);
        \node[below] at (-4,0) {\smaller $-1$};
        \node[below] at (-2,0) {\smaller $-\tfrac{1}{2}$};
        \node[below] at (2,0) {\smaller $\tfrac{1}{2}$};
        \node[below] at (4,0) {\smaller $1$};
        \node[left] at (0,-4) {\smaller $-1$};
        \node[left] at (0,-2) {\smaller $-\tfrac{1}{2}$};
        \node[left] at (0,2) {\smaller $\tfrac{1}{2}$};
        \node[left] at (0,4) {\smaller $1$};
        \node[below left] at (0,0) {\smaller $0$};
        \coordinate[label={[green!75!black]below left:{\smaller $\vec{x}$}}] (X) at (4,4);
        \fill[green!75!black,opacity=1.0]  (4,4) circle (4pt);
        \draw[black, thick, <->] (0,-4.5) -- (0,4.5);
        \draw[black, thick, <->] (-4.5,0) -- (4.5,0);
        \fill[red,opacity=0.8]  (2,0) circle (4pt);
        \fill[red,opacity=0.8]  (0,2) circle (4pt);
        \coordinate[label={[black]above:{\smaller $w_1$}}] (W1) at (4,0);
        \coordinate[label={[black]right:{\smaller $w_2$}}] (W2) at (0,4);
        \draw[blue,thick,opacity=0.5,<->] (-2.25,4.25) -- (4.25,-2.25);
        \coordinate[label={[blue]left:{\smaller $\mathcal{H}(\vec{x},f(\vec{x}))$}}] (H) at (4,-2);
        \end{tikzpicture}
        \caption{For $\vec{x} = (1,1)$ with $f(\vec{x}) = \tfrac{1}{2}$ and $k = 1$, the hyperplane $\mathcal{H}(\vec{x}, f(\vec{x}))$ appears in blue, and the $L_0$ ball $\mathcal{B}_0(1)$ corresponds to the two black coordinate axes. Their intersection consists of the two red points $(\tfrac{1}{2},0)$ and $(0,\tfrac{1}{2})$: one per candidate support, each of dimension $k - 1 = 0$.}
        \label{fig:explanations:2D}
    \end{subfigure}
    \hfill
    \begin{subfigure}{0.48\textwidth}
        \centering
        \begin{tikzpicture}[scale=3.75, x={(-0.6cm,-0.35cm)}, y={(0.8cm,-0.15cm)}, z={(0cm,1cm)}, >=stealth]
        \draw[gray!70, densely dotted] (1,0,0) -- (1,1,0) -- (0,1,0);
        \draw[gray!70, densely dotted] (1,0,0) -- (1,0,1) -- (1,1,1) -- (1,1,0);
        \draw[gray!70, densely dotted] (0,1,0) -- (0,1,1) -- (0,0,1);
        \draw[gray!70, densely dotted] (1,0,1) -- (0,0,1);
        \draw[gray!70, densely dotted] (0,1,1) -- (1,1,1);
        \draw[gray!70, densely dotted] (0.5,0,0) -- (0.5,1,0) -- (0.5,1,1) -- (0.5,0,1) -- cycle;
        \draw[gray!70, densely dotted] (0,0.5,0) -- (1,0.5,0) -- (1,0.5,1) -- (0,0.5,1) -- cycle;
        \draw[gray!70, densely dotted] (0,0,0.5) -- (1,0,0.5) -- (1,1,0.5) -- (0,1,0.5) -- cycle;
        \draw[gray!70, densely dotted] (0.5,0.5,0) -- (0.5,0.5,1);
        \draw[gray!70, densely dotted] (0,0.5,0.5) -- (1,0.5,0.5);
        \draw[gray!70, densely dotted] (0.5,0,0.5) -- (0.5,1,0.5);
        \draw[gray, thick, ->] (0,0,0) -- (1.1,0,0);
        \draw[gray, thick, ->] (0,0,0) -- (0,1.1,0);
        \draw[gray, thick, ->] (0,0,0) -- (0,0,1.1);
        \coordinate[label={[gray]above:{\smaller $w_1$}}] (w1) at (1,-0.05,0);
        \coordinate[label={[gray]above:{\smaller $w_2$}}] (w2) at (0,1,0);
        \coordinate[label={[gray]right:{\smaller $w_3$}}] (w3) at (0,0,1);
        \coordinate[label={[black]right:{\smaller $\tfrac{1}{2}$}}] (x05) at (0.5,0,0);
        \coordinate[label={[black]right:{\smaller $1$}}] (x10) at (1,0,0);
        \coordinate[label={[black]below:{\smaller $\tfrac{1}{2}$}}] (y05) at (0,0.5,0);
        \coordinate[label={[black]below:{\smaller $1$}}] (y10) at (0,1,0);
        \coordinate[label={[black]left:{\smaller $\tfrac{1}{2}$}}] (z05) at (0,0,0.5);
        \coordinate[label={[black]left:{\smaller $1$}}] (z10) at (0,0,1);
        \coordinate[label={[black]below:{\smaller $0$}}] (O) at (0,0,0);
        \fill[green!75!black,opacity=1.0]  (1,1,1) circle (0.8pt);
        \coordinate[label={[green!75!black]below left:{\smaller $\vec{x}$}}] (X) at (1,1,1);
        \fill[blue, opacity=0.15] (0.8, -0.15, -0.15) -- (-0.15, 0.8, -0.15) -- (-0.15, -0.15, 0.8) -- cycle;
        \node[blue, left] at (-0.15, -0.15, 0.8) {\smaller $\mathcal{H}(\vec{x},f(\vec{x}))$};
        \draw[red, thick] (0.65,-0.15,0) -- (-0.15,0.65,0);
        \draw[red, thick] (0.65,0,-0.15) -- (-0.15,0,0.65);
        \draw[red, thick] (0,0.65,-0.15) -- (0,-0.15,0.65);
        \end{tikzpicture}
        \caption{For $\vec{x} = (1,1,1)$ with $f(\vec{x}) = \tfrac{1}{2}$ and $k = 2$, the hyperplane $\mathcal{H}(\vec{x}, f(\vec{x}))$ appears as the large blue triangle, while the $L_0$ ball $\mathcal{B}_0(2)$ is the union of the three coordinate planes $w_1 = 0$, $w_2 = 0$ and $w_3 = 0$. Their intersection consists of the three red lines, one per candidate support, each of dimension $k - 1 = 1$.}
        \label{fig:explanations:3D}
    \end{subfigure}
    \caption{A geometric illustration of $k$-sparse linear explanations $\vec{w}$.}
    \label{fig:explanations}
\end{figure}

As illustrated in Figure~\ref{fig:explanations}, the geometric space $\mathcal W_{\vec{x}, k}$ of $k$-sparse linear explanations for $f(\vec{x})$ is the intersection of two distinct objects: the anchoring hyperplane $\mathcal{H}(\vec{x}, f(\vec{x}))$ and the $L_0$ ball $\mathcal{B}_0(k)$. While the hyperplane is a continuous convex space, the $L_0$ ball is a non-convex union of the $\tbinom{d}{k}$ coordinate subspaces of dimension $k$. Their intersection is therefore a combinatorial union of affine subspaces of dimension $k - 1$, one for each candidate support---isolated points when $k = 1$, as in Figure~\ref{fig:explanations:2D}, and lines when $k = 2$, as in Figure~\ref{fig:explanations:3D}. None of them is empty, since $x_j \neq 0$ for every $j \in [d]$, so that every support carries linear explanations for $f(\vec{x})$. This decomposition locates a key source of computational hardness: selecting a support is a combinatorial problem over \smash{$\tbinom{d}{k}$} alternatives. The quality of probabilistic explanations is assessed in relation to a probability distribution $\mathcal{D}$ over $\{-1, +1\}^d$, such that $\mathcal D(\vec x) > 0$. For instance, $\mathcal{D}$ could represent the uniform distribution $\mathcal{U}$ across $\{-1, +1\}^d$ or, more restrictively, a localized neighborhood distribution surrounding the instance $\vec{x}$ that is being explained.

To quantify this quality, our framework adopts a normalized quadratic loss: $\ell(y, y') = \tfrac{1}{4}(y - y')^2$. This loss is applied consistently across both continuous regression and binary classification tasks. While its use in regression is standard, it is equally principled in the binary classification setting, where $f$ outputs values in $\{-1, +1\}$. The $1/4$ scaling ensures that the normalized quadratic loss exactly matches the zero-one loss $\ell(y, y') = \mathbb{1}[y \neq y']$ for $y, y' \in \{-1, +1\}$. The theoretical justification and convergence guarantees for the quadratic loss as a surrogate in classification are well established in statistical learning theory \citep{Bartlett.JASA.2006, Suykens.NPL.1999} and supported by empirical results in machine learning \citep{Rifkin.JMLR.2004, Hui.ICLR.2021}.

\begin{definition}[Relevance Error]
    For a prediction model $f: \{-1,+1\}^d \to [-1,+1]$, a reference instance $\vec{x} \in \{-1,+1\}^d$, and a distribution $\mathcal{D}$,
    the \emph{relevance error} of a vector $\vec{w} \in \mathbb{R}^d$ is given by:\footnote{The conditioning event always contains $\vec{z} = \vec{x}$, so that $\mathbb{P}_{\vec{z} \sim \mathcal{D}}[\vec{w} \cdot \vec{z} = \vec{w} \cdot \vec{x}] \geq \mathcal{D}(\vec{x})$. Assuming $\mathcal{D}(\vec{x}) > 0$ is thus enough for $\mathsf{R}_{f, \vec{x}, \mathcal{D}}(\vec{w})$ to be well defined, simultaneously for every $\vec{w} \in \mathbb{R}^d$.}
    \begin{equation}
    \label{eq:relevance}
    \mathsf{R}_{f, \vec{x}, \mathcal{D}}(\vec{w}) = \mathbb{E}_{\vec{z} \sim \mathcal{D}} \left[ \ell(f(\vec{z}), f(\vec{x})) \mid \vec{w} \cdot \vec{z} = \vec{w} \cdot \vec{x} \right]
    \end{equation}
\end{definition}
In other words, the relevance error of $\vec{w}$ measures the expected discrepancy between $f(\vec{z})$ and $f(\vec{x})$ for random instances $\vec{z}$ that are indistinguishable from $\vec{x}$ under the linear projection defined by $\vec{w}$. Notably, due to the exact equivalence established above, when $f$ is a binary classifier, this expectation perfectly recovers the original probabilistic objective formulated in \eqref{pb:classification}.

With these concepts established, the decision version of the stochastic optimization problem presented in \eqref{pb:relevance} is defined as follows.

\begin{definition}[\textsc{SLE} Problem]
    An instance of the \textsc{Sparse Linear Explanation} (\textsc{SLE}) problem consists of a predictive model $f: \{-1, +1\}^d \rightarrow [-1,+1]$, a data instance $\vec{x} \in \{-1, +1\}^d$, a probability distribution $\mathcal{D}$ over $\{-1, +1\}^d$, a sparsity level $k \geq 1$, and a relevance threshold $\tau \in [0,1]$. The question is whether there exists a linear explanation $\vec{w} \in \mathbb{R}^d$ for $f(\vec{x})$ such that $\|\vec{w}\|_0 \leq k$ and $\mathsf{R}_{f, \vec{x}, \mathcal{D}}(\vec{w}) \leq 1 - \tau$.
\end{definition}

When $f(\vec{x}) = 0$, the null vector $\vec{w} = \vec{0}$ meets both constraints: it is $k$-sparse, and the anchoring condition holds trivially. However, this explanation is degenerate, as its conditioning event covers the entire instance space and its body is empty, providing no meaningful insight into the model. Consequently, we set this degenerate case aside in the remainder of the paper, while the complexity results below apply to the unrestricted problem.

The above definition treats $f$ and $\mathcal{D}$ as abstract objects, as required for the remainder of the paper. Notably, our algorithms interact with $f$ solely through value queries---by evaluating $f(\vec{z})$ at chosen instances $\vec{z}$---and never through its internal structure. In contrast, a complexity statement requires an input with a measurable description length. To address this, we fix such representations in the next subsection: $f$ is specified as a neural network, and $\mathcal{D}$ is provided in closed form.

\subsection{Problem Complexity}

Solving \textsc{SLE} exactly presents two intertwined challenges. First, we must select the correct support from among $\tbinom{d}{k}$ possible candidates. Second, for each candidate, we need to compute a probability over the exponentially many instances that agree with $\vec{x}$ on that support. The first is a combinatorial search, the second a counting task. The complexity class that encompasses this combination is \ClassNPPP: problems solvable in polynomial time by a nondeterministic machine with access to a counting oracle. In this subsection, our goal is to show that \textsc{SLE} is hard for this class. Consequently, no algorithm can be expected to solve it exactly at scale, and the empirical relaxations developed in later sections are not merely convenient---they are essential.

A natural approach is to reduce from the subset version of the problem, whose hardness was established by \citet{Waldchen.JAIR.2021}. Recall that a subset explanation $S$ selects instances agreeing with $\vec{x}$ on $S$, while a linear explanation $\vec{w}$ selects all instances $\vec{z}$ such that $\vec{w} \cdot \vec{z} = \vec{w} \cdot \vec{x}$. If these two families of sets always coincided, the reduction would be immediate. However, in general, they do not, and discrepancies arise in both directions. This gap is the central obstacle addressed in the remainder of the subsection, so it is helpful to first illustrate it concretely before proceeding to the formal details.

\begin{example}
Consider $d = 3$, the instance $\vec{x} = (+1, -1, +1)$, and the weight vector $\vec{w} = (1, 1, 0)$, so that $\vec{w} \cdot \vec{x} = 0$. The instance $\vec{z} = (-1, +1, +1)$ also satisfies $\vec{w} \cdot \vec{z} = 0$, even though it differs from $\vec{x}$ on two coordinates. Thus, the set selected by $\vec{w}$ is strictly larger than the one selected by the subset $\{1, 2\}$, as it also includes instances obtained by flipping both of the first two coordinates. In general, a linear explanation selects a \emph{union} of such subsets---one passing through $\vec{x}$ and others that do not. This difference is important for the reduction: converting a subset explanation into a linear one is straightforward, since the coefficients can be chosen so that the union collapses to a single set. The reverse direction is more subtle, since a linear explanation may owe its quality to a set that does not pass through $\vec{x}$, and such a set does not necessarily certify the subset problem.
\end{example}

The proof proceeds in three movements. First, any linear explanation can be replaced by one of the subcubes it selects, without increasing the relevance error (Lemma~\ref{lem:extraction}); this reduces the linear problem to a subset-like problem, though at the cost of losing the guarantee that the subcube passes through $\vec{x}$. Second, we make this loss harmless: by replacing selected variables with the parity of $k+1$ fresh copies, a budget of $k$ can never extract any information about them (Lemma~\ref{lem:shielding} and Corollary~\ref{cor:shielding}), so only the coordinates of interest can be exploited. Third, for the resulting family of instances, a subcube meeting the relevance threshold exists if and only if an \emph{anchored} one does---a property we call alignment (Definition~\ref{def:alignment})---and both are equivalent to the existence of a witness for the source problem (Lemma~\ref{lem:witness} and Corollary~\ref{cor:alignment}). The reduction then follows (Theorem~\ref{thm:hardness}). Before proceeding, we clarify the representation of the model and the source problem.

\paragraph{Model Representation.} For a meaningful complexity statement, the input must have a finite, measurable description length.
Accordingly, we assume that $f$ is specified as a feedforward ReLU neural network, with all weights and biases in $[-1, +1]$ and activation functions restricted to the identity $u \mapsto u$ and the rectifier $u \mapsto \max\{0, u\}$. The description length of $f$ is measured by its number of gates. These networks can emulate Boolean circuits of comparable size and depth \citep{Mukherjee.arXiv.2017, Parberry.MPNN.1996}, so we may equivalently describe a binary classifier as a circuit built from the connectives $\wedge$ (\textsc{and}), $\vee$ (\textsc{or}), $\neg$ (\textsc{not}), and $\oplus$ (\textsc{xor}), counting its gates. While circuits are traditionally defined over $\{0,1\}^d$, we retain the symmetric cube $\{-1, +1\}^d$ throughout this paper, identifying the bit $b$ with $2b - 1$. This bijection preserves both the set of instances and the uniform distribution, ensuring every result from \citep{Waldchen.JAIR.2021} applies directly in our encoding. Finally, we assume that $\mathcal{D}(\vec{z})$ can be evaluated in polynomial time.

\paragraph{Source Problem.} All the hardness in this subsection ultimately comes from \textsc{E-Maj-Sat} \citep{Littman.JAIR.1998}, the canonical complete problem for \ClassNPPP. This problem is the probabilistic analog of \textsc{Sat}: instead of asking whether \emph{some} assignment of the free variables satisfies the formula, it asks whether \emph{most} of them do.

\begin{definition}[\textsc{E-Maj-Sat} Problem]
\label{def:emajsat}
An instance of \textsc{E-Maj-Sat} is a Boolean formula $\varphi$ in conjunctive normal form over $n$ variables, together with an integer $\kappa \leq n$. The question is whether there exists an assignment $\vec{u}^\ast$ of the first $\kappa$ variables, called a \emph{witness}, such that a majority of the assignments of the remaining $n - \kappa$ variables satisfy $\varphi$.
\end{definition}

The existential choice of the witness underlies the problem's \ClassNP-hardness, while the majority test accounts for its \ClassPP-hardness. These are precisely the two ingredients found in \textsc{SLE}: selecting a support and then performing counting within it.

\paragraph{Subcubes.} The sets discussed above have a standard name. Given a subset $S \subseteq [d]$ and an assignment $\vec{a} \in \{-1,+1\}^{S}$, the \emph{subcube} $\mathcal{C}(S,\vec{a}) = \{\vec{z} : \vec{z}_S = \vec{a}\}$ is the set of instances agreeing with $\vec{a}$ on $S$, and $|S|$ is its \emph{codimension}. A subcube is \emph{anchored} when $\vec{a} = \vec{x}_S$, that is, when it passes through the instance being explained. It is convenient to assess the quality of a subcube in exactly the same way as for linear explanations, by extending the relevance error of \eqref{eq:relevance} to any event $E$ with positive probability:
\[
\mathsf{R}_{f, \vec{x}, \mathcal{D}}(E) = \mathbb{E}_{\vec{z} \sim \mathcal{D}} \left[ \ell(f(\vec{z}), f(\vec{x})) \mid \vec{z} \in E \right]
\]
With this convention, $\mathsf{R}_{f, \vec{x}, \mathcal{D}}(\vec{w})$ represents the relevance error of the set $\{\vec{z} : \vec{w} \cdot \vec{z} = \vec{w} \cdot \vec{x}\}$ selected by $\vec{w}$; a $k$-sparse subset explanation corresponds to an anchored subcube of codimension at most $k$, and a single threshold $1 - \tau$ applies in both cases.

The first lemma formalizes the observation from the example: a linear explanation selects a union of subcubes of equal size, so its relevance error is the average of their errors. Since an average cannot be less than its minimum, at least one subcube must achieve a relevance error no greater than that of the linear explanation itself.

\begin{lemma}[Extraction]
\label{lem:extraction}
Let $\vec{w} \in \mathbb{R}^d$ with $S = \mathrm{support}(\vec{w})$. Under the uniform distribution $\mathcal{U}$, the set selected by $\vec{w}$ is the disjoint union of the subcubes $\mathcal{C}(S,\vec{a})$ over the assignments $\vec{a}$ satisfying $\vec{w}_S \cdot \vec{a} = \vec{w}_S \cdot \vec{x}_S$. Moreover, at least one of them satisfies
\begin{align*}
\mathsf{R}_{f, \vec{x}, \mathcal{U}}(\mathcal{C}(S,\vec{a})) \; \leq \; \mathsf{R}_{f, \vec{x}, \mathcal{U}}(\vec{w})
\end{align*}
\end{lemma}
\begin{proof}
The value $\vec{w} \cdot \vec{z}$ depends on $\vec{z}$ only through $\vec{z}_S$, so the set selected by $\vec{w}$ is indeed the union of the subcubes of codimension $|S|$ described above, and this union is disjoint since distinct assignments of $S$ define disjoint subcubes. Each of these subcubes contains $2^{d - |S|}$ instances, hence carries the same conditional mass under $\mathcal{U}$, so that $\mathsf{R}_{f, \vec{x}, \mathcal{U}}(\vec{w})$ is the plain average of their relevance errors. A minimum is never larger than an average.
\end{proof}

The subcube produced by Lemma~\ref{lem:extraction} is not necessarily anchored, which is the gap we must address. We bridge this gap by ensuring that unanchored subcubes are uninformative, using a simple device: if a variable is replaced by the parity of $k+1$ fresh copies, then fixing at most $k$ of these leaves at least one unfixed, so the parity remains an unbiased coin. As a result, no explanation with budget $k$ can extract any information about that variable, regardless of the values assigned.

\begin{definition}[Shielding]
\label{def:shielding}
Let $g$ be a function of $n$ variables and let $A$ be a subset of them. The \emph{$k$-shielding of $g$ over $A$} is obtained by replacing every variable $y_i$ with $i \in A$ by the parity $\bigoplus_{j=1}^{k+1} y_{i,j}$ of $k+1$ fresh variables. It adds $k|A|$ variables and $O(k|A|)$ gates.
\end{definition}

\begin{lemma}[Shielding]
\label{lem:shielding}
Fix any assignment of at most $k$ of the fresh variables, so that every block of $k+1$ copies retains at least one free variable. Given this assignment, the vector of parities $(\bigoplus_j y_{i,j})_{i \in A}$ is uniformly distributed and independent of all the other variables.
\end{lemma}
\begin{proof}
Each $i \in A$ is associated with a block of $k+1$ fresh variables. Since at most $k$ variables are fixed in total, every block retains at least one unfixed variable. The parity of each block, being the parity of a nonempty set of independent unbiased bits, is itself unbiased and remains independent across blocks as well as from the unshielded variables.
\end{proof}

We use this device as follows: a candidate explanation does not benefit by allocating any part of its budget to fresh variables, so we may always assume it spends none on them.

\begin{corollary}[Shielded Coordinates are Useless]
\label{cor:shielding}
Let $\tilde{f}$ be the $k$-shielding of a function $f$, and let $B$ denote the set of fresh variables it introduces. Then, for every subcube $\mathcal{C}(S, \vec{a})$ of codimension at most $k$,
\begin{align*}
\mathsf{R}_{\tilde{f}, \vec{x}, \mathcal{U}}(\mathcal{C}(S, \vec{a})) \; = \; \mathsf{R}_{\tilde{f}, \vec{x}, \mathcal{U}}\big(\mathcal{C}(S \setminus B, \, \vec{a}_{S \setminus B})\big)
\end{align*}
In particular, for every subcube of codimension at most $k$ there is another one, of no larger codimension, that fixes no fresh variable and has the same relevance error.
\end{corollary}
\begin{proof}
By Lemma~\ref{lem:shielding}, conditioning on the coordinates of $S \cap B$ leaves the vector of parities uniform and independent of every other variable. The conditional law of the arguments of the shielded function, and therefore that of $\tilde{f}(\vec{z})$, is thus the same as under the conditioning by $S \setminus B$ alone, and the two conditional expectations coincide.
\end{proof}

We can now name the property that our family of instances must satisfy. It is a statement at the threshold rather than an equality between optima: an unanchored subcube may well be strictly better than every anchored one, and this is harmless as long as it does not cross the threshold separating positive from negative instances.

\begin{definition}[Alignment]
\label{def:alignment}
An instance $(f, \vec{x}, k)$ is \emph{aligned at level $\tau$} if, whenever some subcube of codimension at most $k$ has relevance error at most $1 - \tau$, some \emph{anchored} subcube of codimension at most $k$ has relevance error at most $1 - \tau$ as well.
\end{definition}

On an aligned instance, the anchored and the free versions of the problem have the same answer, which is what our two directions require: the forward direction builds a linear explanation from an anchored subcube, while the backward direction receives an arbitrary subcube from Lemma~\ref{lem:extraction}.

\paragraph{The Construction.} Let $(\varphi, \kappa)$ be an instance of \textsc{E-Maj-Sat} and let $\tau \in (0,1)$. We build a classifier $f$, a reference instance $\vec{x}$ and a budget $k = \kappa$ in two steps borrowed from \citet{Waldchen.JAIR.2021}, followed by one application of Definition~\ref{def:shielding}.

The first step is the duplication construction of \citet[Lemma 3.5]{Waldchen.JAIR.2021}. It produces a circuit $\Phi$ over two blocks $\vec{u}$ and $\vec{v}$ of size $\kappa$ each, a block $\vec{r}$ holding the remaining variables of $\varphi$, and one extra variable. Its reference instance carries opposite values on the two duplicated blocks, so that for each index $i$, exactly one of the two $i$-th coordinates encodes a given value (True or False) of the $i$-th variable of $\varphi$. Let $\textsf{EQ}$ denote the event that the two duplicated blocks agree; conditionally on $\textsf{EQ}$, fixing exactly one of these two coordinates to its value in $\vec{x}$ assigns that value to the $i$-th variable of $\varphi$. For a subcube $\mathcal{C}$ fixing coordinates of $\vec{u}$ and $\vec{v}$ only, let $\mathsf{q}(\mathcal{C})$ denote the probability that $\Phi$ evaluates to $+1$ given $\mathcal{C}$. Their analysis \citep[Lemma 3.3]{Waldchen.JAIR.2021} gives
\begin{align}
\label{eq:duplication}
\mathsf{q}(\mathcal{C}) = \tfrac{1}{2} + \Big( \mathbb{P}\big[\varphi \mid \mathcal{C}, \textsf{EQ}\big] - \tfrac{1}{2} \Big) \, \mathbb{P}\big[\textsf{EQ} \mid \mathcal{C}\big]
\end{align}
No assignment of the variables of $\varphi$ appears here: a subcube may leave coordinates free, in which case the first conditional probability averages over them. Note also that $\mathbb{P}[\textsf{EQ} \mid \mathcal{C}] > 0$ whenever $\mathcal{C}$ leaves at least one coordinate of each duplicated pair free, or fixes both to equal values, so that the conditional probability $\mathbb{P}[\varphi \mid \mathcal{C}, \textsf{EQ}]$ is well defined in these cases; when $\mathcal{C}$ fixes both coordinates of some pair to opposite values, $\mathbb{P}[\textsf{EQ} \mid \mathcal{C}] = 0$ and \eqref{eq:duplication} reads $\mathsf{q}(\mathcal{C}) = \tfrac{1}{2}$, with the second term interpreted as zero.

The second step is the threshold adjustment from \citet[Lemmas 3.8--3.11]{Waldchen.JAIR.2021}. This transforms $\Phi$ into the final circuit $f$ by composing it with auxiliary blocks, ensuring that $f(\vec{x}) = +1$ and that, for every subcube $\mathcal{C}$ fixing coordinates of $\vec{u}$ and $\vec{v}$ only,
\begin{align}
\label{eq:threshold}
\mathsf{R}_{f, \vec{x}, \mathcal{U}}(\mathcal{C}) \leq 1 - \tau \quad \Longleftrightarrow \quad \mathsf{q}(\mathcal{C}) > \tfrac{1}{2}
\end{align}
The null vector is not a feasible explanation here, since $f(\vec{x}) = +1 \neq 0$.

Finally, we $k$-shield every variable except those belonging to the two duplicated blocks: specifically, the block $\vec{r}$, the extra variable from the first step, and all auxiliary blocks from the second step are shielded. By Lemma~\ref{lem:shielding}, this modification does not affect any of the probabilities established above, and the construction remains polynomial since shielding multiplies the number of variables in a block by $k+1$ and adds a proportional number of gates. The resulting classifier, still denoted $f$, is a $k$-shielding, and therefore falls under Corollary~\ref{cor:shielding}. Shielding also renders the intermediate step of \citep[Lemma 3.7]{Waldchen.JAIR.2021} unnecessary: its sole purpose was to prevent a budget from being allocated to $\vec{r}$, which is now directly guaranteed by Corollary~\ref{cor:shielding}.

One further notion completes the picture. The \emph{subcube induced by} an assignment $\vec{u}^\ast$ to the first $\kappa$ variables of $\varphi$, denoted $\mathcal{C}^\ast(\vec{u}^\ast)$, is defined as the subcube that fixes, for each $i \leq \kappa$, the $i$-th coordinate of $\vec{u}$ or $\vec{v}$ that encodes $u^\ast_i$ to its value in $\vec{x}$. By construction, this subcube is anchored, and its codimension is $\kappa = k$.

The next lemma forms the technical heart of the reduction: it establishes the connection between subcubes in the construction and witnesses of the source instance, working in both directions.

\begin{lemma}[Witnesses and Subcubes]
\label{lem:witness}
For every instance produced by the construction above:
\begin{enumerate}
    \item if some subcube of codimension at most $k$ has relevance error at most $1 - \tau$, then $\varphi$ admits a witness;
    \item conversely, for every witness $\vec{u}^\ast$, the induced subcube $\mathcal{C}^\ast(\vec{u}^\ast)$ has relevance error at most $1 - \tau$.
\end{enumerate}
\end{lemma}
\begin{proof}
\emph{Clause 1.} Let $\mathcal{C}$ be such a subcube. By Corollary~\ref{cor:shielding}, we may assume that it fixes only coordinates of $\vec{u}$ and $\vec{v}$, since replacing it by its unshielded part changes neither its relevance error nor the bound on its codimension. Then \eqref{eq:threshold} gives $\mathsf{q}(\mathcal{C}) > \tfrac{1}{2}$, so that by \eqref{eq:duplication} the product $\big(\mathbb{P}[\varphi \mid \mathcal{C}, \textsf{EQ}] - \tfrac{1}{2}\big)\,\mathbb{P}[\textsf{EQ} \mid \mathcal{C}]$ is strictly positive. Neither factor can therefore vanish, and both must share the same sign. Only the sign of the second factor matters here, not its magnitude: since $\mathbb{P}[\textsf{EQ} \mid \mathcal{C}]$ is a probability, it is strictly positive, and hence $\mathbb{P}[\varphi \mid \mathcal{C}, \textsf{EQ}] > \tfrac{1}{2}$ as well. Conditionally on $\mathcal{C}$ and $\textsf{EQ}$, the two duplicated blocks agree, the coordinates pinned by $\mathcal{C}$ carry the values they encode, the remaining ones are uniform, and $\vec{r}$ is independent of all of them. The value $\mathbb{P}[\varphi \mid \mathcal{C}, \textsf{EQ}]$ is thus an average over assignments of the first $\kappa$ variables compatible with $\mathcal{C}$, of the probability that $\varphi$ is satisfied by a uniform assignment to the remaining variables. At least one such assignment must yield a value exceeding $\tfrac{1}{2}$, and is thus a witness.

\emph{Clause 2.} Let $\vec{u}^\ast$ be a witness and $\mathcal{C}^\ast = \mathcal{C}^\ast(\vec{u}^\ast)$. This fixes one coordinate in each of the $\kappa$ duplicated pairs and leaves the other free, so $\mathbb{P}[\textsf{EQ} \mid \mathcal{C}^\ast] = 2^{-\kappa} > 0$, and conditioning further on $\textsf{EQ}$ pins the first $\kappa$ variables of $\varphi$ to $\vec{u}^\ast$. Then $\mathbb{P}[\varphi \mid \mathcal{C}^\ast, \textsf{EQ}] > \tfrac{1}{2}$ because $\vec{u}^\ast$ is a witness, and \eqref{eq:duplication} gives $\mathsf{q}(\mathcal{C}^\ast) > \tfrac{1}{2}$, so \eqref{eq:threshold} concludes.
\end{proof}

Both properties we need now follow formally, the first one by chaining the two clauses.

\begin{corollary}[Alignment]
\label{cor:alignment}
Let $\tau \in (0,1)$. The construction above maps, in polynomial time, any instance of \textsc{E-Maj-Sat} to a triple $(f, \vec{x}, k)$ such that
\begin{enumerate}
    \item $(f, \vec{x}, k)$ is aligned at level $\tau$; and
    \item the \textsc{E-Maj-Sat} instance is positive if and only if some anchored subcube of codimension at most $k$ has relevance error at most $1 - \tau$.
\end{enumerate}
\end{corollary}
\begin{proof}
For clause 1, suppose that some subcube of codimension at most $k$ has relevance error at most $1 - \tau$. Clause 1 of Lemma~\ref{lem:witness} yields a witness $\vec{u}^\ast$, and clause 2 turns it into the subcube $\mathcal{C}^\ast(\vec{u}^\ast)$, which is anchored, of codimension at most $k$, and of relevance error at most $1 - \tau$. This is precisely Definition~\ref{def:alignment}.

For clause 2, a positive instance has a witness, hence an anchored subcube below the threshold by clause 2 of Lemma~\ref{lem:witness}. Conversely, an anchored subcube below the threshold is in particular a subcube, so clause 1 of Lemma~\ref{lem:witness} produces a witness.
\end{proof}

In particular, if the \textsc{E-Maj-Sat} instance is negative, then no anchored subcube of codimension at most $k$ meets the threshold, and thus, by alignment, no subcube of any kind does. This contrapositive form is the one to keep in mind when reading the backward direction below.

\begin{theorem}[Hardness]
\label{thm:hardness}
For every fixed relevance threshold $\tau \in (0,1)$, the \textsc{Sparse Linear Explanation} problem for ReLU neural networks is \emph{\ClassNPPP-hard}.
\end{theorem}
\begin{proof}
Let $f$, $\vec{x}$, and $k$ be produced by Corollary~\ref{cor:alignment} from an instance of \textsc{E-Maj-Sat}, and consider the \textsc{SLE} instance $(f, \vec{x}, \mathcal{U}, k, \tau)$. This instance is well defined, as $\mathcal{U}(\vec{x}) > 0$, and its construction is polynomial in size.

Suppose first that the \textsc{E-Maj-Sat} instance is positive, and let $\mathcal{C}(S, \vec{x}_S)$ be the anchored subcube given by clause 2 of Corollary~\ref{cor:alignment}. If $S$ is empty, set $\vec{w} = \vec{e}_1$; since $x_1 = 1$ is the constant literal, $\vec{w} \cdot \vec{z} = 1 = f(\vec{x})$ for every instance $\vec{z}$, so $\vec{w}$ selects the entire cube, which is precisely the subcube of empty support. Otherwise, set $w_j = x_j / |S|$ for $j \in S$ and $w_j = 0$ elsewhere; the nonzero coefficients are then equal in magnitude, so the condition $\vec{w} \cdot \vec{z} = \vec{w} \cdot \vec{x}$ reduces to $\sum_{j \in S} x_j z_j = |S|$, and since each term is at most one, this forces $\vec{z}_S = \vec{x}_S$: the set selected by $\vec{w}$ is exactly $\mathcal{C}(S, \vec{x}_S)$. In both cases $\vec{w} \cdot \vec{x} = 1 = f(\vec{x})$ and $\|\vec{w}\|_0 \leq k$, so $\vec{w}$ is a feasible explanation whose relevance error matches that of the subcube, hence is at most $1 - \tau$, and the \textsc{SLE} instance is positive.

Conversely, suppose the \textsc{SLE} instance is positive, witnessed by a feasible $\vec{w}$, and let $S = \mathrm{support}(\vec{w})$, so that $|S| \leq k$. Lemma~\ref{lem:extraction} guarantees the existence of a subcube $\mathcal{C}(S, \vec{a})$ whose relevance error is at most that of $\vec{w}$, and thus at most $1 - \tau$. If $1 \in S$, fixing the constant coordinate is vacuous, so this subcube coincides with $\mathcal{C}(S \setminus \{1\}, \vec{a}_{S \setminus \{1\}})$, of no larger codimension. Because the instance is aligned at level $\tau$, some anchored subcube of codimension at most $k$ also achieves this threshold, and clause 2 of Corollary~\ref{cor:alignment} then ensures the existence of a witness.

The two instances are therefore equivalent, and since \textsc{E-Maj-Sat} is \ClassNPPP-complete \citep{Littman.JAIR.1998}, the result follows.
\end{proof}


\section{Dealing with \ClassPP-Hardness}
\label{sec:pp}

Among the two sources of intractability highlighted in Section~\ref{sec:problem}, only one is out of the ordinary. The combinatorial search over candidate supports $S \subseteq [d]$ of size at most $k$ is the familiar cost of enforcing sparsity, and an extensive literature exists for addressing it---ranging from exact encodings to greedy and thresholding algorithms. The second source is more problematic: even for a \emph{fixed} candidate $\vec{w}$, determining whether its relevance error is at most $1 - \tau$ is \ClassPP-hard. This complexity places the problem beyond the scope of standard optimization techniques, which require that the objective be efficiently computable.

This section addresses the second obstacle by introducing the \emph{fidelity error}, an unconditional surrogate loss. Section~\ref{subsec:relevance_to_fidelity} formally relates relevance and fidelity errors within the parameterized family of neighborhood distributions, while Section~\ref{subsec:empirical_mip} demonstrates that the fidelity error can be efficiently approximated via sampling. This allows for a tractable Mixed Integer Programming (MIP) formulation with formal sample complexity guarantees.

\subsection{From Relevance to Fidelity}
\label{subsec:relevance_to_fidelity}

To circumvent the \ClassPP-hard evaluation of the conditional relevance error, we focus on the \emph{fidelity error}. Commonly used in model-agnostic explainability \citep{Lakkaraju.AIES.2019, Li.ICLR.2021, Ribeiro.KDD.2016}, the fidelity error assesses the expected loss of the linear explanation over the entire distribution, omitting the conditioning event $\vec{w} \cdot \vec{z} = \vec{w} \cdot \vec{x}$. The anchoring constraint $\vec{w} \cdot \vec{x} = f(\vec{x})$ remains an essential part of the definition and will continue to play a central role below. By evaluating the loss unconditionally, this surrogate metric becomes much more tractable for both theoretical analysis and empirical approximation. 

\begin{definition}[Fidelity Error]
    For a prediction model $f$, a reference instance $\vec{x}$, and a distribution $\mathcal{D}$, the fidelity error of a vector $\vec{w}$ under the normalized quadratic loss $\ell$ is given by:
    \begin{align}
        \label{eq:fidelity}
        \mathsf{F}_{f, \vec{x}, \mathcal{D}}(\vec{w}) = \mathbb{E}_{\vec{z} \sim \mathcal{D}} \left[ \ell\big(\vec{w} \cdot \vec{z}, f(\vec{z})\big) \right]
    \end{align}
\end{definition}

To bridge the theoretical gap between the conditional relevance error and the unconditional fidelity error, we rely on the parameterized family of localized distributions introduced in Section~\ref{sec:introduction}. Recall that for a given concentration parameter $\sigma \geq 0$ and a central reference instance $\vec{x}$, the probability of sampling an instance $\vec{z}$ decays exponentially with its Hamming distance $\tfrac{1}{2}\|\vec{x} - \vec{z}\|_1$. Formally, we define the \emph{neighborhood distribution} $\mathcal{D}_{\vec{x}, \sigma}$ as:
\begin{align}
    \label{eq:distribution}
\mathcal{D}_{\vec{x}, \sigma}(\vec{z}) = \frac{1}{Z_{\sigma}} e^{-\frac{\sigma}{2} \|\vec{x} - \vec{z}\|_1} \quad \text{where} \quad Z_{\sigma} = \sum_{j = 0}^{d} \binom{d}{j} e^{-\sigma j} = (1 + e^{-\sigma})^d
\end{align}

As previously noted, the parameter $\sigma$ directly controls the locality of the explanation. When $\sigma = 0$, the decay vanishes and $\mathcal{D}_{\vec{x}, 0}$ recovers the uniform distribution $\mathcal{U}$ over the entire instance space. Conversely, as $\sigma \to \infty$, the distribution concentrates entirely on the center $\vec{x}$.

Importantly, for neighborhood distributions, the fidelity error provides an upper bound for the relevance error of any explanation in $\mathcal{W}_{\vec{x}, k} = \mathcal{H}(\vec{x}, f(\vec{x})) \cap \mathcal{B}_0(k)$, modulated by the sparsity level $k$ and the concentration $\sigma$.

\begin{lemma}[Approximating Relevance via Fidelity]
    \label{lem:fidelity}
    Let $f: \{-1, +1\}^d \rightarrow [-1,+1]$ be a prediction model, let $\vec{x} \in \{-1,+1\}^d$ be a data instance, let $k \geq 1$ be a sparsity parameter, and let $\sigma \geq 0$ be a concentration parameter. Then, for any explanation $\vec{w} \in \mathcal W_{\vec{x}, k}$, its relevance error satisfies
    \begin{align*}
        \mathsf{R}_{f, \vec{x}, \mathcal{D}_{\vec{x}, \sigma}}(\vec{w}) \leq (1 + e^{-\sigma})^k \, \mathsf{F}_{f, \vec{x}, \mathcal{D}_{\vec{x}, \sigma}}(\vec{w}) 
    \end{align*}    
\end{lemma}
\begin{proof}
    Let $S = \mathrm{support}(\vec{w})$, so that $|S| \leq k$ since $\vec{w} \in \mathcal{B}_0(k)$, and let $A = \{\vec{z} \in \{-1,+1\}^d : \vec{w} \cdot \vec{z} = \vec{w} \cdot \vec{x}\}$ be the conditioning event. Since $\vec{w}$ is anchored (i.e., $\vec{w} \in \mathcal{H}(\vec{x}, f(\vec{x}))$), $f(\vec{x}) = \vec{w} \cdot \vec{x}$, hence $f(\vec{x}) = \vec{w} \cdot \vec{z}$ for every $\vec{z} \in A$. Therefore
    \begin{align*}
        \mathsf{R}_{f, \vec{x}, \mathcal{D}_{\vec{x}, \sigma}}(\vec{w})
        = \mathbb{E}\big[\ell\big(\vec{w} \cdot \vec{z}, f(\vec{z})\big) \mid A\big]
        = \frac{1}{\mathbb{P}[A]} \sum_{\vec{z} \in A} \mathcal{D}_{\vec{x}, \sigma}(\vec{z}) \, \ell\big(\vec{w} \cdot \vec{z}, f(\vec{z})\big)
        \leq \frac{\mathsf{F}_{f, \vec{x}, \mathcal{D}_{\vec{x}, \sigma}}(\vec{w})}{\mathbb{P}[A]},
    \end{align*}
    where the first equality uses the symmetry of $\ell$, and the inequality follows from the non-negativity of the summands. It remains to lower-bound $\mathbb{P}[A]$. As $\vec{w}$ vanishes outside $S$, the event $\vec{z}_S = \vec{x}_S$ implies $A$. Moreover, $\|\vec{x} - \vec{z}\|_1$ decomposes coordinatewise, so $\mathcal{D}_{\vec{x}, \sigma}$ is a product distribution with marginals $\mathbb{P}[z_j = x_j] = 1/(1 + e^{-\sigma})$. Hence
    \begin{align*}
        \mathbb{P}[A] \geq \mathbb{P}[\vec{z}_S = \vec{x}_S] = (1 + e^{-\sigma})^{-|S|} \geq (1 + e^{-\sigma})^{-k}. 
    \end{align*}
\end{proof}

It should be stressed that this relationship is one-sided. Decomposing the fidelity error over the conditioning event $A$ and its complement gives $\mathsf{F} = \mathbb{P}[A] \, \mathsf{R} + \mathbb{P}[\bar{A}] \, \mathbb{E}[\ell \mid \bar{A}]$, and the second term is not controlled by the relevance error: an explanation can be highly relevant yet incur a large fidelity error, if it behaves poorly outside the anchoring subcube. Minimizing fidelity is therefore a conservative surrogate---it certifies relevance, but may overlook relevant explanations. Lemma~\ref{lem:fidelity} quantifies the price of this conservatism, which the discussion below shows to be a small constant in the regimes of interest.

Lemma~\ref{lem:fidelity} has important practical consequences for our framework. The bound between the fidelity and relevance errors becomes tighter as the concentration parameter $\sigma$ increases, since the term $e^{-\sigma}$ drops off rapidly. Although a large sparsity level $k$ could, in principle, amplify the exponential factor $(1 + e^{-\sigma})^k$, such high values of $k$ are at odds with the very goal of interpretability. As discussed in Section~\ref{sec:introduction}, human cognitive limits necessitate highly sparse explanations, generally restricting $k$ to four or five features. In the regimes of interest, the factor therefore remains a small constant: for $k = 5$, it is about $4.8$ at $\sigma = 1$, $1.9$ at $\sigma = 2$, and $1.3$ at $\sigma = 3$. This alignment between cognitive constraints and mathematical properties ensures that, for real-world XAI applications, the fidelity error provides a reliable proxy for the computationally intractable relevance error.

\subsection{Empirical Approximation and MIP Formulation}
\label{subsec:empirical_mip}

Interestingly, the fidelity error in \eqref{eq:fidelity} involves an \emph{unconditional} expectation, which is highly amenable to sampling. 
To approximate this expectation, let $\{(\vec{z}_i,f(\vec{z}_i))\}_{i=1}^m$ be a sample set where each $\vec{z}_i$ is drawn independently at random according to $\mathcal{D}_{\vec{x}, \sigma}$, and its target value $f(\vec{z}_i)$ is obtained through query access to $f$. 
The corresponding \emph{empirical fidelity error} is given by:
\begin{align}
    \label{eq:emp_fidelity}
    \widehat{\mathsf{F}}_{f, \vec{x}, m}(\vec{w}) = \frac{1}{m} \sum_{i=1}^m \ell\big(\vec{w} \cdot \vec{z}_i, f(\vec{z}_i)\big) = \frac{1}{4m} \left\| \vec{Z}\vec{w} - \vec{y} \right\|_2^2
\end{align} 
where $\vec{Z} \in \{-1,+1\}^{m \times d}$ collects the sampled instances as rows and $\vec{y} = (f(\vec{z}_1), \dots, f(\vec{z}_m))$ collects their target values.

Based on this objective function, \eqref{pb:fidelity} takes the form of a variant of the well-studied problem known as \emph{sparse regression}, also referred to as \emph{best subset selection}, which dates back at least to \citep{Beale.Biometrika.1967, Hocking.Tech.1967}. While this problem is non-convex and \ClassNP-hard \citep{Natarajan.SJC.1995}, \citet{Bertsimas.AOS.2016} showed that Mixed Integer Programming (MIP) formulations make it practically solvable at scale, an approach that has since benefited from steady progress in branch-and-cut solvers. The following formulation is a variation of their parameter-free approach utilizing \emph{Specially Ordered Sets} (SOS) \citep{Bertsimas.Book.2005}:

\begin{equation}
    \label{pb:mip}
    \tag{MIP}
    \begin{aligned}
        \text{minimize}     \quad & \frac{1}{m} \sum_{i=1}^m \ell\big(\vec{w} \cdot \vec{z}_i, f(\vec{z}_i)\big)  \\
        \text{subject to}   \quad & \vec{w} \cdot \vec{x} = f(\vec{x})  \\
                            \quad & \vec{1} \cdot \vec{u} \leq k \\
                            \quad & \|(w_j, 1 - u_j)\|_0 \leq 1, \quad \text{for all } j \in [d] \\
                            \quad & u_j \in \{0,1\}, \quad \text{for all } j \in [d] \\
                            \quad & w_j \in [-1,+1], \quad \text{for all } j \in [d]
    \end{aligned}
\end{equation}

The constraint $\|(w_j, 1 - u_j)\|_0 \leq 1$ formally encodes a Specially Ordered Set of type 1: if $u_j = 0$ then $w_j$ is forced to zero, and if $u_j = 1$ then $w_j$ is free. Since $w_j \in [-1,+1]$, this is equivalent to the Big-M constraints $-u_j \leq w_j \leq u_j$, which modern MIP solvers handle with high efficiency. The formulation is also guaranteed to be feasible: as $k \geq 1$ and $|f(\vec{x})| \leq 1$, the vector $\vec{w} = f(\vec{x}) \, \vec{e}_1$ satisfies every constraint, and it will serve as the admissible starting point for the iterative method developed in Section~\ref{sec:np}.

Compared to \eqref{pb:fidelity}, this formulation introduces the bounding constraint $w_j \in [-1,+1]$ for all $j \in [d]$. From the standpoint of interpretability, the restriction is mild: since the output of $f$ ranges over $[-1,+1]$, a coefficient of magnitude greater than one would attribute to a single feature an effect exceeding the entire range of the model, which runs against the very purpose of feature attribution. Technically, however, the constraint is essential. Combined with the sparsity constraint, it guarantees that $\|\vec{w}\|_1 \leq \|\vec{w}\|_0 \leq k$, and it is this $L_1$ bound---the radius $B$ of Theorem~\ref{thm:sample_complexity} below---that governs the sample complexity of the formulation; without it, no finite sample size would suffice.

The following result shows that if the solver is supplied with a number of samples $m$ that is quartic in $k$ and logarithmic in $d$, then with high probability, the relevance error of \emph{every} feasible explanation---and hence, in particular, that of the solution returned by the solver---is upper-bounded by a function of its empirical fidelity. We state it in terms of a generic $L_1$ radius $B$, the box-constrained case of \eqref{pb:mip} corresponding to $B = k$; Section~\ref{sec:np} applies the same result with a larger radius.

\begin{theorem}[Approximating Relevance via Empirical Fidelity]
    \label{thm:sample_complexity}
    Let $f: \{-1, +1\}^d \rightarrow [-1,+1]$ be a prediction model, $\vec{x} \in \{-1, +1\}^d$ be a data instance, $\sigma \geq 0$ be a concentration parameter, $k \geq 1$ be a sparsity parameter, and $B \geq 1$ be a magnitude parameter. Then, for any $\delta \in (0,1]$ and any $\varepsilon \in (0,1]$, if 
    \begin{align*}
    m \geq \frac{(B + 1)^4}{4\,\varepsilon^2} \left(16 \ln(2d) + \ln\Bigl(\tfrac{2}{\delta}\Bigr)\right)
    \end{align*}
    then with probability at least $1 - \delta$ over the choice of an i.i.d. sample set of size $m$, every explanation $\vec{w} \in \mathcal W_{\vec{x}, k}$ with $\|\vec{w}\|_1 \leq B$ satisfies
    \begin{align*}
        \mathsf{R}_{f, \vec{x}, \mathcal{D}_{\vec{x}, \sigma}}(\vec{w}) \leq (1 + e^{-\sigma})^k \left( \widehat{\mathsf{F}}_{f, \vec{x}, m}(\vec{w})  + \varepsilon \right)
    \end{align*}    
\end{theorem}
\begin{proof}
    The result relies on standard uniform convergence bounds for linear predictors over $L_1$ spaces \citep{ShalevShwartz.Book.2014, Kakade.NeurIPS.2008}. The underlying setting is as follows: for a scalar $B$, consider a bounded input space $\mathcal Z \subseteq \{\vec{z} \in \mathbb{R}^d: \| \vec z\|_{\infty} \leq 1\}$ and an output space $\mathcal Y \subseteq \mathbb{R}$, alongside the hypothesis class $\mathcal{B}_1(B) = \{\vec{w} \in \mathbb{R}^d: \| \vec{w} \|_1 \leq B \}$ of $L_1$-bounded linear functions. In addition, given two scalars $\rho$ and $c$, let $\ell: \mathcal Y \times \mathcal Y \to \mathbb R$ be a loss function such that, for all $y \in \mathcal Y$, the mapping $a \mapsto \ell(a,y)$ is $\rho$-Lipschitz, and the magnitude $\max_{a \in [-B,+B]} | \ell(a,y) |$ is upper bounded by $c$. Then, as established by Theorem 26.15 in \citet{ShalevShwartz.Book.2014}, for any distribution $\mathcal D$ over $\mathcal Z \times \mathcal Y$, with probability at least $1 - \delta$ over the choice of an i.i.d. sample set of size $m$, every $\vec{w} \in \mathcal{B}_1(B)$ satisfies:
    \begin{align*}
    \mathbb E_{(\vec z,y) \sim \mathcal D}[\ell(\vec w \cdot \vec z, y)] \leq \frac{1}{m} \sum_{i = 1}^m \ell(\vec w \cdot \vec z_i, y_i) 
    + 2\rho B \sqrt{\frac{2 \ln(2d)}{m}} + c\sqrt{\frac{2\ln(2/\delta)}{m}}
    \end{align*} 
    In the setting of our framework, $\mathcal Z = \{-1,+1\}^d$, which meets the above requirement since $\|\vec{z}\|_{\infty} = 1$ for all $\vec{z} \in \mathcal Z$.
    Furthermore, the explanations of interest belong to our primary hypothesis space $\mathcal W_{\vec{x}, k}$ and have an $L_1$ norm of at most $B$; they therefore lie in $\mathcal{B}_1(B)$, so the above bound applies to all of them with this radius, the anchoring and sparsity constraints defining $\mathcal W_{\vec{x}, k}$ only shrinking the class and hence only helping. Note that any $\vec{w}$ feasible for \eqref{pb:mip} satisfies $\|\vec{w}\|_1 \leq \|\vec{w}\|_0 \leq k$, since each non-zero coordinate is bounded by $1$ in absolute value; the box-constrained case is thus recovered by setting $B = k$. The radius is kept explicit because Section~\ref{sec:np} applies this result to iterates that are not confined to the unit box.

    It remains to instantiate the two constants of the loss $\ell(a, y) = \tfrac{1}{4}(a - y)^2$. For any $\vec{z} \in \{-1,+1\}^d$ and any such $\vec{w}$, the prediction $\vec{w} \cdot \vec{z}$ is bounded in $[-B, B]$ and the true target $f(\vec{z})$ is bounded in $[-1, +1]$, so their absolute difference is at most $B+1$. The magnitude of the loss is therefore bounded by $c = (B+1)^2/4$. In addition, the derivative of $\ell$ with respect to the prediction is $\tfrac{1}{2}(\vec{w} \cdot \vec{z} - f(\vec{z}))$, whose magnitude is at most $(B+1)/2$, so that $\rho = (B+1)/2$.
    
    Based on these observations, and upper-bounding $2\rho B = B(B+1)$ by $(B+1)^2$, it follows that with probability at least $1 - \delta$, every such explanation $\vec{w}$ satisfies:
    \begin{align*}
        \mathsf{F}_{f, \vec{x}, \mathcal{D}_{\vec{x}, \sigma}}(\vec{w})  \leq \widehat{\mathsf{F}}_{f, \vec{x}, m}(\vec{w}) + \frac{(B+1)^2}{\sqrt{m}} \left( \sqrt{2\ln(2d)} + \tfrac{1}{4}\sqrt{2\ln(2/\delta)} \right)
    \end{align*}
    Requiring the deviation term to be at most $\varepsilon$, solving for $m$, and applying the algebraic inequality $(a+b)^2 \leq 2a^2 + 2b^2$ yields the lower bound on $m$ stated in Theorem~\ref{thm:sample_complexity}. Under this condition, $\mathsf{F}_{f, \vec{x}, \mathcal{D}_{\vec{x}, \sigma}}(\vec{w}) \leq \widehat{\mathsf{F}}_{f, \vec{x}, m}(\vec{w}) + \varepsilon$ for every such explanation, and multiplying both sides by $(1 + e^{-\sigma})^k$ before applying Lemma~\ref{lem:fidelity} concludes the proof.
\end{proof}

Two important remarks follow. First, for a fixed radius $B$, the required sample size is independent of the concentration parameter $\sigma$;
a single bound applies across the entire family of neighborhood distributions, from the uniform case to sharply localized scenarios.
This is a genuine advantage of \eqref{pb:mip}: its box-constrained feasible set ensures $B = k$ for any $\sigma$;
Section~\ref{sec:np} shows that, in contrast, the polynomial-time alternative incurs a cost for locality, with its radius increasing alongside $\sigma$.
Second, since these guarantees are worst-case results reliant on uniform convergence, the constants involved are quite conservative. For example, setting $k = 5$, $d = 100$, $\varepsilon = 0.1$, and $\delta = 0.05$,
yields a theoretical requirement of roughly $2.9 \times 10^6$ samples. In practice, much smaller sample sizes are sufficient, as demonstrated empirically in Section~\ref{sec:experiments}.
Crucially, the dependence on dimensionality is only logarithmic, ensuring the bound degrades gracefully as $d$ increases.


\section{Dealing with \ClassNP-Hardness}
\label{sec:np}

While the MIP formulation presented in Section~\ref{subsec:empirical_mip} characterizes the exact optimum of the empirical problem, its worst-case exponential runtime reflects the fundamental \ClassNP-hardness of the underlying sparse regression problem. To transition from this computational bottleneck to a strictly polynomial-time approximation, we must move beyond combinatorial search and exploit the geometric structure of the empirical objective. To this end, we rely on the \emph{Restricted Strong Convexity} (RSC) and \emph{Restricted Strong Smoothness} (RSS) properties \citep{Negahban.NeurIPS.2009}. Widely utilized in the statistical learning literature for sparse recovery \citep{Agarwal.NeurIPS.2010, ShalevShwartz.SJO.2010, Jalali.NeurIPS.2011, Jain.NeurIPS.2014, Yuan.JMLR.2017}, these conditions ensure that the empirical loss function behaves much like a strongly convex and smooth function, provided the optimization trajectory is restricted to sparse vectors.

This section establishes that our parameterized family of neighborhood distributions $\mathcal{D}_{\vec{x}, \sigma}$ satisfies these RSC and RSS properties with high probability, provided the sample size scales appropriately with the concentration parameter $\sigma$. This geometric regularity is what allows us to deploy a variant of the \emph{Iterative Hard Thresholding} (IHT) algorithm \citep{Blumensath.JFAA.2008, Blumensath.ACHA.2009}, adapted to the anchoring constraint of our framework and equipped with convergence guarantees. The resulting method runs in polynomial time and offers a scalable alternative to the MIP formulation.

\subsection{Restricted Strong Convexity and Smoothness}
\label{subsec:rsc_rss}

In standard continuous optimization, strong convexity ensures that an objective function curves upward everywhere, allowing gradient-based methods to rapidly converge to a unique global minimum. In contrast, generating $k$-sparse linear explanations constrains the feasible space within the non-convex ball $\mathcal B_0(k)$. When instances are sampled uniformly at random ($\sigma = 0$), the data matrix satisfies, with high probability and for $m$ large enough, the classic \emph{Restricted Isometry Property} (RIP) after the usual normalization $(\sfrac{1}{\sqrt{m}}) \vec{Z}$---a foundational concept in compressed sensing that guarantees a matrix acts nearly as an isometry on sparse vectors \citep{Candes.TIT.2005}. Unfortunately, when we localize explanations ($\sigma > 0$), the sampled instances become biased toward the central instance $\vec{x}$. The individual features remain mutually independent, but they are no longer centered, so that the second-moment matrix ceases to be near-isotropic: it acquires a rank-one component along the single direction $\vec{x}$, which is what violates standard RIP assumptions. The RSC and RSS properties handle this issue by extending the isometry concept to more general loss functions---such as our empirical fidelity error---and, as we will see, the offending direction is precisely the one that the anchoring constraint removes from consideration. Geometrically, these properties guarantee that the loss function's curvature remains well-behaved (neither excessively flat nor steep) as long as updates are limited to a small number of features.

\begin{definition}[Orthogonal RSC/RSS]
    \label{def:rcs_rss}
Given a reference instance $\vec{x} \in \{-1, +1\}^d$, an integer $s \geq 1$, and two scalars $\alpha, \beta > 0$, a matrix $\vec{Z} \in \mathbb{R}^{m \times d}$ is said to satisfy the \emph{$\alpha$-RSC} and \emph{$\beta$-RSS} properties of \emph{order $s$ orthogonal to $\vec{x}$} if for any $\vec w \in \mathcal{B}_0(s)$ such that $\vec{w} \cdot \vec{x} = 0$, we have:
\begin{align*}
\alpha \|\vec{w}\|_2^2 \leq \tfrac{1}{m} \|\vec{Z}\vec{w} \|_2^2 \leq \beta \|\vec{w}\|_2^2
\end{align*}
The ratio $\beta/\alpha$ is called the \emph{restricted condition number}. Note that $\mathcal{B}_0(s') \subseteq \mathcal{B}_0(s)$ whenever $s' \leq s$, so these properties at a given order automatically hold at every smaller order.
\end{definition}

To ground these properties in our unified framework, recall that our goal is to minimize the empirical fidelity error \eqref{eq:emp_fidelity} subject to the size limit $\| \vec w \|_0 \leq k$ and the anchoring constraint $\vec w \cdot \vec x = f(\vec x)$. The reader may notice a discrepancy between this anchoring constraint and the condition $\vec{w} \cdot \vec{x} = 0$ in Definition~\ref{def:rcs_rss}. This distinction is intentional: in gradient-based optimization, we evaluate curvature along \emph{update directions}. If $\vec{w}_1$ and $\vec{w}_2$ are two feasible linear explanations, their difference $\Delta \vec{w} = \vec{w}_1 - \vec{w}_2$ is a $2k$-sparse vector that satisfies $\Delta \vec{w} \cdot \vec{x} = 0$. Thus, bounding the Hessian over these zero-anchored sparse directions directly dictates the stability of the algorithm. Because the empirical fidelity error is quadratic, its Hessian is the constant empirical second-moment matrix $\nabla^2 \widehat{\mathsf{F}} = (\sfrac{1}{2m}) \vec{Z}^{\top} \vec{Z}$. The scalars of Definition~\ref{def:rcs_rss} therefore translate into curvature bounds $\alpha/2$ and $\beta/2$ for $\widehat{\mathsf{F}}$ along sparse anchored directions, leaving the restricted condition number unchanged. Section~\ref{subsec:convergence} works instead with the rescaled objective $L = 2\widehat{\mathsf{F}}$, for which these curvature bounds are exactly $\alpha$ and $\beta$.

The lemma below establishes these bounds, together with a second, auxiliary property. The rank-one shift separating $\vec{Z}$ from its centered counterpart is annihilated along anchored directions, but not elsewhere: the residual of the empirical fidelity error carries a constant component, whose interaction with sparse directions is governed by the empirical mean of the centered rows. Property (ii) states that this mean is uniformly small along such directions, at a tolerance $\omega$ left free at this stage. It plays no part in the RSC/RSS geometry itself and will be invoked in Section~\ref{subsec:convergence}.

\begin{lemma}[Orthogonal RSC/RSS of Neighborhood Distributions]
\label{lem:rsc_rss}
Given a concentration parameter $\sigma \geq 0$, a confidence parameter $\delta \in (0,1]$, an order $s \geq 1$, a target condition scalar $\gamma > 1$, and a tolerance $\omega > 0$, write $\nu = \tanh(\sigma/2)$, $\lambda = \operatorname{sech}^2(\sigma/2)$, $\theta = \tfrac{\gamma-1}{2(\gamma+1)}$, and let $\overline{\vec{Z}} = \vec{Z} - \nu \vec{1} \vec{x}^{\top}$ be the centered data matrix. Then there exist absolute constants $c_0, c_1 > 0$ such that if the sample size satisfies
\begin{align*}
m \geq \max \left\{ \; c_0 \left( \tfrac{\gamma + 1}{\gamma - 1} \right)^2 \left( s \ln\left(\tfrac{ed}{s}\right) + \ln\left(\tfrac{2 c_1}{\delta}\right) \right) \cosh^4\Big(\tfrac{\sigma}{2}\Big), \quad \frac{2s}{\omega^2} \ln\left(\tfrac{4d}{\delta}\right) \; \right\}
\end{align*}
then, with probability at least $1 - \delta$, the following two properties hold simultaneously:
\begin{enumerate}
    \item[(i)] for every $\vec{w} \in \mathcal{B}_0(s)$, $\; \alpha \|\vec{w}\|_2^2 \leq \tfrac{1}{m} \|\overline{\vec{Z}}\vec{w}\|_2^2 \leq \beta \|\vec{w}\|_2^2$, where $\alpha = \lambda(1 - \theta)$ and $\beta = \lambda(1 + \theta)$, so that $\beta/\alpha = \tfrac{3\gamma+1}{\gamma+3} < \gamma$;
    \item[(ii)] $\displaystyle \sup_{\vec{u} \in \mathcal{B}_0(s), \|\vec{u}\|_2 = 1} \left| \tfrac{1}{m} \vec{1}^{\top} \overline{\vec{Z}} \vec{u} \right| \leq \omega$.
\end{enumerate}
In particular, since $\vec{Z}\vec{w} = \overline{\vec{Z}}\vec{w}$ for every $\vec{w}$ with $\vec{w} \cdot \vec{x} = 0$, property (i) implies that $\vec{Z}$ satisfies the $\alpha$-RSC and $\beta$-RSS properties of order $s$ orthogonal to $\vec{x}$.
\end{lemma}

\begin{proof}
    Since the Hamming distance decomposes additively across features as $\tfrac{1}{2}\|\vec{x} - \vec{z}\|_1 = \sum_{j=1}^d \mathbb{1}[x_j \neq z_j]$, the distribution $\mathcal{D}_{\vec{x}, \sigma}$ is a product distribution: the coordinates $z_1, \dots, z_d$ are mutually independent, with $\mathbb{P}[z_j = x_j] = 1/(1 + e^{-\sigma})$, hence
    \begin{align*}
        \mathbb{E}[z_j] = x_j \left( \tfrac{1}{1 + e^{-\sigma}} \right) - x_j \left( \tfrac{e^{-\sigma}}{1 + e^{-\sigma}} \right) = x_j \tfrac{1 - e^{-\sigma}}{1 + e^{-\sigma}} = \nu x_j , \qquad \operatorname{Var}(z_j) = 1 - \nu^2 = \lambda
    \end{align*}
    Writing $\vec{\mu} = \nu \vec{x}$ for the mean vector, so that $\overline{\vec{Z}} = \vec{Z} - \vec{1} \vec{\mu}^{\top}$, the covariance matrix of $\mathcal{D}_{\vec{x}, \sigma}$ is perfectly isotropic, while the second-moment matrix carries an additional rank-one term:
    \begin{align*}
        \vec{\Sigma} = \lambda \vec{I}, \qquad \vec{Q} = \mathbb{E}[\vec{z}\vec{z}^{\top}] = \vec{\Sigma} + \nu^2 \vec{x}\vec{x}^{\top}
    \end{align*}
    Locality thus distorts the geometry in two distinct ways: it shrinks the isotropic component by a factor $\lambda$, and it adds a bias along the single direction $\vec{x}$. The rows of $\overline{\vec{Z}}$ are i.i.d.\ copies of $\vec{z} - \vec{\mu}$; they are centered, have covariance $\vec{\Sigma}$, and each of their coordinates lies in an interval of length $2$, hence is sub-Gaussian with parameter $1$ by Hoeffding's Lemma.

    \emph{Property (i).} By standard restricted eigenvalue bounds for empirical covariance matrices of centered sub-Gaussian vectors \citep[Theorem 6.5]{Wainwright.CUP.2019}, for any tolerance $t \in (0,1]$ the uniform deviation over all $s$-sparse unit vectors satisfies
    \begin{align*}
        \sup_{\vec{w} \in \mathcal{B}_0(s), \|\vec{w}\|_2 = 1} \left| \tfrac{1}{m}\|\overline{\vec{Z}}\vec{w}\|_2^2 - \vec{w}^{\top} \vec{\Sigma} \vec{w} \right| \leq t
    \end{align*}
    with probability at least $1 - c_1 \exp(-c_2 m t^2 + s \ln(ed/s))$, where $c_1, c_2$ are absolute constants; the tolerance we select below is at most $\lambda/2 \leq 1/2$, so this quadratic regime indeed applies. Since $\vec{w}^{\top} \vec{\Sigma} \vec{w} = \lambda \|\vec{w}\|_2^2$ and both sides scale with $\|\vec{w}\|_2^2$, the bound extends to every $\vec{w} \in \mathcal{B}_0(s)$ in the form $| \tfrac{1}{m}\|\overline{\vec{Z}}\vec{w}\|_2^2 - \lambda\|\vec{w}\|_2^2 | \leq t \|\vec{w}\|_2^2$. Setting $t = \theta \lambda$ yields exactly the scalars $\alpha = \lambda(1-\theta)$ and $\beta = \lambda(1+\theta)$ announced in the statement, whose ratio is
    \begin{align*}
        \frac{\beta}{\alpha} = \frac{1+\theta}{1-\theta} = \frac{3\gamma+1}{\gamma+3} < \gamma
    \end{align*}
    the last inequality being equivalent to $\gamma^2 > 1$. The appearance of $\cosh^4(\sigma/2) = 1/\lambda^2$ in the sample size traces back to a mismatch of scales. The deviation bound above is stated in absolute terms, the rows of $\overline{\vec{Z}}$ being sub-Gaussian with parameter $1$ irrespective of $\sigma$, whereas the curvature to be estimated is itself of order $\lambda$. Securing a \emph{relative} precision $\theta$ on that curvature thus requires an absolute precision $\theta \lambda$, and the sample size scales with the inverse square of this quantity.\footnote{This dependence is not tight: the sub-Gaussian parameter obtained from Hoeffding's Lemma is loose for strongly biased coordinates, whose optimal proxy also decreases with $\sigma$. Sharpening it would complicate the analysis without altering the qualitative picture.} Formally, requiring $c_1 \exp(-c_2 m t^2 + s \ln(ed/s)) \leq \delta/2$, taking logarithms, isolating $m$, and substituting $t = \theta\lambda$ with $1/\lambda = \cosh^2(\sigma/2)$ gives
    \begin{align*}
        m \geq \tfrac{1}{c_2 t^2} \left( s \ln\left(\tfrac{ed}{s}\right) + \ln\left(\tfrac{2c_1}{\delta}\right) \right) = \tfrac{4}{c_2} \left( \tfrac{\gamma + 1}{\gamma - 1} \right)^2 \left( s \ln\left(\tfrac{ed}{s}\right) + \ln\left(\tfrac{2c_1}{\delta}\right) \right) \cosh^4\Big(\tfrac{\sigma}{2}\Big)
    \end{align*}
    which is the first term of the maximum, with $c_0 = 4/c_2$.

    \emph{Property (ii).} Let $\overline{\vec{\zeta}} = \tfrac{1}{m} \overline{\vec{Z}}^{\top} \vec{1}$ be the empirical mean of the centered rows, so that the supremum in (ii) equals the largest $\ell_2$ norm of any $s$ coordinates of $\overline{\vec{\zeta}}$, itself at most $\sqrt{s} \, \|\overline{\vec{\zeta}}\|_{\infty}$. Each coordinate is an average of $m$ i.i.d.\ centered variables that are sub-Gaussian with parameter $1$, so Hoeffding's inequality and a union bound over the $d$ coordinates give $\|\overline{\vec{\zeta}}\|_{\infty} \leq \sqrt{2\ln(4d/\delta)/m}$ with probability at least $1 - \delta/2$. The second term of the maximum is precisely the condition ensuring $\sqrt{2 s \ln(4d/\delta)/m} \leq \omega$. A union bound over the two events concludes the proof.
\end{proof}

Two consequences of Lemma~\ref{lem:rsc_rss} deserve emphasis, particularly as they formalize the fundamental computational trade-off with the MIP approach of Section~\ref{subsec:empirical_mip}. First, while MIP's sample complexity is entirely independent of $\sigma$, the required number of samples here scales with $\cosh^4(\sigma/2)$. This is a direct statistical consequence of the distribution's locality: as $\sigma \to \infty$, the variance $\lambda$ of each individual feature vanishes. Estimating a geometric curvature that is itself flattening out requires proportionally more data to separate the signal from the noise. Fortunately, within the cognitively meaningful range of concentration parameters discussed in Section~\ref{sec:pp}, this overhead remains manageable---the multiplier $\cosh^4(\sigma/2)$ is approximately $1.6$ at $\sigma = 1$, $5.7$ at $\sigma = 2$, and $31$ at $\sigma = 3$---only becoming prohibitive for $\sigma \gtrsim 5$. Second, and crucially, both $\alpha$ and $\beta$ shrink by the \emph{exact same} factor $\lambda$, meaning the restricted condition number $\beta/\alpha$ remains strictly bounded by $\gamma$, uniformly across all $\sigma$. In short, locality costs samples, not conditioning---and as the next subsection demonstrates, it is the conditioning that ultimately dictates the convergence guarantees of gradient-based methods.

As for property (ii), the tolerance $\omega$ enters the sample size only through the second term of the maximum, which is linear in $s$ and logarithmic in $d$. Section~\ref{subsec:convergence} calls for $\nu \omega = \Theta(\beta/\sqrt{k})$ at order $s = 3k$, under which this second term scales as $k^2 \sinh^2(\sigma) \ln(4d/\delta)$. It carries the same asymptotic dependence on $\sigma$ as the first term, since $\sinh^2(\sigma) \sim 4\cosh^4(\sigma/2)$, but a quadratic rather than linear dependence on the order; it vanishes altogether at $\sigma = 0$, where the rows are already centered and the property is void. Both terms are in turn dominated by the sample size required for generalization in Theorem~\ref{thm:end_to_end}, so the auxiliary property does not drive the overall complexity.

\subsection{The Iterative Hard Thresholding Explainer}
\label{subsec:iht}

With restricted convexity and smoothness ensured by our exponentially localized distributions, we can efficiently approximate the optimal sparse explanation using a tailored variant of the Iterative Hard Thresholding (IHT) algorithm. The classic IHT method \citep{Blumensath.JFAA.2008,Blumensath.ACHA.2009} operates by alternating between two steps: a standard gradient descent update to reduce the loss, followed by a nonlinear thresholding operation that retains only the $k$ largest-magnitude coefficients, projecting the weights back onto the $L_0$ ball.

However, standard IHT alone is insufficient for our explainability framework. A valid probabilistic linear explanation must not only be sparse but also exactly reconstruct the black-box prediction for the given instance---that is, it must satisfy the anchoring constraint $\vec{w} \cdot \vec{x} = f(\vec{x})$. To achieve this, our variant (described in Algorithm~\ref{alg:iht}) incorporates a specialized projection step: after each gradient update, the continuous weights are projected onto the geometric intersection
$\mathcal{W}_{\vec{x}, k} = \mathcal{H}(\vec{x}, f(\vec{x})) \cap \mathcal{B}_0(k)$ of the anchoring hyperplane and the $k$-sparse ball. The algorithm is initialized with the trivial anchored explanation $\vec{w}_0 = f(\vec{x}) \, \vec{e}_1$, which places the entire prediction on the intercept and attributes nothing to any feature. This $1$-sparse vector is already feasible---anchored and inside the unit box---so that every iterate, from $\vec{w}_0$ onwards, is a valid $k$-sparse anchored explanation. As we show below, this seemingly cosmetic choice is what keeps the iterates bounded, and hence the sample complexity under control.

\begin{algorithm}[t]
    \SetArgSty{textrm}
    \DontPrintSemicolon
    \SetKw{Input}{Input:}
    \BlankLine
    \Input{explanation query $(\vec x, f(\vec x))$, labeled sample set $\{(\vec z_i, f(\vec z_i))\}_{i=1}^m$, sparsity level $k$, step-size $\eta$, iteration count $T$}\;
    \BlankLine
    $\vec Z \gets (\vec z_1,\cdots,\vec z_m)$\;
    $\vec y \gets (f(\vec z_1),\cdots,f(\vec z_m))$\;
    $\vec w_0 \gets f(\vec x) \, \vec e_1$\;
    \BlankLine
    \For{$t = 1,2,\ldots,T$}
    {
        $\vec v_t  \gets \vec w_{t-1} - \tfrac{\eta}{m} \vec Z^{\top}(\vec Z\vec w_{t-1} - \vec y)$\;
        $\vec u_t \gets \vec v_t \odot \vec x$\;
        $\vec u^*_t \gets \textsc{gshp}(\vec u_t, k, f(\vec x))$\;
        $\vec w_t  \gets \vec u^*_t \odot \vec x$\;
    }
    \caption{Iterative Hard Thresholding Explainer}
    \label{alg:iht}
\end{algorithm}

The following result ensures that the projection step inside the loop operates in low polynomial time, circumventing the combinatorial explosion typically associated with sparse constrained optimization.

\begin{lemma}[Anchored Sparse Projection]
    \label{lem:projection}
    Let $(\vec x, f(\vec x)) \in \{-1,+1\}^d \times [-1,+1]$ be a labeled data instance and $k \geq 1$ be a sparsity level. Then, the Euclidean projection of any vector $\vec v \in \mathbb R^d$ onto $\mathcal{W}_{\vec{x}, k} = \mathcal H(\vec x, f(\vec x)) \cap \mathcal B_0(k)$ can be computed exactly in $\mathcal O(d \log d + k^2)$ time.
\end{lemma}
\begin{proof}
    As outlined in the projection phase of Algorithm~\ref{alg:iht}, the objective is to compute $\vec{w}_t = \Pi_{\mathcal{W}_{\vec{x}, k}}(\vec v_t)$.
    To leverage existing sparse projection techniques, we perform a change of variables by taking the Hadamard product $\vec u_t = \vec v_t \odot \vec x$. Let $\mathcal V$ denote the intersection of $\mathcal B_0(k)$ with the standard diagonal hyperplane $\mathcal H(\vec 1, f(\vec{x}))$.

    Because $\vec{x} \in \{-1, +1\}^d$, we have $\vec{x} \odot \vec{x} = \vec{1}$, meaning the Hadamard product is invertible via self-multiplication (i.e., $\vec{u}_t \odot \vec{x} = \vec{v}_t \odot \vec{x} \odot \vec{x} = \vec{v}_t$). Consequently, for any vector $\vec{w}'$ and its transformed counterpart $\vec{u}' = \vec{w}' \odot \vec{x}$, we have the algebraic equivalence $\vec u' \cdot \vec 1 = f(\vec{x})$ if and only if $\vec w' \cdot \vec x = f(\vec{x})$. This implies that $\vec u' \in \mathcal V$ if and only if $\vec w' \in \mathcal{W}_{\vec{x}, k}$.

    This structural equivalence, combined with the fact that Euclidean distance is preserved under this sign-flipping bijection ($\|\vec u' - \vec u_t\|_2 = \|\vec w' - \vec v_t\|_2$), implies that the target projection can be factored as:
    \begin{align*}
        \Pi_{\mathcal{W}_{\vec{x}, k}}(\vec v_t) = \left(\Pi_{\mathcal V}(\vec u_t)\right) \odot \vec x
    \end{align*}
    Let $\vec u^*_t$ be the projection of $\vec u_t$ onto $\mathcal V$. Since $\mathcal V$ is exactly the intersection of the $k$-sparse set with a diagonal hyperplane, computing $\vec u^*_t$ is an instance of the sparse hyperplane projection problem of \citet{Kyrillidis.ICML.2013}, whose \emph{Greedy Selector and Hyperplane Projector} (\textsc{gshp}) algorithm solves it \emph{exactly} (their Theorem 2) in $\mathcal O(d \log d + k^2)$ time. By setting $\vec w_t = \vec u^*_t \odot \vec x$, we obtain the exact projection of $\vec v_t$ onto $\mathcal{W}_{\vec{x}, k}$ within the stated time complexity limit, which concludes the proof.
\end{proof}

\subsection{Convergence Analysis}
\label{subsec:convergence}

Two design choices deserve a comment before we turn to the convergence analysis. First, the feasible set $\mathcal{W}_{\vec{x}, k}$ deliberately omits the box constraint $\|\vec{w}\|_{\infty} \leq 1$ that appears in the MIP formulation \eqref{pb:mip}. The reason is that exactness of the projection---the hypothesis on which the whole convergence argument rests---is available for $\mathcal{B}_0(k)$ intersected with a hyperplane, but not for the further intersection with a box: the correctness proof of \textsc{gshp} relies on a telescoping decomposition of the objective that clipping destroys, and \citet{Kyrillidis.ICML.2013} themselves show that the naive greedy selector already fails on the hyperplane variant. Second, and independently, the box remains present where it matters: we take as reference optimum
\begin{align*}
    \widehat{\vec{w}} \in \argmin \left\{ \widehat{\mathsf{F}}_{f, \vec{x}, m}(\vec{w}) \; : \; \vec{w} \in \mathcal{W}_{\vec{x}, k} \cap [-1,+1]^d \right\}
\end{align*}
that is, exactly the optimum targeted by \eqref{pb:mip}. The analysis only requires $\widehat{\vec{w}} \in \mathcal{W}_{\vec{x}, k}$, which holds; and comparing Algorithm~\ref{alg:iht} against the very same point as the MIP solver is what makes the two approaches directly commensurable.

All the results below rest on a single instance of Lemma~\ref{lem:rsc_rss}, which we establish as a standing assumption.
\begin{assumption}
    \label{ass:regularity}
    The centered matrix $\overline{\vec{Z}}$ satisfies properties (i) and (ii) of Lemma~\ref{lem:rsc_rss} at order $s = 3k$ with target condition scalar $\gamma = 1.25$---hence $\theta = 1/18$, $\alpha = \tfrac{17}{18}\lambda$ and $\beta = \tfrac{19}{18}\lambda$---and with a tolerance $\omega$ such that $\nu \, \omega \leq \beta/(16\sqrt{3k})$. Moreover, Algorithm~\ref{alg:iht} is run with the step-size $\eta = 1/\beta$.
\end{assumption}

The order $3k$ is the largest one the analysis requires: it arises in Step 1 of Lemma~\ref{lem:iht_convergence} below, where the supports of two consecutive iterates and of $\widehat{\vec{w}}$ must be considered jointly. Every other sparse direction appearing in this subsection is a difference of two anchored $k$-sparse explanations, hence $2k$-sparse; by the monotonicity noted in Definition~\ref{def:rcs_rss}, Assumption~\ref{ass:regularity} covers these as well, so we never invoke a second order. Lemma~\ref{lem:rsc_rss} instantiated at $s = 3k$ and $\gamma = 1.25$ shows that the assumption holds with probability at least $1 - \delta$ as soon as $m$ meets the corresponding sample size bound. The resulting restricted condition number $\beta/\alpha = 19/17 \approx 1.12$ sits comfortably below the threshold $16/9 \approx 1.78$ that Lemma~\ref{lem:iht_convergence} requires for contraction, and yields the factor $\rho = 51/152 \approx 0.34$. Fixing $\gamma$ in this way is what makes the constants of Theorem~\ref{thm:end_to_end} explicit; any other admissible choice would only change their numerical values.

Note also that the prescribed step-size is fully determined by the concentration parameter, since $\eta = 1/\beta = \tfrac{18}{19}\cosh^2(\sigma/2)$. Unlike most IHT variants, whose step-size must be tuned or estimated from the data, Algorithm~\ref{alg:iht} therefore requires no calibration: $\sigma$ is chosen by the user as part of the explanation query.

One further point deserves attention. In classical sparse recovery, the target is a stationary point of the unconstrained objective, so that its gradient vanishes and the iterates converge to it exactly. Here, $\widehat{\vec{w}}$ minimizes the empirical fidelity only within a constrained set, and there is no reason for its gradient to vanish: the model $f$ is not assumed to agree with any anchored $k$-sparse linear function on the sample. The natural quantity measuring this residual is the \emph{restricted gradient norm} at the optimum,
\begin{align}
    \label{eq:grad_floor}
    \xi = \max_{|S| \leq 3k} \left\| \nabla_S L(\widehat{\vec{w}}) \right\|_2, \qquad \text{where } L(\vec{w}) = \tfrac{1}{2m}\|\vec{Z}\vec{w} - \vec{y}\|_2^2 = 2 \, \widehat{\mathsf{F}}_{f, \vec{x}, m}(\vec{w})
\end{align}
and $\nabla_S$ denotes the gradient restricted to the coordinates in $S$. The scaling of $L$ is chosen so that $\nabla L(\vec{w}) = \tfrac{1}{m}\vec{Z}^{\top}(\vec{Z}\vec{w} - \vec{y})$ is exactly the gradient appearing in the update rule of Algorithm~\ref{alg:iht}; it differs from the normalized fidelity error of \eqref{eq:emp_fidelity} by the constant factor $2$. Accordingly, Algorithm~\ref{alg:iht} converges linearly not to $\widehat{\vec{w}}$ itself, but to a neighborhood of $\widehat{\vec{w}}$ whose radius is proportional to $\xi$. The degenerate case $\xi = 0$---which occurs exactly when $\widehat{\vec{w}}$ is also an unconstrained stationary point---recovers the classical exact linear convergence.

\begin{lemma}[Convergence of IHT]
    \label{lem:iht_convergence}
    Under Assumption~\ref{ass:regularity}, the iterates of Algorithm~\ref{alg:iht} satisfy, for every $t \geq 0$,
    \begin{align*}
        \|\vec{w}_t - \widehat{\vec{w}}\|_2 \leq \rho^{\,t} \|\vec{w}_0 - \widehat{\vec{w}}\|_2 + \frac{2\xi}{\beta(1 - \rho)}, \qquad \rho = 2\left(1 - \tfrac{\alpha}{\beta}\right) + \tfrac{1}{8} < 1
    \end{align*}
\end{lemma}
\begin{proof}
    Fix $t \geq 1$, let $\vec{h} = \vec{w}_{t-1} - \widehat{\vec{w}}$, and let $S = \mathrm{support}(\vec{w}_t) \cup \mathrm{support}(\vec{w}_{t-1}) \cup \mathrm{support}(\widehat{\vec{w}})$, so that $|S| \leq 3k$ and $\vec{h}$ is supported in $S$. Since $\vec{w}_{t-1}$ and $\widehat{\vec{w}}$ are both anchored, $\vec{h} \cdot \vec{x} = 0$. With $\vec{v}_t = \vec{w}_{t-1} - \eta \nabla L(\vec{w}_{t-1})$, the update rule reads $\vec{w}_t = \Pi_{\mathcal{W}_{\vec{x}, k}}(\vec{v}_t)$.

    \emph{Step 1: reduction to the support $S$.} Outside $S$, both $\vec{w}_t$ and $\widehat{\vec{w}}$ vanish, so the vectors $\vec{w}_t - \vec{v}_t$ and $\widehat{\vec{w}} - \vec{v}_t$ agree on $S^c$. Since $\widehat{\vec{w}} \in \mathcal{W}_{\vec{x}, k}$ and $\vec{w}_t$ is the \emph{exact} projection of $\vec{v}_t$ onto $\mathcal{W}_{\vec{x}, k}$ (Lemma~\ref{lem:projection}), we have $\|\vec{w}_t - \vec{v}_t\|_2 \leq \|\widehat{\vec{w}} - \vec{v}_t\|_2$, and subtracting the common $S^c$ contribution leaves $\|(\vec{w}_t - \vec{v}_t)_S\|_2 \leq \|(\widehat{\vec{w}} - \vec{v}_t)_S\|_2$. The triangle inequality then gives
    \begin{align*}
        \|\vec{w}_t - \widehat{\vec{w}}\|_2 = \|(\vec{w}_t - \widehat{\vec{w}})_S\|_2 \leq 2 \, \|(\vec{v}_t - \widehat{\vec{w}})_S\|_2
    \end{align*}

    \emph{Step 2: decomposition of the gradient step.} Writing $\widehat{\vec{H}} = \tfrac{1}{m}\overline{\vec{Z}}^{\top}\overline{\vec{Z}}$ and using $\nabla L(\vec{w}_{t-1}) = \nabla L(\widehat{\vec{w}}) + \tfrac{1}{m}\vec{Z}^{\top}\vec{Z}\vec{h}$ together with $\vec{Z}\vec{h} = \overline{\vec{Z}}\vec{h}$ (valid since $\vec{h} \cdot \vec{x} = 0$), we obtain
    \begin{align*}
        \tfrac{1}{m}\vec{Z}^{\top}\vec{Z}\vec{h} = \widehat{\vec{H}}\vec{h} + \nu \left( \tfrac{1}{m} \vec{1}^{\top} \overline{\vec{Z}} \vec{h} \right) \vec{x}
        \quad \Longrightarrow \quad
        \vec{v}_t - \widehat{\vec{w}} = \big(\vec{I} - \eta \widehat{\vec{H}}\big)\vec{h} - \eta \nu \left( \tfrac{1}{m} \vec{1}^{\top} \overline{\vec{Z}} \vec{h} \right)\vec{x} - \eta \nabla L(\widehat{\vec{w}})
    \end{align*}
    The single rank-one leakage term is aligned with $\vec{x}$; it is the only place where the non-centered part of the data matrix survives, and it vanishes identically at $\sigma = 0$.

    \emph{Step 3: bounding the three terms on $S$.} Since $\vec{h}$ is supported in $S$, property (i) at order $3k$ states that the eigenvalues of $\widehat{\vec{H}}_{SS}$ lie in $[\alpha, \beta]$; with $\eta = 1/\beta$, those of $\vec{I}_S - \eta \widehat{\vec{H}}_{SS}$ lie in $[0, 1 - \alpha/\beta]$, whence $\|((\vec{I} - \eta\widehat{\vec{H}})\vec{h})_S\|_2 \leq (1 - \alpha/\beta)\|\vec{h}\|_2$. For the leakage term, property (ii) gives $|\tfrac{1}{m}\vec{1}^{\top}\overline{\vec{Z}}\vec{h}| \leq \omega\|\vec{h}\|_2$, while $\|\vec{x}_S\|_2 = \sqrt{|S|} \leq \sqrt{3k}$; the assumption $\nu \omega \leq \beta/(16\sqrt{3k})$ therefore bounds it by $\tfrac{1}{16}\|\vec{h}\|_2$. The last term is at most $\eta \xi$ by definition of $\xi$. Combining with Step 1,
    \begin{align*}
        \|\vec{w}_t - \widehat{\vec{w}}\|_2 \leq \rho \|\vec{w}_{t-1} - \widehat{\vec{w}}\|_2 + \tfrac{2\xi}{\beta}, \qquad \rho = 2\left(1 - \tfrac{\alpha}{\beta}\right) + \tfrac{1}{8}
    \end{align*}
    The hypothesis $\beta < \tfrac{16}{9}\alpha$, guaranteed by Assumption~\ref{ass:regularity}, is exactly $\rho < 1$. Unrolling this recursion and summing the geometric series yields the stated bound.
\end{proof}

Two further properties of the trajectory are needed to close the analysis. The first bounds the residual $\xi$ by the optimal empirical fidelity itself, which turns the additive floor of Lemma~\ref{lem:iht_convergence} into a \emph{multiplicative} approximation factor. The second bounds the magnitude of the iterates, which is what allows the uniform convergence guarantee of Theorem~\ref{thm:sample_complexity} to be applied to them: the iterates are not confined to the unit box, so their $L_1$ norm---the radius $B$ of that theorem---must be controlled.

\begin{lemma}[Trajectory Bounds and Residual Gradient]
    \label{lem:trajectory}
    Under Assumption~\ref{ass:regularity}:
    \begin{enumerate}
        \item[(i)] $\xi \leq 2\sqrt{3k} \, \sqrt{\widehat{\mathsf{F}}_{f, \vec{x}, m}(\widehat{\vec{w}})} \leq 2\sqrt{3k}$;
        \item[(ii)] for every $t \geq 0$, $\; \widehat{\mathsf{F}}_{f, \vec{x}, m}(\vec{w}_t) \leq \widehat{\mathsf{F}}_{f, \vec{x}, m}(\vec{w}_0) \leq 1$ and $\;\|\vec{w}_t\|_1 \leq k + 4\sqrt{k/\alpha}$.
    \end{enumerate}
\end{lemma}
\begin{proof}
    (i) Let $\vec{r} = \vec{Z}\widehat{\vec{w}} - \vec{y}$, so that $\nabla L(\widehat{\vec{w}}) = \tfrac{1}{m}\vec{Z}^{\top}\vec{r}$ and $\tfrac{1}{m}\|\vec{r}\|_2^2 = 4 \, \widehat{\mathsf{F}}_{f, \vec{x}, m}(\widehat{\vec{w}})$. Since every entry of $\vec{Z}$ lies in $\{-1,+1\}$, each coordinate of the gradient satisfies $|\nabla_j L(\widehat{\vec{w}})| \leq \tfrac{1}{m}\|\vec{r}\|_1 \leq (\tfrac{1}{m}\|\vec{r}\|_2^2)^{1/2}$ by Cauchy--Schwarz; taking the $\ell_2$ norm over at most $3k$ coordinates gives the first inequality. The second follows from $\widehat{\mathsf{F}}_{f, \vec{x}, m}(\widehat{\vec{w}}) \leq \widehat{\mathsf{F}}_{f, \vec{x}, m}(\vec{w}_0) \leq 1$, since $\vec{w}_0$ is feasible for $\widehat{\vec{w}}$'s problem and $|\vec{w}_0 \cdot \vec{z}_i - f(\vec{z}_i)| \leq 2$, so that each loss term $\ell(\vec{w}_0 \cdot \vec{z}_i, f(\vec{z}_i))$ is at most $1$.

    (ii) \emph{Monotonicity.} Fix $t \geq 1$. The loss $L$ is quadratic and $\vec{w}_t - \vec{w}_{t-1}$ is a difference of two anchored $k$-sparse explanations, hence $2k$-sparse and orthogonal to $\vec{x}$; the $\beta$-RSS property of Assumption~\ref{ass:regularity} therefore applies to it and gives the standard descent inequality $L(\vec{w}_t) \leq L(\vec{w}_{t-1}) + \langle \nabla L(\vec{w}_{t-1}), \vec{w}_t - \vec{w}_{t-1}\rangle + \tfrac{\beta}{2}\|\vec{w}_t - \vec{w}_{t-1}\|_2^2$. Completing the square with $\eta = 1/\beta$ rewrites the right-hand side as $L(\vec{w}_{t-1}) + \tfrac{\beta}{2}\big(\|\vec{w}_t - \vec{v}_t\|_2^2 - \|\vec{w}_{t-1} - \vec{v}_t\|_2^2\big)$, and the bracket is non-positive because $\vec{w}_{t-1} \in \mathcal{W}_{\vec{x}, k}$ while $\vec{w}_t$ is the exact projection of $\vec{v}_t$ onto $\mathcal{W}_{\vec{x}, k}$. Hence $L(\vec{w}_t) \leq L(\vec{w}_{t-1})$, and the claim follows by induction since $\widehat{\mathsf{F}}_{f, \vec{x}, m} = L/2$.

    \emph{Magnitude.} The difference $\vec{h} = \vec{w}_t - \widehat{\vec{w}}$ is likewise $2k$-sparse and orthogonal to $\vec{x}$, so the $\alpha$-RSC property and the elementary inequality $\|\vec{a} - \vec{b}\|_2^2 \leq 2\|\vec{a}\|_2^2 + 2\|\vec{b}\|_2^2$ applied to the two residuals give
    \begin{align*}
        \alpha \|\vec{h}\|_2^2 \leq \tfrac{1}{m}\|\vec{Z}\vec{h}\|_2^2 \leq 8 \left( \widehat{\mathsf{F}}_{f, \vec{x}, m}(\vec{w}_t) + \widehat{\mathsf{F}}_{f, \vec{x}, m}(\widehat{\vec{w}}) \right) \leq 16
    \end{align*}
    Therefore $\|\vec{w}_t\|_2 \leq \|\widehat{\vec{w}}\|_2 + \|\vec{h}\|_2 \leq \sqrt{k} + 4/\sqrt{\alpha}$, and since $\vec{w}_t$ is $k$-sparse, $\|\vec{w}_t\|_1 \leq \sqrt{k}\|\vec{w}_t\|_2$.
\end{proof}

Because $1/\alpha = \cosh^2(\sigma/2)/(1-\theta)$, part (ii) bounds the $L_1$ norm of every iterate by $B = k + \mathcal{O}(\sqrt{k} \cosh(\sigma/2))$. This is the radius that enters Theorem~\ref{thm:sample_complexity}, the box constraint of \eqref{pb:mip} corresponding to the particular case $B = k$. Applying that result with the larger radius leaves the generalization guarantee intact, and extends it to the ``unclipped'' IHT iterates. We are now in a position to state the main result of this section.

\begin{theorem}[End-to-End Guarantee]
    \label{thm:end_to_end}
    Let $\sigma \geq 0$ and $k \geq 1$, and suppose Algorithm~\ref{alg:iht} is run with the step-size $\eta = 1/\beta$ prescribed by Assumption~\ref{ass:regularity}. Let $\xi$, $\rho$ and $B = k + 4\sqrt{k/\alpha}$ be as above. For any tolerance $\varepsilon \in (0,1]$ and confidence level $\delta \in (0,1]$, if the sample size satisfies both the condition of Lemma~\ref{lem:rsc_rss} at order $3k$ with confidence $\delta/2$ and
    \begin{align*}
    m \geq \frac{(B + 1)^4}{\varepsilon^2} \left(16 \ln(2d) + \ln\Bigl(\tfrac{4}{\delta}\Bigr)\right)
    \end{align*}
    and if the iteration count satisfies
    \begin{align*}
        T \geq \frac{1}{\ln(1/\rho)} \ln\left( \frac{\sqrt{k}+1}{\vartheta} \right), \qquad \vartheta = \min\left\{ \tfrac{\varepsilon}{2\xi}, \sqrt{\tfrac{\varepsilon}{2\beta}} \right\}
    \end{align*}
    with the convention $\varepsilon/0 = +\infty$, then, with probability at least $1 - \delta$, the returned explanation $\vec{w}_T$ satisfies
    \begin{align*}
        \mathsf{R}_{f, \vec{x}, \mathcal{D}_{\vec{x}, \sigma}}(\vec{w}_T) \leq (1 + e^{-\sigma})^k \left( \left(1 + 81 k \cosh^2(\sigma/2)\right) \widehat{\mathsf{F}}_{f, \vec{x}, m}(\widehat{\vec{w}}) + \varepsilon \right)
    \end{align*}
\end{theorem}
\begin{proof}
    By Lemma~\ref{lem:rsc_rss} at order $3k$ with confidence $\delta/2$, the first condition on $m$ ensures that Assumption~\ref{ass:regularity} holds with probability at least $1 - \delta/2$; we work on that event throughout and account for the remaining $\delta/2$ below. Note that $\alpha$, $\beta$, $B$ and $\eta$ are deterministic functions of $\sigma$ and $k$ alone, so no circularity arises.

    We now combine the ingredients in turn. First, by Lemma~\ref{lem:fidelity},
    \begin{align}
        \label{eq:e2e_step1}
        \mathsf{R}_{f, \vec{x}, \mathcal{D}_{\vec{x}, \sigma}}(\vec{w}_T) \leq (1 + e^{-\sigma})^k \, \mathsf{F}_{f, \vec{x}, \mathcal{D}_{\vec{x}, \sigma}}(\vec{w}_T)
    \end{align}
    Second, every iterate satisfies $\|\vec{w}_T\|_1 \leq B$ by Lemma~\ref{lem:trajectory}(ii), so Theorem~\ref{thm:sample_complexity} applied with radius $B$, tolerance $\varepsilon/2$ and confidence $\delta/2$---which is exactly the second displayed condition on $m$---gives, with probability at least $1 - \delta/2$,
    \begin{align}
        \label{eq:e2e_step2}
        \mathsf{F}_{f, \vec{x}, \mathcal{D}_{\vec{x}, \sigma}}(\vec{w}_T) \leq \widehat{\mathsf{F}}_{f, \vec{x}, m}(\vec{w}_T) + \tfrac{\varepsilon}{2}
    \end{align}
    Third, write $\Delta_T = \|\vec{w}_T - \widehat{\vec{w}}\|_2$. The difference $\vec{w}_T - \widehat{\vec{w}}$ is $2k$-sparse and orthogonal to $\vec{x}$, so $\widehat{\mathsf{F}}_{f, \vec{x}, m} = L/2$ is $\tfrac{\beta}{2}$-smooth along it and a second-order expansion around $\widehat{\vec{w}}$ gives
    \begin{align*}
        \widehat{\mathsf{F}}_{f, \vec{x}, m}(\vec{w}_T) - \widehat{\mathsf{F}}_{f, \vec{x}, m}(\widehat{\vec{w}}) \leq \big\langle \tfrac{1}{2}\nabla L(\widehat{\vec{w}}), \vec{w}_T - \widehat{\vec{w}} \big\rangle + \tfrac{\beta}{4} \Delta_T^2 \leq \tfrac{1}{2} \xi \Delta_T + \tfrac{\beta}{4} \Delta_T^2
    \end{align*}
    the first-order term being nonzero precisely because $\widehat{\vec{w}}$ is only a constrained optimum. By Lemma~\ref{lem:iht_convergence}, $\Delta_T \leq \vartheta_T + \Delta_{\infty}$ with $\vartheta_T = \rho^{\,T}\|\vec{w}_0 - \widehat{\vec{w}}\|_2$ and $\Delta_{\infty} = 2\xi/(\beta(1-\rho))$. Applying $(a+b)^2 \leq 2a^2 + 2b^2$ and separating the two contributions,
    \begin{align*}
        \widehat{\mathsf{F}}_{f, \vec{x}, m}(\vec{w}_T) - \widehat{\mathsf{F}}_{f, \vec{x}, m}(\widehat{\vec{w}}) \leq \underbrace{\tfrac{1}{2}\xi\Delta_{\infty} + \tfrac{\beta}{2}\Delta_{\infty}^2}_{\leq \; 3 \xi^2 / (\beta(1-\rho)^2)} \; + \; \underbrace{\tfrac{1}{2}\xi\vartheta_T + \tfrac{\beta}{2}\vartheta_T^2}_{\leq \; \varepsilon/2 \text{ when } \vartheta_T \leq \vartheta}
    \end{align*}
    Since $\|\vec{w}_0\|_2 \leq 1$ and $\|\widehat{\vec{w}}\|_2 \leq \sqrt{k}$, the stated iteration count guarantees $\vartheta_T \leq \vartheta$; and Lemma~\ref{lem:trajectory}(i) bounds the first bracket by $36k\,\widehat{\mathsf{F}}_{f, \vec{x}, m}(\widehat{\vec{w}})/(\beta(1-\rho)^2)$. Inserting the values $\beta = \tfrac{19}{18}\lambda$ and $\rho = \tfrac{51}{152}$ fixed by Assumption~\ref{ass:regularity}, we can upper-bound the multiplicative factor $36k/(\beta(1-\rho)^2)$ by $81k/\lambda = 81k \cosh^2(\sigma/2)$. Chaining with \eqref{eq:e2e_step2} and \eqref{eq:e2e_step1}, and taking a union bound over the two events, concludes the proof.
\end{proof}

The iteration count of Theorem~\ref{thm:end_to_end} involves the residual $\xi$, which is not known before running the algorithm. It can nonetheless be made explicit: Lemma~\ref{lem:trajectory}(i) gives $\xi \leq 2\sqrt{3k}$, so that running Algorithm~\ref{alg:iht} for
\begin{align*}
    T \geq \frac{1}{\ln(1/\rho)} \ln\left( \frac{\sqrt{k}+1}{\vartheta_0} \right), \qquad \vartheta_0 = \min\left\{ \frac{\varepsilon}{4\sqrt{3k}}, \; \sqrt{\frac{\varepsilon}{2\beta}} \right\}
\end{align*}
iterations always suffices. This count depends only on $k$, $\sigma$ and $\varepsilon$, all of which are available upfront, and each iteration costs $\mathcal{O}(md + d\log d + k^2)$ time---dominated by the gradient step for any nontrivial sample size.

The guarantee has a simple reading: Algorithm~\ref{alg:iht} recovers the optimum of the MIP formulation \eqref{pb:mip} up to a multiplicative factor, and recovers it exactly in the limit where the model is genuinely explainable by an anchored $k$-sparse linear function on the neighborhood of $\vec{x}$. A multiplicative guarantee of this kind is the expected shape for a polynomial-time method applied to an \ClassNP-hard problem; what the analysis adds is that the factor is explicit, and that the number of iterations required is only logarithmic in $1/\varepsilon$.

The multiplier $81k \cosh^2(\sigma/2)$ in Theorem~\ref{thm:end_to_end} is a strict, worst-case analytical bound driven by two structural parameters. The dependence on $k$ originates from a single union bound in Lemma~\ref{lem:trajectory}(i), which pessimistically assumes that all $3k$ coordinates of the restricted gradient are simultaneously extremal and that no cancellation occurs between the residual and the Boolean columns of $\vec{Z}$. The dependence on $\cosh^2(\sigma/2)$ reflects the vanishing curvature of the data matrix (the $1/\lambda$ term in our RSC bound) as the distribution localizes. In practice, such adversarial alignment is exceedingly rare. As demonstrated in Section~\ref{sec:experiments}, the empirical gap between the IHT iterate $\vec{w}_T$ and the exact MIP optimum $\widehat{\vec{w}}$ is orders of magnitude smaller than this worst-case factor suggests. The primary value of the theorem is therefore qualitative and structural: it guarantees that the approximation ratio is finite, explicit, and vanishes as the optimum empirical fidelity approaches zero.

Overall, Theorem~\ref{thm:end_to_end} captures the core operational trade-off discussed in Section~\ref{sec:introduction}. Locality enters the sample complexity twice: through the restricted eigenvalue condition of Lemma~\ref{lem:rsc_rss}, in $\cosh^4(\sigma/2)$, and through the magnitude of the iterates, in $(B+1)^4 = \mathcal{O}(k^4 + k^2 \cosh^4(\sigma/2))$. Notably, the second term only starts to dominate once $\cosh^2(\sigma/2)$ exceeds $k$, so that for the concentration parameters of practical interest the sample complexity retains the $k^4$ behavior of Theorem~\ref{thm:sample_complexity}. This gives rise to two distinct regimes of algorithmic effectiveness:
\begin{itemize}
    \item \emph{Localized Regime (Large $\sigma$):} When the distribution is highly concentrated around $\vec{x}$, the relevance-fidelity factor $(1 + e^{-\sigma})^k$ from Lemma~\ref{lem:fidelity} approaches $1$, aligning the empirical fidelity objective almost perfectly with the true relevance error. In this setting, the exact MIP formulation from Section~\ref{subsec:empirical_mip} is the method of choice, as the sample complexity of Theorem~\ref{thm:sample_complexity} does not depend on $\sigma$ at all---its cost being the worst-case exponential runtime of the solver rather than the number of samples. In contrast, IHT suffers a double theoretical penalty: its sample complexity grows with $\cosh^4(\sigma/2)$, and its approximation guarantee degrades with $\cosh^2(\sigma/2)$, making the gradient-based approach eventually infeasible both statistically and computationally.
    \item \emph{Intermediate Regime (Moderate $\sigma$):} For moderate $\sigma$, this factor remains small enough to allow for a practical sample size $m$. In this ``sweet spot,'' the data matrix geometry is regular enough to guarantee linear convergence, and the relevance-fidelity factor remains close enough to $1$ for high-quality explanations. In this regime, the IHT explainer emerges as a scalable, strictly polynomial-time alternative, avoiding the worst-case exponential runtime associated with the \ClassNP-hard exact MIP formulation.
\end{itemize}
Ultimately, this theoretical duality shows that formal explainability cannot rely on a one-size-fits-all approach: the choice between exact combinatorial search and approximate gradient-based methods should be guided by the locality properties of the underlying distribution.


\section{Experiments}
\label{sec:experiments}

This section compares our two explainers with LIME, MAPLE, and a convex relaxation of \eqref{pb:mip} across 16 OpenML benchmarks in both classification and regression. Section~\ref{sec:setup} details the evaluation protocol. Section~\ref{subsec:exp_optimization} examines how effectively each explainer minimizes empirical fidelity error and at what cost to admissibility: MIP and IHT yield statistically indistinguishable results, while LIME’s occasional fidelity advantage stems from its systematic violation of the anchoring condition. Section~\ref{subsec:exp_generalization} addresses the metrics most relevant to users, demonstrating that empirical error reliably predicts out-of-sample performance, and that the ranking reverses when considering relevance---the objective our framework optimizes. Parameter sweeps over sparsity budget, locality, and sample size confirm that this reversal is robust, and that no amount of data can compensate for an unanchored explanation.

\subsection{Experimental Setup}
\label{sec:setup}

All experiments were implemented in \texttt{Python}, and executed on a computing node equipped with a 24-core Intel Xeon Gold 6252 processor (2.10 GHz), 
384 GB RAM, and an NVIDIA Quadro RTX 8000 GPU (48 GB VRAM). The full source code, environment specifications, 
and scripts for reproducing these experiments are publicly available.\footnote{Source code and replication scripts: \url{https://github.com/FredericKoriche/ProbabilisticLinearExplanations.git}}

\paragraph{Datasets.}
We evaluate our methods on 16 tabular datasets from the \texttt{OpenML} repository, as summarized in Table~\ref{tab:datasets}.
These datasets span a diverse range of application domains, including biology (\textit{Cancer Drug Response Methylation}, \textit{Musk}, \textit{NCI 60 Thioguanine}),
healthcare (\textit{Heart Disease}, \textit{Medical Charges}, \textit{Postoperative}), environment (\textit{Forest Fires}, \textit{Seoul Bike Sharing}, \textit{Nomao}),
sociology (\textit{Speed Dating}, \textit{Student Performance}), finance (\textit{Adult}, \textit{California Housing}, \textit{Credit Approval}),
and law (\textit{COMPAS}, \textit{Communities and Crime}).

The datasets are divided into two groups: the first eight correspond to binary classification tasks, while the remaining eight are used for continuous regression. To create the Boolean hypercube representation required by our framework, we encode each attribute as a set of indicator columns, mapping $\{0, 1\}$ to $\{-1, +1\}$. Attributes with at most $K = 4$ distinct values are one-hot encoded regardless of type, so a numerical attribute with few values is left undiscretized. Categorical attributes with more than $K$ categories are reduced to the $K - 1$ most frequent, with all others grouped into a residual category. Numerical attributes with more than $K$ values are discretized into $K$ quantile-based bins using a standard $K$-bins strategy. We discard constant columns, as well as numerical attributes whose quantile edges are not all distinct, a situation that can occur for highly skewed or zero-inflated variables. This explains why the indicator counts in Table~\ref{tab:datasets} may be smaller than the raw attribute count from the original data. Each discretized numerical attribute contributes exactly $K$ indicators, while a retained categorical attribute contributes at most $K$; the quantity $d - (K \cdot \textit{Num})$ therefore measures the contribution of the categorical attributes alone, which ranges from two indicators per attribute (\textit{COMPAS}) to three (\textit{Speed Dating}). For regression tasks, target variables are rescaled to $[-1, +1]$ to match the assumed codomain. After preprocessing, the average dimensionality is $d \approx 220$ (median $44$), ranging from $d = 12$ (\textit{Medical Charges}) to $d = 1652$ (\textit{Cancer Drug Response Methylation}).

\begin{table}[ht]
\centering
\caption{Overview of the 16 OpenML benchmark datasets utilized in our experiments. \emph{Total Raw} denotes the number of original attributes in each OpenML file; \emph{Cat} and \emph{Num} indicate the counts of categorical and numerical attributes retained after preprocessing, encoded via one-hot encoding and $K = 4$ quantile-based binning, respectively. \emph{Dim} represents the resulting Boolean dimension. The first eight datasets are classification tasks, and the last eight are regression tasks. The final column shows the neural network's generalization error, assessed by 5-fold cross-validation using the normalized quadratic loss $\ell(y, \hat{y}) = \frac{1}{4}(y - \hat{y})^2$.}
\label{tab:datasets}
\begin{tabular}{lrrrrrrr}
\toprule
Dataset & ID & Instances & Cat & Num & Total Raw & Dim ($d$) & CV Loss \\
\midrule
\multicolumn{8}{l}{\textbf{Classification}} \\
Credit Approval & 29 & 690 & 4 & 2 & 15 & 17 & 0.145 \\
COMPAS & 42192 & 5278 & 8 & 1 & 13 & 20 & 0.405 \\
Postoperative & 40683 & 88 & 8 & 0 & 8 & 23 & 0.373 \\
Heart Disease & 43672 & 1190 & 6 & 5 & 11 & 37 & 0.111 \\
Adult & 179 & 48842 & 9 & 2 & 14 & 42 & 0.163 \\
Nomao & 1486 & 34465 & 41 & 6 & 118 & 149 & 0.054 \\
Speed Dating & 40536 & 8378 & 58 & 2 & 120 & 180 & 0.169 \\
Musk & 1116 & 6598 & 1 & 160 & 167 & 644 & 0.015 \\
\midrule
\multicolumn{8}{l}{\textbf{Regression}} \\
Medical Charges & 44146 & 163065 & 0 & 3 & 3 & 12 & 0.003 \\
California Housing & 44024 & 20640 & 0 & 8 & 8 & 32 & 0.020 \\
Forest Fires & 44962 & 517 & 2 & 8 & 12 & 40 & 0.005 \\
Seoul Bike Sharing & 46328 & 8760 & 5 & 8 & 17 & 46 & 0.008 \\
Student Performance & 42352 & 395 & 20 & 4 & 32 & 69 & 0.021 \\
NCI 60 Thioguanine & 46132 & 60 & 0 & 47 & 48 & 188 & 0.027 \\
Communities and Crime & 46286 & 1994 & 1 & 91 & 122 & 367 & 0.030 \\
Cancer Drug Resp. Meth. & 46139 & 475 & 0 & 413 & 808 & 1652 & 0.035 \\
\bottomrule
\end{tabular}
\end{table}

\paragraph{Models.}
For each dataset, we employ a Multi-Layer Perceptron (MLP) implemented using \texttt{Scikit-Learn} as the black-box model $f$. For classification, we use the default \texttt{MLPClassifier} settings: a single hidden layer of 100 neurons, ReLU activation, the Adam optimizer, and up to 200 training iterations; the model outputs its predicted label, mapped to $\{-1, +1\}$. For regression, we use \texttt{MLPRegressor} with three hidden layers (150, 100, and 50 neurons), ReLU activation, the Adam optimizer, and up to 500 training iterations; the real-valued outputs are clipped to $[-1, +1]$ to ensure $f$ matches the assumed codomain. Model quality is evaluated using the normalized quadratic loss $\ell(y, \hat{y}) = \frac{1}{4}(y - \hat{y})^2$, which, in the binary case, coincides with the zero-one error. The resulting 5-fold cross-validation estimates are given in the final column of Table~\ref{tab:datasets}.

\paragraph{Explanation Tasks.}
In our experiments, an \emph{explanation task} is defined by a tuple $(f, \vec{x}, k, m, \sigma)$, where $f$ denotes the neural network trained on the benchmark and $\vec{x}$ is a reference instance sampled uniformly at random from the test set. Sampling from $\mathcal D_{\vec x, \sigma}$ is exact and requires no burn-in: conditioned on $\vec{x}$, the coordinates of $\vec{z}$ are mutually independent, each retaining its reference value $x_i$ with probability $1 / (1 + e^{-\sigma})$ and flipping with probability $p_\sigma = e^{-\sigma} / (1 + e^{-\sigma})$. This interpretation provides an intuitive understanding of the locality parameter: $\sigma = 0$ yields the uniform distribution over the hypercube ($p_\sigma = \frac{1}{2}$), $\sigma = 1$ perturbs about a quarter of the coordinates ($p_\sigma \approx 0.27$), and $\sigma = 1.75$ perturbs about one coordinate in seven ($p_\sigma \approx 0.15$). Our experiments span small sparsity budgets ($k \leq 8$), locality regimes ranging from uniform to sharply concentrated neighborhoods, and sample sizes up to $m = 10^5$; the specific grid for each sweep is detailed in the corresponding subsection. For each configuration, all explainers receive the same sample so that performance differences reflect optimization rather than sampling variability. Finally, because the difficulty of an explanation task depends on the reference instance, we select 10 independent instances $\vec{x}$ per benchmark and, for each configuration $(f, k, m, \sigma)$, report the mean and dispersion of each explainer's performance across these instances.

\paragraph{Explainers.}
To solve the explanation tasks, we employ two explainers: MIP, which solves the mixed-integer formulation \eqref{pb:mip}, and IHT, a scalable Iterative Hard Thresholding algorithm. The MIP formulation is solved using the \texttt{Gurobi} solver with default parameters, imposing a strict 120-second time limit. If the time limit is reached, the solver returns its incumbent solution, which remains admissible but may not be optimal; we therefore report the frequency of such timeouts alongside the affected results. The IHT explainer is implemented in \texttt{PyTorch}, leveraging \texttt{CUDA}-accelerated tensor operations for efficient gradient computations and sparse projections. Consequently, wall-clock comparisons between the two methods should be interpreted with the CPU/GPU asymmetry in mind. For IHT, we use the constant step size $\eta = \frac{18}{19} \cosh^{2}(\sigma / 2)$, which is exactly the inverse restricted smoothness constant $1 / \beta$ prescribed by Assumption~\ref{ass:regularity}. The algorithm is run for 5000 iterations, initialized at the admissible explanation $\vec{w}_0 = f(\vec{x}) \, \vec{e}_1$, in accordance with our convergence analysis.

We compare our MIP and IHT methods with three complementary baselines. The first, \MIPRelax, serves as an ablation baseline by implementing the continuous convex relaxation of \eqref{pb:mip}. This approach is modeled in \texttt{CVXPY} and solved using the \texttt{Gurobi} backend. \MIPRelax substitutes the combinatorial sparsity constraint with the continuous surrogate $\|\vec{w}\|_1 \leq k$, while directly enforcing the anchoring hyperplane $\vec{w} \cdot \vec{x} = f(\vec{x})$ and the variable bounds $\vec{w} \in [-1, +1]^d$. The other two baselines are state-of-the-art model-agnostic explainers: LIME \citep{Ribeiro.KDD.2016} and MAPLE \citep{Plumb.NeurIPS.2018}.\footnote{Because the official MAPLE repository has not been updated since its original release, we manually ported its source code to Python 3.14 to ensure compatibility with our experimental environment. The port, included in our replication code, contains only compatibility fixes.}

Neither LIME nor MAPLE was developed for our hypothesis space, and as such, neither is guaranteed to produce an admissible explanation: the anchoring condition $\vec{w} \cdot \vec{x} = f(\vec{x})$ is not inherently enforced by either method, and while LIME permits a hard sparsity constraint $\|\vec{w}\|_0 \leq k$, MAPLE does not. We intentionally retain both methods at their default algorithmic settings to assess them as originally published. As a result, they optimize a less constrained objective and may sometimes report lower fidelity errors than methods restricted to $\mathcal{W}_{\vec{x}, k}$---a benefit that is only meaningful if the explanation is admissible. For each configuration, we therefore report the rate at which each baseline violates the anchoring and sparsity constraints, and interpret its fidelity accordingly.

\subsection{Optimization Quality and Admissibility}
\label{subsec:exp_optimization}

This initial set of experiments focuses on the optimization aspect: given a fixed sample, how effectively does each explainer minimize the empirical fidelity error $\widehat{\mathsf{F}}$, and at what cost to admissibility and computation? We set the sparsity budget to $k = 5$, the sample size to $m = 5000$, and the locality parameter to $\sigma = 1.0$, running all five explainers on ten reference instances for each of the 16 benchmarks---resulting in 160 explanation tasks in total. Table~\ref{tab:fidelity} presents the resulting fidelity errors, with datasets ordered by dimensionality $d$ as in Table~\ref{tab:datasets}. To avoid multiple tables, support sizes, anchoring violation rates, and runtimes are discussed in the main text and shown in Figure~\ref{fig:runtime}.

\begin{table}[ht]
\centering
\caption{Empirical fidelity error ($\widehat{\mathsf{F}}$) across classification and regression benchmarks ($k=5$, $m = 5000$, $\sigma=1.0$), ordered by dimensionality ($d$). The last two columns report methods that do not satisfy the sparsity budget; \MIPRelax\ is a relaxation, and its value is a valid lower bound on the optimum of \eqref{pb:mip} rather than an achievable score.}
\label{tab:fidelity}
\begin{tabular}{lrrrrr}
\toprule
 & \multicolumn{3}{c}{Sparse ($\|\vec{w}\|_0 \leq k$)} & \multicolumn{2}{c}{Dense} \\
\cmidrule(lr){2-4} \cmidrule(lr){5-6}
Dataset & IHT & MIP & LIME & \MIPRelax & MAPLE \\
\midrule
\multicolumn{6}{l}{\textbf{Classification}} \\
Credit Approval & $0.107 \pm 0.010$ & $0.107 \pm 0.010$ & $0.108 \pm 0.009$ & $0.090 \pm 0.005$ & $0.284 \pm 0.060$ \\
COMPAS & $0.105 \pm 0.068$ & $0.104 \pm 0.068$ & $0.094 \pm 0.048$ & $0.081 \pm 0.044$ & $0.199 \pm 0.179$ \\
Postoperative & $0.141 \pm 0.048$ & $0.141 \pm 0.048$ & $0.130 \pm 0.027$ & $0.108 \pm 0.024$ & $0.246 \pm 0.102$ \\
Heart Disease & $0.162 \pm 0.029$ & $0.162 \pm 0.030$ & $0.162 \pm 0.026$ & $0.121 \pm 0.015$ & $0.225 \pm 0.046$ \\
Adult & $0.152 \pm 0.044$ & $0.152 \pm 0.044$ & $0.144 \pm 0.029$ & $0.114 \pm 0.021$ & $0.269 \pm 0.066$ \\
Nomao & $0.226 \pm 0.064$ & $0.224 \pm 0.063$ & $0.190 \pm 0.039$ & $0.124 \pm 0.023$ & $0.188 \pm 0.036$ \\
Speed Dating & $0.205 \pm 0.070$ & $0.203 \pm 0.071$ & $0.181 \pm 0.044$ & $0.126 \pm 0.032$ & $0.179 \pm 0.041$ \\
Musk & $0.124 \pm 0.074$ & $0.124 \pm 0.074$ & $0.118 \pm 0.062$ & $0.077 \pm 0.040$ & $0.095 \pm 0.035$ \\
\midrule
\multicolumn{6}{l}{\textbf{Regression}} \\
Medical Charges & $0.001 \pm 0.000$ & $0.001 \pm 0.000$ & $0.001 \pm 0.000$ & $0.001 \pm 0.000$ & $0.013 \pm 0.008$ \\
California Housing & $0.020 \pm 0.002$ & $0.020 \pm 0.001$ & $0.019 \pm 0.001$ & $0.012 \pm 0.001$ & $0.027 \pm 0.003$ \\
Forest Fires & $0.005 \pm 0.000$ & $0.005 \pm 0.000$ & $0.004 \pm 0.000$ & $0.003 \pm 0.000$ & $0.005 \pm 0.001$ \\
Seoul Bike Sharing & $0.015 \pm 0.009$ & $0.015 \pm 0.009$ & $0.010 \pm 0.002$ & $0.007 \pm 0.002$ & $0.014 \pm 0.006$ \\
Student Performance & $0.012 \pm 0.003$ & $0.012 \pm 0.003$ & $0.011 \pm 0.001$ & $0.007 \pm 0.000$ & $0.017 \pm 0.004$ \\
NCI 60 Thioguanine & $0.022 \pm 0.021$ & $0.022 \pm 0.021$ & $0.012 \pm 0.001$ & $0.007 \pm 0.001$ & $0.009 \pm 0.002$ \\
Communities and Crime & $0.023 \pm 0.006$ & $0.023 \pm 0.006$ & $0.018 \pm 0.003$ & $0.011 \pm 0.001$ & $0.014 \pm 0.002$ \\
Cancer Drug Resp. Meth. & $0.019 \pm 0.009$ & $0.018 \pm 0.008$ & $0.009 \pm 0.001$ & $0.003 \pm 0.000$ & $0.008 \pm 0.001$ \\
\bottomrule
\end{tabular}
\end{table}

\paragraph{Admissibility comes first.}
Two explainers can be fairly compared on $\widehat{\mathsf{F}}$ only if they search the same solution space, so we begin by reporting where each explainer actually lands. IHT, MIP, and LIME consistently yield exactly five nonzero coefficients across all 160 tasks. By construction, \MIPRelax\ produces dense solutions: its $L_1$ surrogate disperses weight across the entire hypercube, resulting in supports that range from 13 coefficients on \textit{Medical Charges} to $1412.6 \pm 54.8$ on \textit{Cancer Drug Resp.\ Meth.} MAPLE is also dense, but much less predictable---its support size varies from a single coefficient on \textit{Medical Charges} to $598.0 \pm 648.7$ on \textit{Cancer Drug Resp.\ Meth.}, with standard deviations often matching the mean. This variability makes it difficult for users to anticipate the length of the returned explanation.

The anchoring constraint clearly distinguishes the methods. IHT, MIP, and \MIPRelax\ enforce this condition structurally and satisfy it exactly on every task. LIME violates the constraint in all 160 cases: its average gap $|\vec{w} \cdot \vec{x} - f(\vec{x})|$ ranges from $0.026$ on \textit{Medical Charges} to $0.783$ on \textit{Nomao}, and in 7 out of 80 classification tasks, it exceeds $1.0$---meaning the surrogate at $\vec{x}$ is further from $f(\vec{x})$ than the distance between the two classes. MAPLE violates anchoring on $84\%$ of the tasks; its only clean benchmark, \textit{Medical Charges}, is also where it collapses to a one-coefficient constant model and where its fidelity error is an order of magnitude higher than all other methods.

By contrast, one constraint is rarely active. Across all 160 tasks, the largest coefficient produced by IHT is $0.935$, and by MIP is $0.956$. For MIP, this indicates that the bound $\|\vec{w}\|_\infty \leq 1$ is simply inactive at the optimum. For IHT, the measurement reveals something more: Section~\ref{subsec:convergence} deliberately omits the box from the feasible set, retaining it only on the reference optimum $\widehat{\vec{w}}$, so the iterates are free to leave it---but they never do. This is the empirical counterpart of Lemma~\ref{lem:trajectory}(ii), which controls the magnitude without imposing it. Only \MIPRelax\ reaches the bound exactly, as expected for an $L_1$ surrogate that distributes weight across many coordinates. LIME, though unconstrained in this regard, also stays within the box on every task (maximum $0.917$), whereas MAPLE exceeds the bound on four occasions, reaching up to $1.270$. At these sparsity levels, the box constraint does not distinguish between the explainers; instead, the admissibility question reduces to sparsity and anchoring.

\paragraph{IHT matches the exact solver.}
With these considerations in mind, Table~\ref{tab:fidelity} reveals that our two explainers are nearly indistinguishable in terms of optimization quality. Their fidelity errors match to three decimal places on 12 out of 16 benchmarks. On the remaining four, IHT lags by at most $6\%$ in relative terms (\textit{Cancer Drug Resp.\ Meth.}, $0.019$ versus $0.018$)---a difference smaller than the natural variability across reference instances. Furthermore, for three of these four benchmarks, the MIP solver was halted by the time limit, meaning its reported value is merely the incumbent rather than a certified optimum. Thus, the apparent advantage is not even definitive.

The \MIPRelax\ column serves a distinct purpose. Any $\vec{w}$ meeting $\|\vec{w}\|_0 \leq k$ and $\|\vec{w}\|_\infty \leq 1$ automatically satisfies $\|\vec{w}\|_1 \leq k$, so the relaxation expands the feasible set and its optimum provides a valid lower bound for the optimum of \eqref{pb:mip}. This makes the ablation informative: on benchmarks where MIP cannot certify optimality, the two columns bracket the unknown optimum---for example, on \textit{Musk}, it lies within $[0.077, 0.124]$. Although this bracket is loose, as expected from an $L_1$ surrogate at $k = 5$, it demonstrates that neither explainer is leaving significant fidelity unachieved.

\paragraph{LIME's advantage is an anchoring artifact.}
LIME reports a lower fidelity error than MIP on several benchmarks---for example, $0.094$ versus $0.104$ on \textit{COMPAS} and $0.144$ versus $0.152$ on \textit{Adult}. This result warrants an explanation rather than a disclaimer, and the answer is straightforward. When the MIP is solved to optimality, it achieves the minimum $\widehat{\mathsf{F}}$ over $\mathcal{W}_{\vec{x}, k}$; thus, any $\vec{w}$ with a strictly lower score must fall outside that feasible set. Since LIME adheres to the sparsity budget and, as discussed above, remains well within the box constraint, the only remaining explanation is anchoring. This advantage does not indicate that LIME is a better optimizer---it simply means LIME is optimizing a different, less constrained problem.

Figure~\ref{fig:anchoring_correlation} quantifies this trade-off. For each task, we plot LIME's relative fidelity advantage over MIP, $1 - \widehat{\mathsf{F}}(\vec{w}_{\mathrm{LIME}}) / \widehat{\mathsf{F}}(\vec{w}_{\mathrm{MIP}})$, against its anchoring gap. These two quantities are strongly correlated across both task families ($r = 0.92$), and the scatter cloud passes through the origin: when LIME produces an explanation close to the anchoring hyperplane, its advantage disappears or even turns slightly negative, as predicted by theory. Across the 160 tasks, the advantage is actually negative for 56 of them: despite exploring a larger solution space, LIME fails to outperform the anchored optimum in over a third of the cases. Conversely, the largest advantages correspond to anchoring gaps approaching or exceeding $1.0$---that is, to explanations that misstate the very prediction they are supposed to explain. The color scale confirms that this effect is not a consequence of dimensionality: it appears at every scale, and the largest gaps occur on high-dimensional benchmarks where a sparse anchored fit is most challenging. All tasks are shown without filtering; the few points with both a small gap and a sizeable relative advantage originate from \textit{Medical Charges} and \textit{Forest Fires}, the two benchmarks where every explainer is already within $5 \times 10^{-3}$ of a perfect fit, making the ratio of two near-zero errors difficult to interpret meaningfully. Together with the elimination argument above, these measurements quantify the magnitude of the discrepancy and establish that this deviation is what yields LIME's apparent fidelity advantage.

\begin{figure}[pos=htbp]
    \centering
    \begin{subfigure}{0.48\textwidth}
        \includegraphics[width=\linewidth]{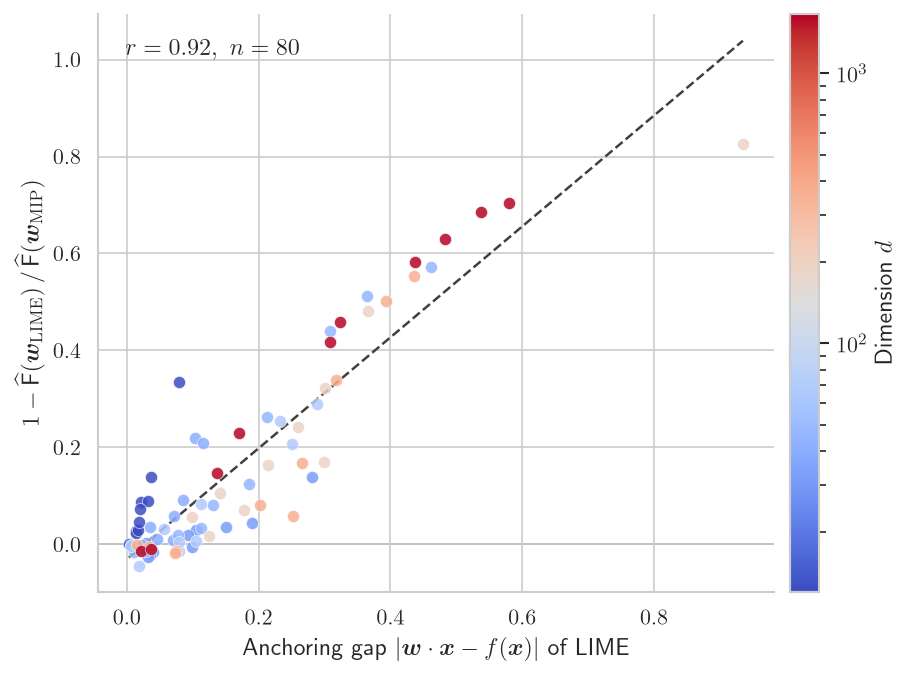}
        \caption{Classification Tasks}
    \end{subfigure}
    \hfill
    \begin{subfigure}{0.48\textwidth}
        \includegraphics[width=\linewidth]{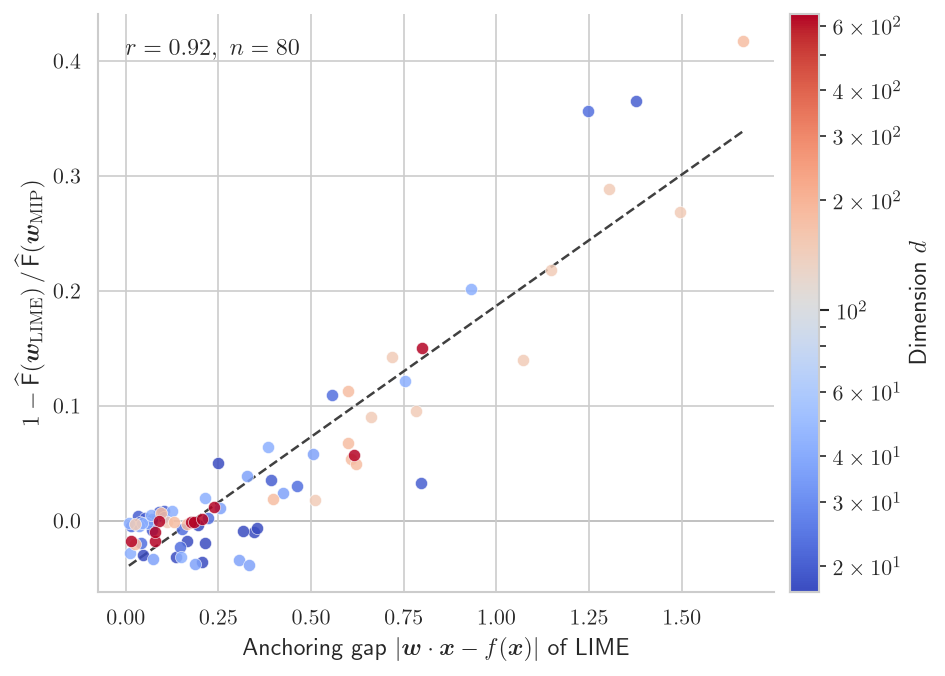}
        \caption{Regression Tasks}
    \end{subfigure}
    \caption{LIME's relative fidelity advantage over MIP, $1 - \widehat{\mathsf{F}}(\vec{w}_{\mathrm{LIME}}) / \widehat{\mathsf{F}}(\vec{w}_{\mathrm{MIP}})$, is plotted against its anchoring gap $|\vec{w} \cdot \vec{x} - f(\vec{x})|$. Each point corresponds to a single explanation task, with all 80 tasks from each family shown, and color indicates the dimension $d$. The advantage grows with anchoring violation and vanishes as the gap closes, illustrating that LIME's lower fidelity error arises from leaving the hypothesis space, not from more effective optimization within it.}
    \label{fig:anchoring_correlation}
\end{figure}

\paragraph{Scalability.} 
The two explainers we propose differ not in solution quality, but in computational cost. Figure~\ref{fig:runtime} presents wall-clock time versus dimension on a doubly logarithmic scale. IHT shows nearly constant runtime: from $d = 12$ to $d = 1652$, it stays between $0.055$ and $0.141$ seconds, a result of a fixed number of iterations over tensors whose size scales linearly with $d$. In contrast, MIP demonstrates the steepest increase. It is the fastest for small instances ($0.088 \pm 0.013$ seconds on \textit{Credit Approval}), overtakes LIME and MAPLE around $d \approx 40$, and then saturates at the time limit for larger dimensions. The certification rate highlights this transition: the solver proves optimality on all ten instances of every benchmark with $d \leq 69$, but only two of the 60 instances drawn from benchmarks with $d \geq 149$---both on \textit{Nomao}, the smallest of them---and none at all beyond $d = 180$. LIME and MAPLE occupy the middle ground, with MAPLE consistently more computationally demanding, reaching $57.7$ seconds on \textit{Cancer Drug Response Methylation}, while IHT completes the same task in just $0.137$ seconds.

This practical division of labor aligns closely with our theoretical expectations. Below the certification threshold, MIP is both exact and efficient, making it the natural choice. Beyond that threshold, it becomes an uncertified anytime heuristic, while IHT still provides explanations of statistically indistinguishable quality in just a tenth of a second, across all tested dimensions.

\begin{figure}[pos=htbp]
    \centering
    \begin{subfigure}{0.48\textwidth}
        \includegraphics[width=\linewidth]{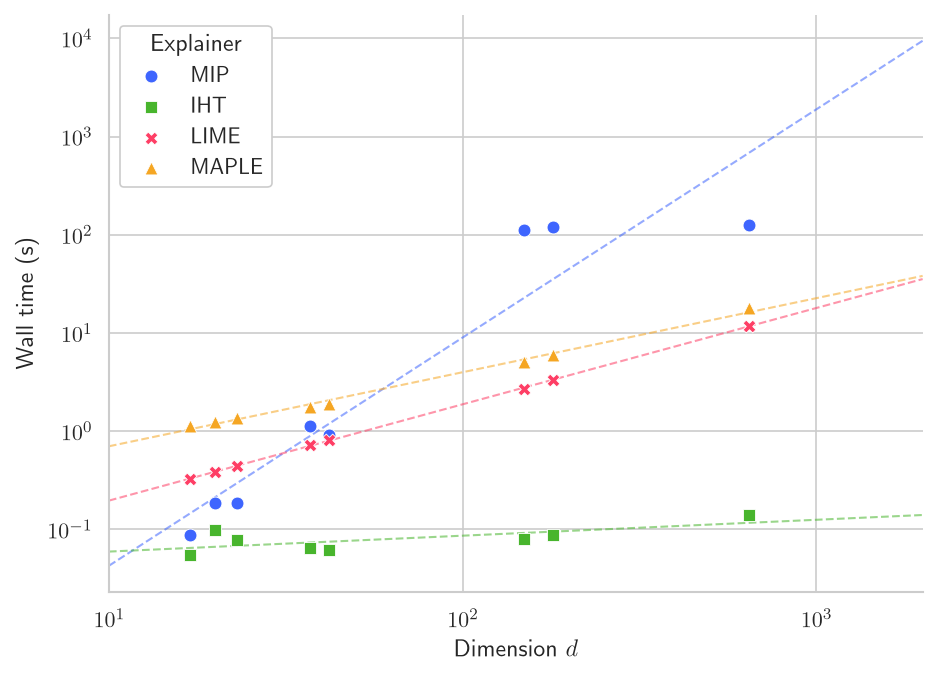}
        \caption{Classification Tasks}
    \end{subfigure}
    \hfill
    \begin{subfigure}{0.48\textwidth}
        \includegraphics[width=\linewidth]{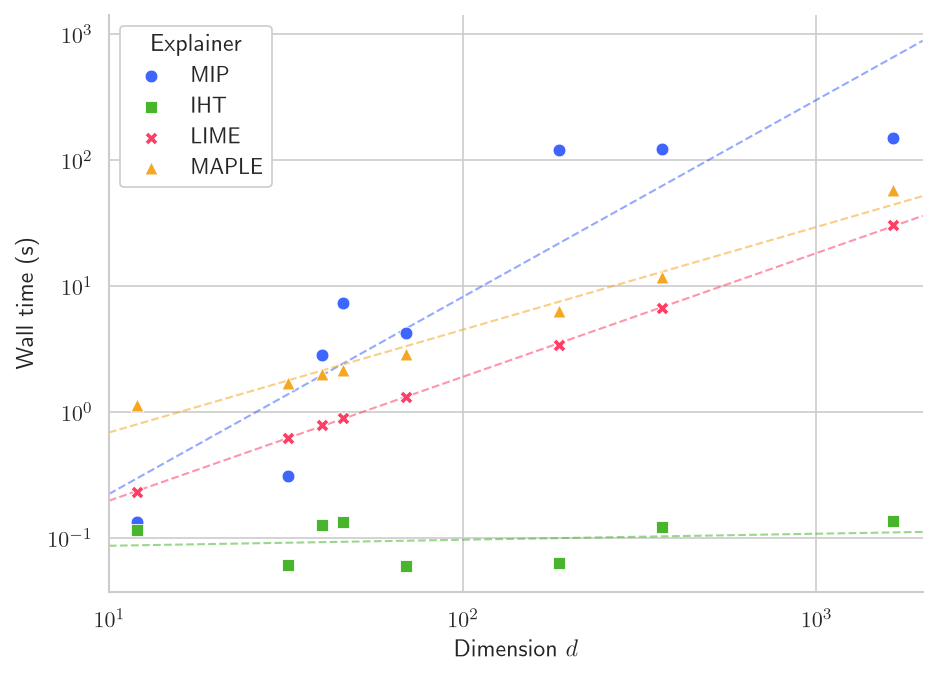}
        \caption{Regression Tasks}
    \end{subfigure}
    \caption{Wall-clock time per explanation task as a function of Boolean dimension $d$, displayed on a doubly logarithmic scale. Each point corresponds to a benchmark, with a least-squares fit shown for each explainer. \MIPRelax\ is omitted for clarity. MIP points clustered near the horizontal band at $120$ seconds represent instances stopped by the time limit. IHT stands out in that its computational cost remains nearly constant as $d$ increases.}
    \label{fig:runtime}
\end{figure}

\subsection{Generalization: Fidelity and Relevance}
\label{subsec:exp_generalization}

The previous subsection compared explainers on the empirical objective they all minimize. We now turn to the two questions that matter most to users: does a low empirical error lead to a low error on the neighborhood distribution $\mathcal{D}_{\vec{x}, \sigma}$, and does a faithful explanation provide insight into the associated prediction? Table~\ref{tab:generalization} addresses both, under the same configuration as before ($k = 5$, $m = 5000$, $\sigma = 1.0$), focusing on the three explainers that respect the sparsity budget; \MIPRelax\ and MAPLE, which return dense explanations, are omitted from this subsection.

\begin{table}[ht]
\centering
\caption{Out-of-sample fidelity error ($\mathsf{F}$) and relevance error ($\mathsf{R}$) for the three explainers that respect the sparsity budget ($k=5$, $m=5000$, $\sigma=1.0$), ordered by dimensionality. \MIPRelax\ and MAPLE are omitted, as neither returns a $k$-sparse explanation. Lower is better in both blocks.}
\label{tab:generalization}
\begin{tabular}{lrrrrrr}
\toprule
 & \multicolumn{3}{c}{Fidelity ($\mathsf{F}$)} & \multicolumn{3}{c}{Relevance ($\mathsf{R}$)} \\
\cmidrule(lr){2-4} \cmidrule(lr){5-7}
Dataset & IHT & MIP & LIME & IHT & MIP & LIME \\
\midrule
\multicolumn{7}{l}{\textbf{Classification}} \\
Credit Approval & $0.107 \pm 0.009$ & $0.107 \pm 0.009$ & $0.108 \pm 0.008$ & $0.008 \pm 0.017$ & $0.007 \pm 0.015$ & $0.043 \pm 0.047$ \\
COMPAS & $0.106 \pm 0.069$ & $0.106 \pm 0.068$ & $0.095 \pm 0.047$ & $0.057 \pm 0.127$ & $0.070 \pm 0.130$ & $0.137 \pm 0.239$ \\
Postoperative & $0.140 \pm 0.046$ & $0.140 \pm 0.046$ & $0.130 \pm 0.026$ & $0.070 \pm 0.099$ & $0.069 \pm 0.098$ & $0.142 \pm 0.207$ \\
Heart Disease & $0.163 \pm 0.029$ & $0.162 \pm 0.030$ & $0.161 \pm 0.025$ & $0.067 \pm 0.050$ & $0.070 \pm 0.048$ & $0.124 \pm 0.073$ \\
Adult & $0.153 \pm 0.044$ & $0.152 \pm 0.043$ & $0.144 \pm 0.030$ & $0.097 \pm 0.097$ & $0.094 \pm 0.095$ & $0.136 \pm 0.159$ \\
Nomao & $0.228 \pm 0.064$ & $0.227 \pm 0.064$ & $0.190 \pm 0.039$ & $0.228 \pm 0.131$ & $0.225 \pm 0.125$ & $0.405 \pm 0.244$ \\
Speed Dating & $0.210 \pm 0.072$ & $0.208 \pm 0.074$ & $0.182 \pm 0.045$ & $0.190 \pm 0.159$ & $0.196 \pm 0.160$ & $0.257 \pm 0.222$ \\
Musk & $0.127 \pm 0.076$ & $0.127 \pm 0.075$ & $0.119 \pm 0.062$ & $0.103 \pm 0.094$ & $0.102 \pm 0.094$ & $0.133 \pm 0.128$ \\
\midrule
\multicolumn{7}{l}{\textbf{Regression}} \\
Medical Charges & $0.001 \pm 0.000$ & $0.001 \pm 0.000$ & $0.001 \pm 0.000$ & $0.001 \pm 0.001$ & $0.001 \pm 0.001$ & $0.001 \pm 0.001$ \\
California Housing & $0.020 \pm 0.002$ & $0.020 \pm 0.001$ & $0.019 \pm 0.001$ & $0.017 \pm 0.004$ & $0.017 \pm 0.003$ & $0.022 \pm 0.007$ \\
Forest Fires & $0.005 \pm 0.000$ & $0.005 \pm 0.000$ & $0.004 \pm 0.000$ & $0.005 \pm 0.001$ & $0.005 \pm 0.001$ & $0.006 \pm 0.001$ \\
Seoul Bike Sharing & $0.015 \pm 0.009$ & $0.015 \pm 0.009$ & $0.010 \pm 0.002$ & $0.022 \pm 0.019$ & $0.022 \pm 0.019$ & $0.026 \pm 0.019$ \\
Student Performance & $0.012 \pm 0.003$ & $0.012 \pm 0.003$ & $0.011 \pm 0.001$ & $0.014 \pm 0.006$ & $0.014 \pm 0.006$ & $0.019 \pm 0.008$ \\
NCI 60 Thioguanine & $0.023 \pm 0.021$ & $0.023 \pm 0.022$ & $0.013 \pm 0.001$ & $0.038 \pm 0.056$ & $0.039 \pm 0.056$ & $0.052 \pm 0.066$ \\
Communities and Crime & $0.024 \pm 0.006$ & $0.024 \pm 0.006$ & $0.018 \pm 0.003$ & $0.032 \pm 0.017$ & $0.032 \pm 0.017$ & $0.038 \pm 0.017$ \\
Cancer Drug Resp. Meth. & $0.020 \pm 0.010$ & $0.020 \pm 0.009$ & $0.009 \pm 0.001$ & $0.039 \pm 0.028$ & $0.038 \pm 0.027$ & $0.047 \pm 0.030$ \\
\bottomrule
\end{tabular}
\end{table}

\paragraph{The empirical error is a faithful proxy.}
Across all $16 \times 3$ entries of the fidelity block, the out-of-sample error $\mathsf{F}$ never deviates from its empirical counterpart $\widehat{\mathsf{F}}$ in Table~\ref{tab:fidelity} by more than $0.005$, with the largest difference observed on \textit{Speed Dating}. This consistency holds across three orders of magnitude in dimension, from $d = 12$ to $d = 1652$, and the gap shows no discernible trend with increasing $d$. Empirical risk minimization over $\mathcal{W}_{\vec{x}, k}$ is thus already tight at $m = 5000$, so the optimization quality measured in Section~\ref{subsec:exp_optimization} directly translates to generalization quality; the $m$-sweep below shows how much earlier this regime is reached.

\paragraph{Fidelity and relevance do not rank the explainers the same way.}
LIME keeps its fidelity advantage out of sample, for the reason established in Section~\ref{subsec:exp_optimization}: it optimizes over a strictly larger space. The relevance block reverses the verdict. LIME is worse than MIP on 15 of the 16 benchmarks, and ties on the sixteenth, \textit{Medical Charges}, where every explainer is already at an error of $0.001$. The margin is substantial rather than marginal: LIME's relevance error is $1.3$ times that of MIP at the median, twice as large on \textit{Nomao} ($0.405$ against $0.225$) and on \textit{Postoperative} ($0.142$ against $0.069$), and six times as large on \textit{Credit Approval} ($0.043$ against $0.007$).

This reversal is the empirical content of the relevance-to-fidelity lemma of Section~\ref{subsec:relevance_to_fidelity}. Fidelity averages the surrogate's error over the whole neighborhood, so an explanation can be faithful on average while being wrong precisely where the user reads it, namely on the subcube its own support designates. Anchoring is what forbids that failure mode, and the fidelity LIME gains by abandoning it is paid for here.

\paragraph{The relevance bound is conservative, and it needs admissibility.} 
At $k = 5$ and $\sigma = 1.0$, the multiplicative factor $(1 + e^{-\sigma})^k$ from Lemma~\ref{lem:fidelity} is $4.79$. The measured ratio $\mathsf{R} / \mathsf{F}$ remains well below this value for both of our explainers: it never exceeds $1.95$ for IHT and $1.90$ for MIP, and stays below $1$ on every classification benchmark. The bound is thus conservative, as expected from a worst-case change-of-measure argument, and a small fidelity error is, in practice, a reliable certificate of relevance at the sparsity budgets explanations can afford.

LIME falls outside the scope of Lemma~\ref{lem:fidelity}, and the measurements confirm this is not a mere technicality. Its ratio rises to $4.00$ on \textit{NCI 60 Thioguanine}, just under the bound, and to $5.22$ on \textit{Cancer Drug Resp.\ Meth.}, exceeding it. On the highest-dimensional benchmark of the study, the guarantee is not just unavailable to LIME---it is invalid. The proof makes the reason clear: relevance measures the error against $f(\vec{x})$ on the subcube singled out by the support, and anchoring ensures the surrogate matches $f(\vec{x})$ there. Without anchoring, the conditioned quantity is no longer the one fidelity controls. An unanchored explanation may be faithful on average over the neighborhood, yet say nothing reliable on the subcube its own support selects---and no amount of fidelity can repair this disconnect. Admissibility is not just what makes an explanation legible to the reader; it is what makes the cheaper criterion a true proxy for the more demanding one.

\begin{figure}[pos=t]
    \centering
    \begin{subfigure}{0.48\textwidth}
        \includegraphics[width=\linewidth]{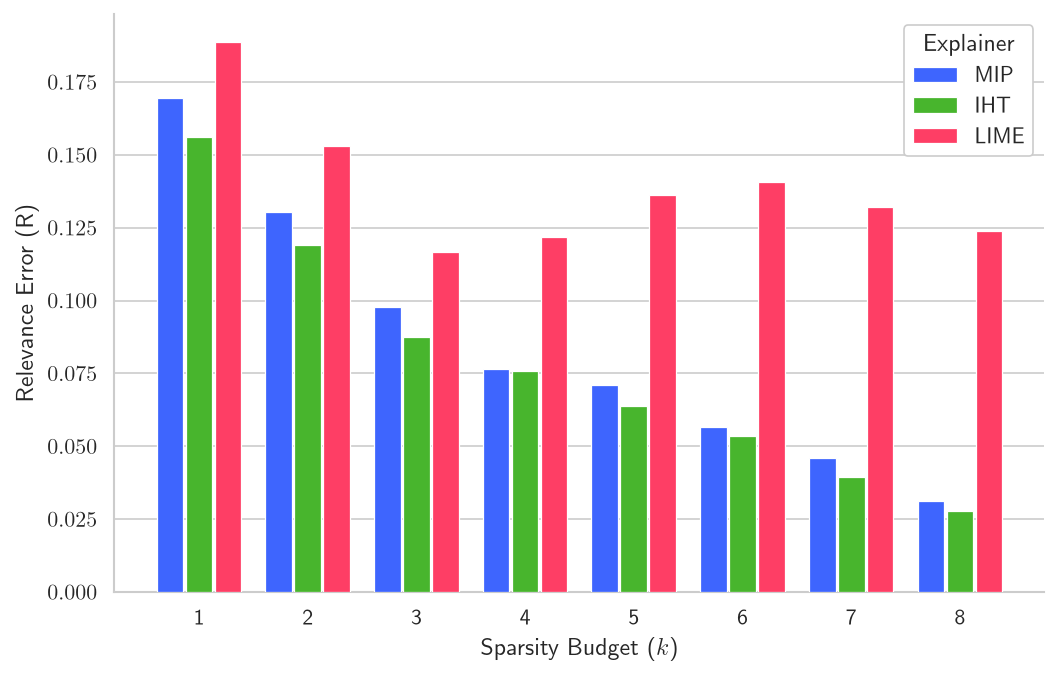}
        \caption{COMPAS (classification, $d = 20$)}
    \end{subfigure}
    \hfill
    \begin{subfigure}{0.48\textwidth}
        \includegraphics[width=\linewidth]{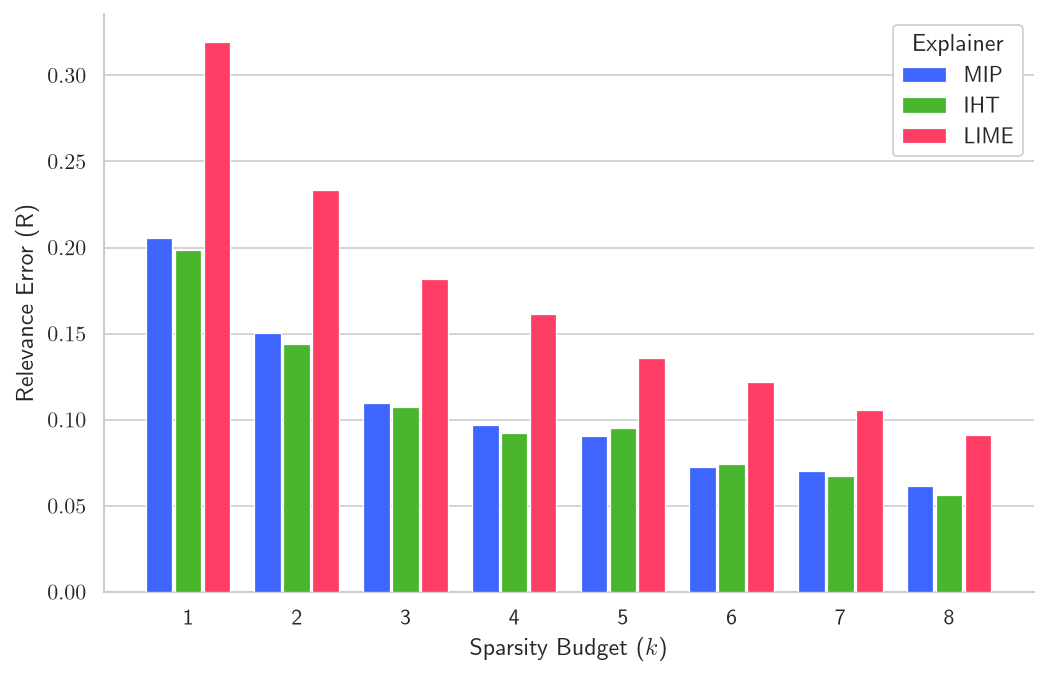}
        \caption{Adult (classification, $d = 42$)}
    \end{subfigure}

    \vspace{0.6em}
    \begin{subfigure}{0.48\textwidth}
        \includegraphics[width=\linewidth]{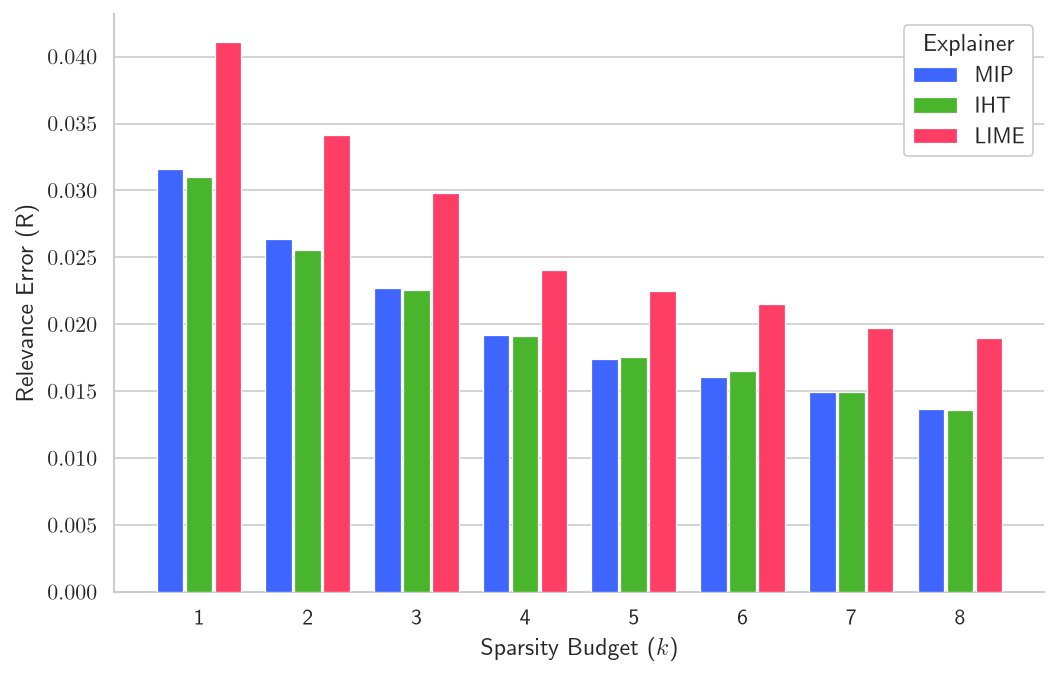}
        \caption{California Housing (regression, $d = 32$)}
    \end{subfigure}
    \hfill
    \begin{subfigure}{0.48\textwidth}
        \includegraphics[width=\linewidth]{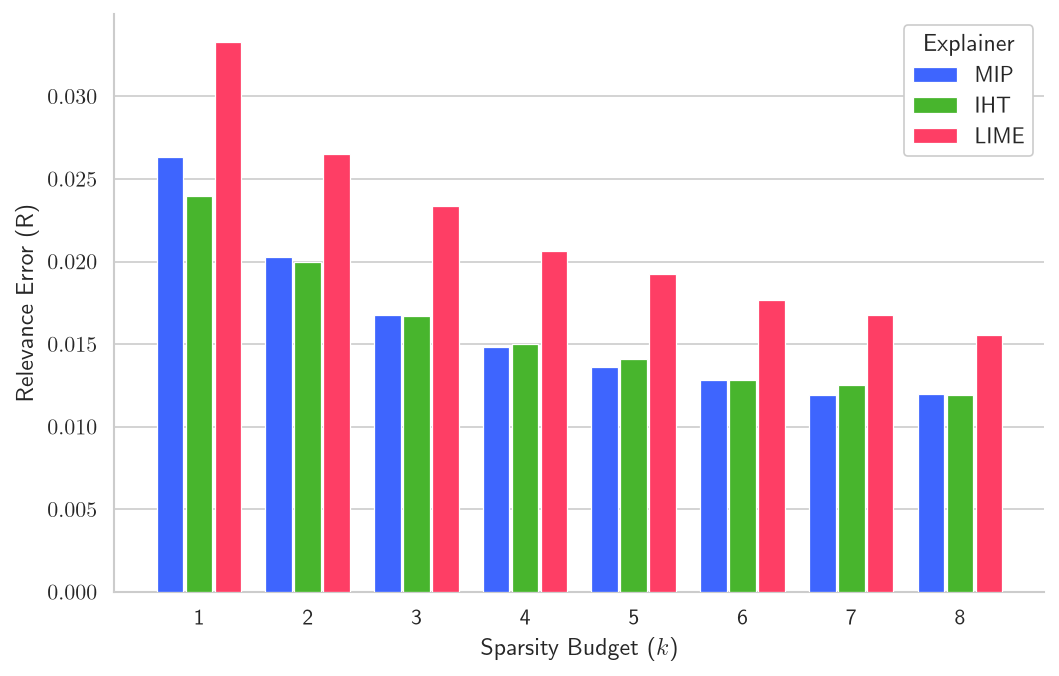}
        \caption{Student Performance (regression, $d = 69$)}
    \end{subfigure}
    \caption{Relevance error $\mathsf{R}$ as a function of sparsity budget $k$, with $m = 5000$ and $\sigma = 1.0$, averaged over ten reference instances per benchmark. Top row: classification; bottom row: regression. Note the differing vertical scales between rows.}
    \label{fig:k_sweep}
\end{figure}

\paragraph{Varying the sparsity budget.} Figure~\ref{fig:k_sweep} varies $k$ from $1$ to $8$ across four representative benchmarks, two from each task family, with $m$ and $\sigma$ held constant. For both IHT and MIP, the relevance error decreases monotonically throughout this range on all benchmarks. This deserves explicit mention, as it is not guaranteed: while a larger budget allows for a better fit, it also reduces the size of the subcube on which relevance is measured by a factor of $1 + e^{-\sigma}$ per additional coordinate, causing the bound in Lemma~\ref{lem:fidelity} to degrade exponentially. Across the budgets a practical explanation can offer, the benefit of increased fit prevails, and the lemma’s exponential warning does not materialize.

LIME behaves differently. Its relevance error exceeds both anchored explainers at every budget and on every benchmark, by a factor that stays around $1.3$ on \textit{California Housing} and \textit{Student Performance} and around $1.5$ on \textit{Adult}. On \textit{COMPAS}, its curve is not even monotonic, and at $k = 8$, its relevance error is roughly four times that of MIP. Thus, increasing the budget does not repair an unanchored explanation and may worsen it: the anchoring gap is not a fixed penalty that a richer hypothesis space eventually eliminates.

For our two explainers, the sweep reinforces the findings of Section~\ref{subsec:exp_optimization}: their curves closely track each other at every budget, with IHT slightly lower at the smallest budgets on some benchmarks and both converging as $k$ increases. Exact coincidence is not expected, since MIP is optimal for empirical fidelity error while the sweep evaluates relevance; nonetheless, the remaining differences are well within the dispersion across reference instances.

\begin{figure}[pos=t]
    \centering
    \begin{subfigure}{0.48\textwidth}
        \includegraphics[width=\linewidth]{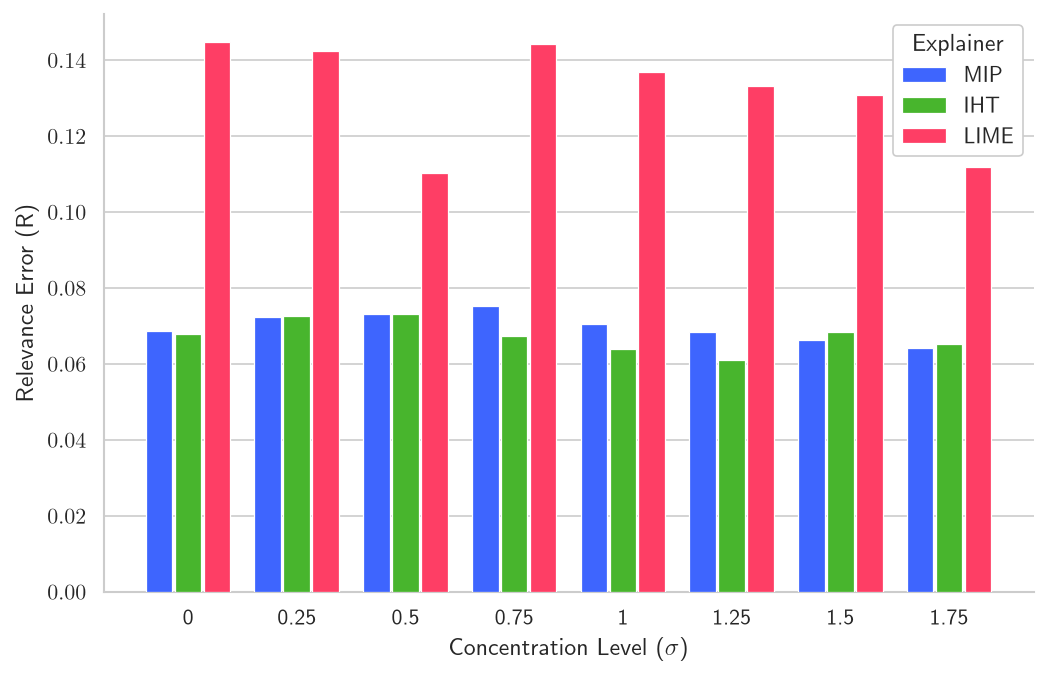}
        \caption{COMPAS (classification, $d = 20$)}
    \end{subfigure}
    \hfill
    \begin{subfigure}{0.48\textwidth}
        \includegraphics[width=\linewidth]{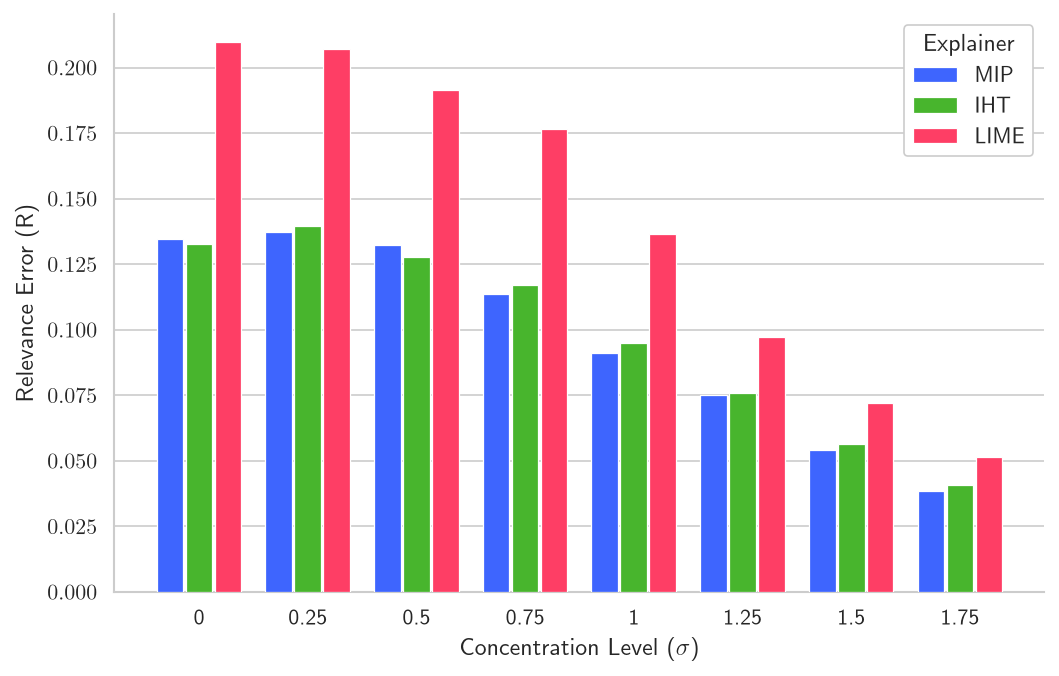}
        \caption{Adult (classification, $d = 42$)}
    \end{subfigure}

    \vspace{0.6em}
    \begin{subfigure}{0.48\textwidth}
        \includegraphics[width=\linewidth]{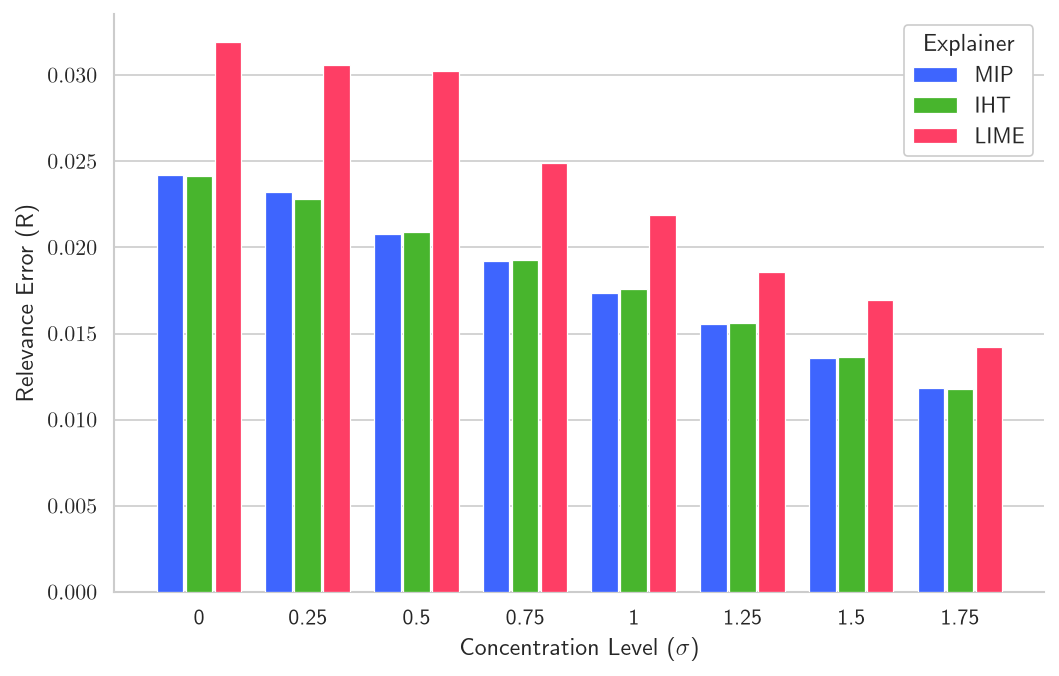}
        \caption{California Housing (regression, $d = 32$)}
    \end{subfigure}
    \hfill
    \begin{subfigure}{0.48\textwidth}
        \includegraphics[width=\linewidth]{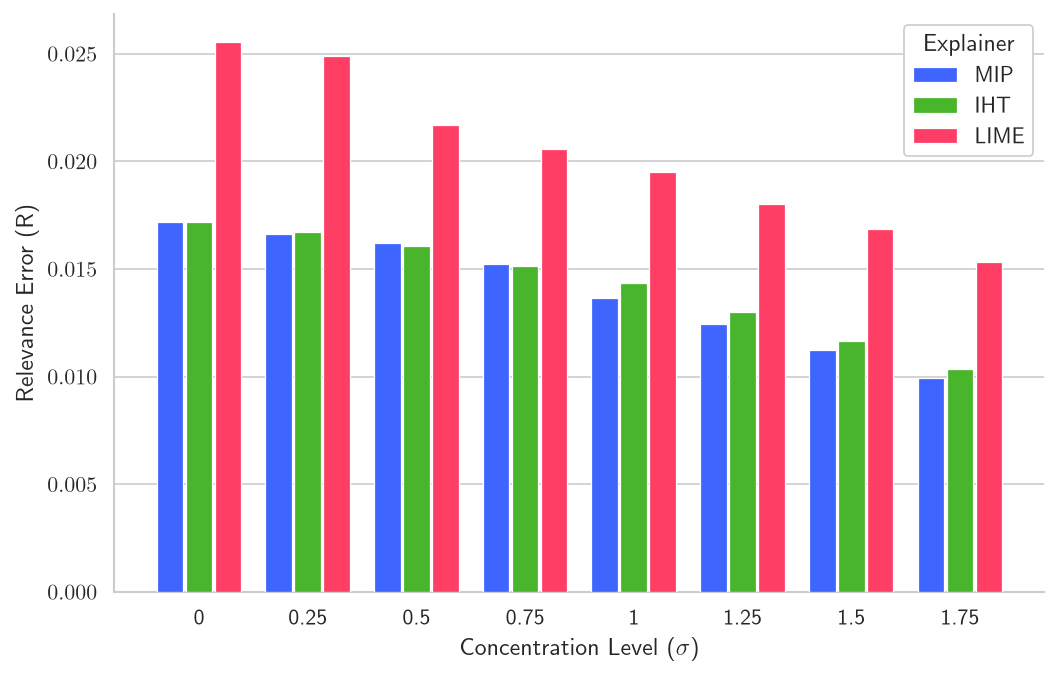}
        \caption{Student Performance (regression, $d = 69$)}
    \end{subfigure}
    \caption{Relevance error $\mathsf{R}$ as a function of the concentration parameter $\sigma$, with $k = 5$ and $m = 5000$, averaged over ten reference instances per benchmark. Top row: classification; bottom row: regression. Note that vertical scales differ between panels.}
    \label{fig:sigma_sweep}
\end{figure}

\paragraph{Varying the locality.} Figure~\ref{fig:sigma_sweep} varies $\sigma$ from $0$ to $1.75$, moving from the uniform distribution to a neighborhood that perturbs roughly one coordinate in seven, with $k = 5$ and $m = 5000$. For three of the four benchmarks, the relevance error for every explainer decreases steadily with increasing $\sigma$: by a factor of about three on \textit{Adult} and $1.7$ on the two regression tasks. A more concentrated neighborhood is easier to track with five coefficients, and the conversion factor $(1 + e^{-\sigma})^k$ from Lemma~\ref{lem:fidelity} tightens accordingly, dropping from $32$ at $\sigma = 0$ to $2.23$ at $\sigma = 1.75$. Guarantee and measurement thus improve together as the explanation becomes more local. \textit{COMPAS} is the exception; Table~\ref{tab:datasets} suggests why: it has by far the largest cross-validation error in the study, so the model exhibits little local structure for any explainer to capture, and increasing locality reveals none.

LIME remains above both anchored explainers at every value of $\sigma$ and across all benchmarks, by a factor ranging from about $1.2$ to $1.6$. The uniform case $\sigma = 0$ is particularly telling: the penalty is not an artifact of a specific locality regime that LIME was not designed for---it is already present, and on \textit{Adult}, it is widest when the neighborhood covers the entire hypercube.

Our two explainers, meanwhile, remain indistinguishable throughout the entire range. This serves as a mild but useful validation of Section~\ref{subsec:convergence}: the step size $\eta$ prescribed by Assumption~\ref{ass:regularity} depends on $\sigma$ via $\cosh^2(\sigma/2)$, and IHT follows the exact solver at every point on the grid with no additional tuning beyond that formula.

\begin{figure}[pos=t]
    \centering
    \begin{subfigure}{0.48\textwidth}
        \includegraphics[width=\linewidth]{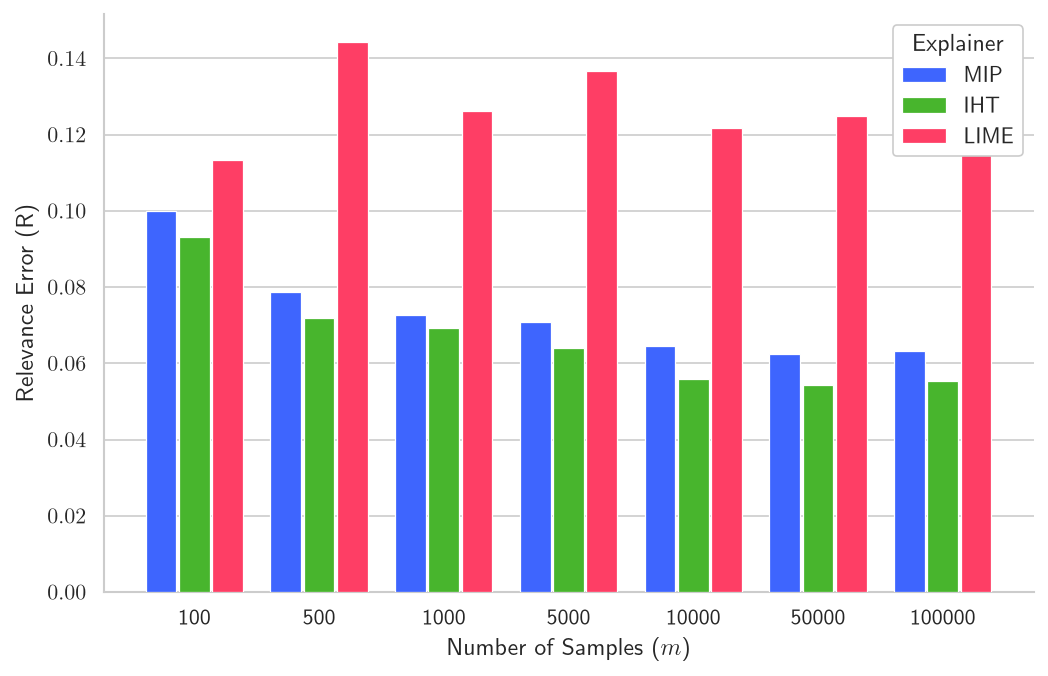}
        \caption{COMPAS (classification, $d = 20$)}
    \end{subfigure}
    \hfill
    \begin{subfigure}{0.48\textwidth}
        \includegraphics[width=\linewidth]{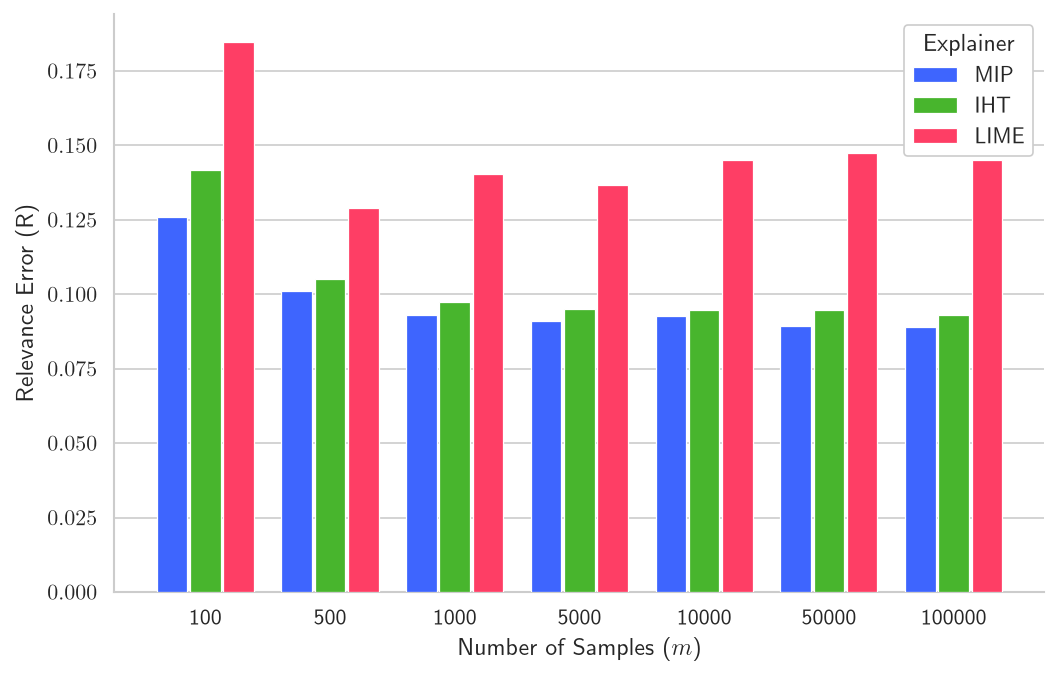}
        \caption{Adult (classification, $d = 42$)}
    \end{subfigure}

    \vspace{0.6em}
    \begin{subfigure}{0.48\textwidth}
        \includegraphics[width=\linewidth]{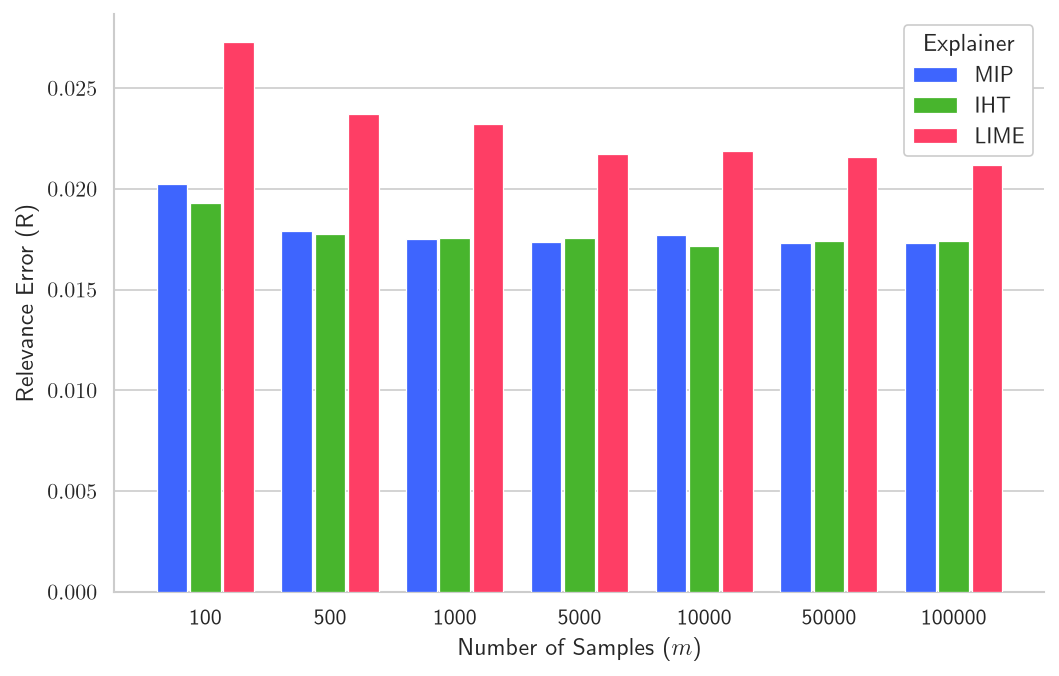}
        \caption{California Housing (regression, $d = 32$)}
    \end{subfigure}
    \hfill
    \begin{subfigure}{0.48\textwidth}
        \includegraphics[width=\linewidth]{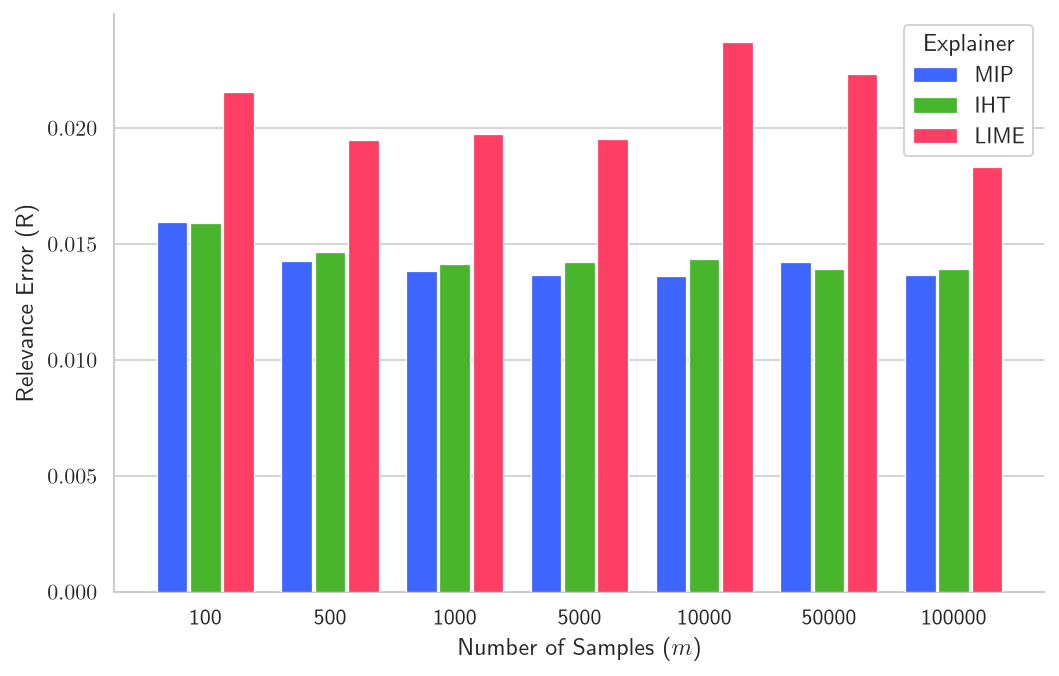}
        \caption{Student Performance (regression, $d = 69$)}
    \end{subfigure}
    \caption{Relevance error $\mathsf{R}$ as a function of sample size $m$, shown on a logarithmic grid, with $k = 5$ and $\sigma = 1.0$, averaged over ten reference instances per benchmark. Top row: classification; bottom row: regression. Note that vertical scales differ between panels.}
    \label{fig:m_sweep}
\end{figure}

\paragraph{Varying the sample size.}
Figure~\ref{fig:m_sweep} sweeps $m$ from $10^2$ to $10^5$ with $k = 5$ and $\sigma = 1.0$. For IHT and MIP, the results align with predictions from empirical risk minimization: relevance error drops sharply during the first decade, then levels off, plateauing at $m = 1000$ for \textit{California Housing} and \textit{Student Performance}, and at $m = 5000$ for \textit{Adult} and \textit{COMPAS}. Beyond this point, estimation error essentially disappears, leaving only the approximation error of a five-coefficient linear model, which no amount of additional sampling can reduce. This context also clarifies the guarantee in Theorem~\ref{thm:end_to_end}: for \textit{Adult}, at $\varepsilon = 10^{-2}$ and $\delta = 0.05$, the required sample size is on the order of $10^{11}$---eight orders of magnitude above the $10^3$ sufficient in practice. The $(B+1)^4$ term, with $B = k + 4\sqrt{k/\alpha}$, accounts for most of this gap: the analysis pays for a bound holding uniformly over all $k$-sparse explanations and the worst-case restricted curvature, neither of which is realized in these instances.

LIME's curves behave quite differently: they remain flat from the outset, are irregular, and on \textit{COMPAS} and \textit{Adult}, are not even monotonic---the value at $m = 10^5$ is no better than at $m = 500$. This contrast underscores the subsection's central point. For anchored explanations, variance is the limiting factor, which can be reduced with more samples; for LIME, it is the violated constraint, which no amount of data can fix. The comparison at both ends of the grid makes this concrete: on all four benchmarks, IHT and MIP with just $m = 100$ samples already achieve lower relevance error than LIME with $m = 10^5$, illustrating a thousandfold increase in data that yields no benefit for LIME.


\section{Conclusion}
\label{sec:conclusion}

We have introduced a unified framework for probabilistic explainability that encompasses both binary classification and continuous regression within a single formalism. Our framework produces sparse, anchored linear explanations over the Boolean hypercube, which generalize traditional subset-based approaches by incorporating both the magnitude and direction of each feature’s contribution. Identifying the optimal explanation is \ClassNPPP-hard for neural networks, so we address the two sources of this complexity separately. The counting challenge is managed by introducing the fidelity error, an unconditional surrogate whose gap to the relevance error is bounded by $(1 + e^{-\sigma})^k$ over a parameterized family of neighborhood distributions. The combinatorial aspect is tackled either exactly with a Mixed Integer Programming formulation or approximately with an Iterative Hard Thresholding method, whose projection onto the anchoring hyperplane is both exact and efficiently computable. Both methods offer sample complexity guarantees that scale polynomially with the sparsity budget and logarithmically with the dimensionality. Experimentally, the two approaches yield nearly identical solution quality, differing primarily in computational cost: MIP certifies optimality in moderate dimensions, while IHT efficiently scales to the largest benchmarks. Compared to LIME and MAPLE, our results support the core claim of the framework: explanations that depart from the hypothesis space may appear more faithful, but they are less relevant---and no amount of data can compensate for that loss.

This work opens several avenues for future research, of which we highlight three.

\paragraph{Richer explanation classes.} Linear explanations are effective at capturing the magnitude and direction of individual feature contributions, but they cannot model interactions between features. A natural next step is to explore sparse, low-degree polynomial explanations, where a small set of monomials---such as binomials reflecting pairwise correlations---replace the linear terms. In the space $\{-1,+1\}^d$, products of coordinates also take values in $\{-1,+1\}$, so these polynomials can be interpreted as sparse, anchored linear functions over an expanded set of features. Both \eqref{pb:mip} and Algorithm~\ref{alg:iht} can be applied directly to this lifted problem. Lemma~\ref{lem:fidelity} generalizes as well, with the exponent $k$ replaced by $q k$ for degree-$q$ monomials, since fixing the support literals determines all relevant monomials. However, the analysis in Section~\ref{subsec:rsc_rss} does not transfer straightforwardly: the lifted features introduce dependencies, resulting in a non-isotropic covariance structure and requiring new restricted eigenvalue bounds. Whether these analytical challenges can be overcome without substantially increasing sample complexity for small degrees remains open.

\paragraph{Optimizing relevance directly.} Our approach achieves relevance by optimizing fidelity, prompting the question of whether this indirect route is necessary. Specifically, could one minimize the \emph{empirical relevance error}---the average loss over samples satisfying $\vec{w} \cdot \vec{z} = \vec{w} \cdot \vec{x}$---directly? Statistically, this seems feasible: 
Lemma~\ref{lem:fidelity} shows that the conditioning event has probability at least $(1 + e^{-\sigma})^{-k}$, so estimation only increases sample size requirements by a constant factor. The main challenge, however, is algorithmic. The conditioning event depends discontinuously on $\vec{w}$, making the empirical objective piecewise constant across a complex arrangement of hyperplanes---devoid of gradients, and unsuitable for methods relying on convexity or thresholding. An exact formulation is equally problematic, requiring an indicator for each sample tied to an equality constraint, which in turn renders the objective a ratio of linear forms in these indicators. Underlying all of this is the \ClassNPPP-hardness established in Theorem~\ref{thm:hardness}, which persists even with empirical relaxation. While the surrogate fidelity-based approach appears more promising in practice, discovering a tractable formulation for the direct problem---even under strong restrictions on $f$---would significantly advance our understanding.

\paragraph{Beyond binarized domains.} Our framework currently operates on the Boolean hypercube, mapping categorical and numerical attributes via the binarization process outlined in Section~\ref{sec:setup}. While this approach is common in model-agnostic explainability, it creates a disconnect between the sparsity motivated by cognitive considerations and user reasoning: the sparsity budget $k$ counts indicator literals, whereas users typically think in terms of attributes. Bridging this gap would require extending the framework to a product space $\prod_{i} \Omega_i$ of finite attribute domains, allowing explanations to be expressed at the attribute level. The core analytical tools---relying only on the product structure of $\mathcal{D}_{\vec{x}, \sigma}$ and the additive decomposition of Hamming distance---remain applicable under this generalization. The key modification lies in the sparsity constraint, which becomes a group constraint over the one-hot encoded block of each attribute, necessitating a group analogue of the \textsc{gshp} projection step. A further, more ambitious generalization would be to relax the assumption of coordinate independence, enabling the analysis to follow the data manifold rather than rely on metric balls around $\vec{x}$. However, this would break the covariance isotropy underpinning Section~\ref{subsec:rsc_rss}, requiring a fundamentally new restricted eigenvalue analysis adapted to the resulting dependence structure.

\paragraph*{Acknowledgments.}
This work has benefited from the support of the AI Chair EXPEKCTATION (ANR-19-CHIA-0005-01) of the French National Research Agency. It was also partially supported by TAILOR, a project funded by EU Horizon 2020 research and innovation programme under GA No 952215.

\bibliographystyle{cas-model2-names}
\bibliography{aij2026_bibliography}

\begin{thebibliography}{66}
\expandafter\ifx\csname natexlab\endcsname\relax\def\natexlab#1{#1}\fi
\providecommand{\url}[1]{\texttt{#1}}
\providecommand{\href}[2]{#2}
\providecommand{\path}[1]{#1}
\providecommand{\DOIprefix}{doi:}
\providecommand{\ArXivprefix}{arXiv:}
\providecommand{\URLprefix}{URL: }
\providecommand{\Pubmedprefix}{pmid:}
\providecommand{\doi}[1]{\href{http://dx.doi.org/#1}{\path{#1}}}
\providecommand{\Pubmed}[1]{\href{pmid:#1}{\path{#1}}}
\providecommand{\bibinfo}[2]{#2}
\ifx\xfnm\relax \def\xfnm[#1]{\unskip,\space#1}\fi
\bibitem[{Agarwal et~al.(2010)Agarwal, Negahban and Wainwright}]{Agarwal.NeurIPS.2010}
\bibinfo{author}{Agarwal, A.}, \bibinfo{author}{Negahban, S.}, \bibinfo{author}{Wainwright, M.J.}, \bibinfo{year}{2010}.
\newblock \bibinfo{title}{Fast global convergence rates of gradient methods for high-dimensional statistical recovery}, in: \bibinfo{booktitle}{Advances in Neural Information Processing Systems 23: Proceedings of the Annual Conference on Neural Information Processing Systems (NeurIPS)}.
\bibitem[{Agarwal et~al.(2021)Agarwal, Jabbari, Agarwal, Upadhyay, Wu and Lakkaraju}]{Agarwal.ICML.2021}
\bibinfo{author}{Agarwal, S.}, \bibinfo{author}{Jabbari, S.}, \bibinfo{author}{Agarwal, C.}, \bibinfo{author}{Upadhyay, S.}, \bibinfo{author}{Wu, S.}, \bibinfo{author}{Lakkaraju, H.}, \bibinfo{year}{2021}.
\newblock \bibinfo{title}{Towards the unification and robustness of perturbation and gradient based explanations}, in: \bibinfo{booktitle}{Proceedings of the 38th International Conference on Machine Learning (ICML)}, pp. \bibinfo{pages}{110--119}.
\bibitem[{Alvarez-Melis and Jaakkola(2018)}]{Alvarez-Melis.WHI.2018}
\bibinfo{author}{Alvarez-Melis, D.}, \bibinfo{author}{Jaakkola, T.}, \bibinfo{year}{2018}.
\newblock \bibinfo{title}{On the robustness of interpretability methods}, in: \bibinfo{booktitle}{Proceedings of the 2018 ICML Workshop on Human Interpretability in Machine Learning (WHI)}.
\bibitem[{Arenas et~al.(2022)Arenas, Barcel{\'{o}}, Orth and Subercaseaux}]{Arenas.NeurIPS.2022}
\bibinfo{author}{Arenas, M.}, \bibinfo{author}{Barcel{\'{o}}, P.}, \bibinfo{author}{Orth, M.A.R.}, \bibinfo{author}{Subercaseaux, B.}, \bibinfo{year}{2022}.
\newblock \bibinfo{title}{On computing probabilistic explanations for decision trees,}, in: \bibinfo{booktitle}{Advances in Neural Information Processing Systems 35: Proceedings of the Annual Conference on Neural Information Processing Systems (NeurIPS)}.
\bibitem[{Audemard et~al.(2022)Audemard, Bellart, Bounia, Koriche, Lagniez and Marquis}]{Audemard.AAAI.2022}
\bibinfo{author}{Audemard, G.}, \bibinfo{author}{Bellart, S.}, \bibinfo{author}{Bounia, L.}, \bibinfo{author}{Koriche, F.}, \bibinfo{author}{Lagniez, J.}, \bibinfo{author}{Marquis, P.}, \bibinfo{year}{2022}.
\newblock \bibinfo{title}{Trading complexity for sparsity in random forest explanations}, in: \bibinfo{booktitle}{Annual {AAAI} Conference on Artificial Intelligence}, pp. \bibinfo{pages}{5461--5469}.
\bibitem[{Barcel{\'{o}} et~al.(2020)Barcel{\'{o}}, Monet, P{\'{e}}rez and Subercaseaux}]{Barcelo.NeurIPS.2020}
\bibinfo{author}{Barcel{\'{o}}, P.}, \bibinfo{author}{Monet, M.}, \bibinfo{author}{P{\'{e}}rez, J.}, \bibinfo{author}{Subercaseaux, B.}, \bibinfo{year}{2020}.
\newblock \bibinfo{title}{Model interpretability through the lens of computational complexity}, in: \bibinfo{booktitle}{Advances in Neural Information Processing Systems 33: Proceedings of the Annual Conference on Neural Information Processing Systems (NeurIPS)}.
\bibitem[{Bartlett et~al.(2006)Bartlett, Jordan and McAuliffe}]{Bartlett.JASA.2006}
\bibinfo{author}{Bartlett, P.L.}, \bibinfo{author}{Jordan, M.I.}, \bibinfo{author}{McAuliffe, J.D.}, \bibinfo{year}{2006}.
\newblock \bibinfo{title}{Convexity, classification, and risk bounds}.
\newblock \bibinfo{journal}{Journal of the American Statistical Association} \bibinfo{volume}{101}, \bibinfo{pages}{138--156}.
\bibitem[{Bassan et~al.(2026)Bassan, Huang and Katz}]{Bassan.ICLR.2026}
\bibinfo{author}{Bassan, S.}, \bibinfo{author}{Huang, X.}, \bibinfo{author}{Katz, G.}, \bibinfo{year}{2026}.
\newblock \bibinfo{title}{Unifying formal explanations: A complexity-theoretic perspective}, in: \bibinfo{booktitle}{International Conference on Learning Representations (ICLR)}.
\bibitem[{Beale et~al.(1967)Beale, Kendall and Mann}]{Beale.Biometrika.1967}
\bibinfo{author}{Beale, E.M.L.}, \bibinfo{author}{Kendall, M.G.}, \bibinfo{author}{Mann, D.W.}, \bibinfo{year}{1967}.
\newblock \bibinfo{title}{The discarding of variables in multivariate analysis}.
\newblock \bibinfo{journal}{Biometrika} \bibinfo{volume}{54}, \bibinfo{pages}{357--366}.
\bibitem[{Bertsimas et~al.(2016)Bertsimas, King and Mazumder}]{Bertsimas.AOS.2016}
\bibinfo{author}{Bertsimas, D.}, \bibinfo{author}{King, A.}, \bibinfo{author}{Mazumder, R.}, \bibinfo{year}{2016}.
\newblock \bibinfo{title}{Best subset selection via a modern optimization lens}.
\newblock \bibinfo{journal}{The Annals of Statistics} \bibinfo{volume}{44}, \bibinfo{pages}{813--852}.
\bibitem[{Bertsimas and Weismantel(2005)}]{Bertsimas.Book.2005}
\bibinfo{author}{Bertsimas, D.}, \bibinfo{author}{Weismantel, R.}, \bibinfo{year}{2005}.
\newblock \bibinfo{title}{Optimization Over Integers}.
\newblock \bibinfo{publisher}{Dynamic Ideas}.
\bibitem[{Blanc et~al.(2021)Blanc, Lange and Tan}]{Blanc.NeurIPS.2021}
\bibinfo{author}{Blanc, G.}, \bibinfo{author}{Lange, J.}, \bibinfo{author}{Tan, L.}, \bibinfo{year}{2021}.
\newblock \bibinfo{title}{Provably efficient, succinct, and precise explanations}, in: \bibinfo{booktitle}{Advances in Neural Information Processing Systems 34: Proceedings of the Annual Conference on Neural Information Processing Systems (NeurIPS)}, pp. \bibinfo{pages}{6129--6141}.
\bibitem[{Blumensath and Davies(2008)}]{Blumensath.JFAA.2008}
\bibinfo{author}{Blumensath, T.}, \bibinfo{author}{Davies, M.E.}, \bibinfo{year}{2008}.
\newblock \bibinfo{title}{Iterative thresholding for sparse approximations}.
\newblock \bibinfo{journal}{Journal of Fourier Analysis and Applications} \bibinfo{volume}{14}, \bibinfo{pages}{629--654}.
\bibitem[{Blumensath and Davies(2009)}]{Blumensath.ACHA.2009}
\bibinfo{author}{Blumensath, T.}, \bibinfo{author}{Davies, M.E.}, \bibinfo{year}{2009}.
\newblock \bibinfo{title}{Iterative hard thresholding for compressed sensing}.
\newblock \bibinfo{journal}{Applied and Computational Harmonic Analysis} \bibinfo{volume}{27}, \bibinfo{pages}{265--274}.
\bibitem[{Bounia and Koriche(2023)}]{Bounia.UAI.2023}
\bibinfo{author}{Bounia, L.}, \bibinfo{author}{Koriche, F.}, \bibinfo{year}{2023}.
\newblock \bibinfo{title}{Approximating probabilistic explanations via supermodular minimization}, in: \bibinfo{booktitle}{Proceedings of the 39th International Conference on Uncertainty in Artificial Intelligence (UAI)}, pp. \bibinfo{pages}{216--225}.
\bibitem[{Burkart and Huber(2021)}]{Burkart.JAIR.2021}
\bibinfo{author}{Burkart, N.}, \bibinfo{author}{Huber, M.F.}, \bibinfo{year}{2021}.
\newblock \bibinfo{title}{A survey on the explainability of supervised machine learning}.
\newblock \bibinfo{journal}{Journal of Artificial Intelligence Research} \bibinfo{volume}{70}, \bibinfo{pages}{245--317}.
\bibitem[{Candes and Tao(2005)}]{Candes.TIT.2005}
\bibinfo{author}{Candes, E.}, \bibinfo{author}{Tao, T.}, \bibinfo{year}{2005}.
\newblock \bibinfo{title}{Decoding by linear programming}.
\newblock \bibinfo{journal}{IEEE Transactions on Information Theory} \bibinfo{volume}{51}, \bibinfo{pages}{4203--4215}.
\bibitem[{Cooper and Marques-Silva(2023)}]{Cooper.AI.2023}
\bibinfo{author}{Cooper, M.}, \bibinfo{author}{Marques-Silva, J.}, \bibinfo{year}{2023}.
\newblock \bibinfo{title}{Tractability of explaining classifier decisions}.
\newblock \bibinfo{journal}{Artificial Intelligence} \bibinfo{volume}{316}, \bibinfo{pages}{103841}.
\bibitem[{Darwiche and Hirth(2020)}]{Darwiche.ECAI.2020}
\bibinfo{author}{Darwiche, A.}, \bibinfo{author}{Hirth, A.}, \bibinfo{year}{2020}.
\newblock \bibinfo{title}{On the reasons behind decisions}, in: \bibinfo{booktitle}{European Conference on Artificial Intelligence ({ECAI})}, pp. \bibinfo{pages}{712--720}.
\bibitem[{Fligner and Verducci(1993)}]{Fligner.Book.1993}
\bibinfo{author}{Fligner, M.}, \bibinfo{author}{Verducci, J.}, \bibinfo{year}{1993}.
\newblock \bibinfo{title}{Probability models and statistical analyses for ranking data}. volume~\bibinfo{volume}{80}.
\newblock \bibinfo{publisher}{Springer}.
\bibitem[{Garreau and von Luxburg(2020)}]{Garreau.AISTATS.2020}
\bibinfo{author}{Garreau, D.}, \bibinfo{author}{von Luxburg, U.}, \bibinfo{year}{2020}.
\newblock \bibinfo{title}{Explaining the explainer: {A} first theoretical analysis of {LIME}}, in: \bibinfo{booktitle}{Proceedings of the 23rd International Conference on Artificial Intelligence and Statistics (AISTATS)}, pp. \bibinfo{pages}{1287--1296}.
\bibitem[{Ghorbani et~al.(2019)Ghorbani, Abid and Zou}]{Ghorbani.AAAI.2019}
\bibinfo{author}{Ghorbani, A.}, \bibinfo{author}{Abid, A.}, \bibinfo{author}{Zou, J.}, \bibinfo{year}{2019}.
\newblock \bibinfo{title}{Interpretation of neural networks is fragile}, in: \bibinfo{booktitle}{Proceedings of the 33rd Conference on Artificial Intelligence (AAAI)}, pp. \bibinfo{pages}{3681--3688}.
\bibitem[{Guidotti(2024)}]{Guidotti.DMKD.2024}
\bibinfo{author}{Guidotti, R.}, \bibinfo{year}{2024}.
\newblock \bibinfo{title}{Counterfactual explanations and how to find them: literature review and benchmarking}.
\newblock \bibinfo{journal}{Data Mining and Knowledge Discovery} \bibinfo{volume}{38}, \bibinfo{pages}{2770–2824}.
\bibitem[{Halford et~al.(1998)Halford, Wilson and Phillips}]{Halford.BBS.1998}
\bibinfo{author}{Halford, G.}, \bibinfo{author}{Wilson, W.}, \bibinfo{author}{Phillips, S.}, \bibinfo{year}{1998}.
\newblock \bibinfo{title}{Relational complexity: A measure of capacity limitations in processing}.
\newblock \bibinfo{journal}{Behavioral and Brain Sciences} \bibinfo{volume}{21}, \bibinfo{pages}{803--831}.
\bibitem[{Hocking and Leslie(1967)}]{Hocking.Tech.1967}
\bibinfo{author}{Hocking, R.R.}, \bibinfo{author}{Leslie, R.N.}, \bibinfo{year}{1967}.
\newblock \bibinfo{title}{Selection of the best subset in regression analysis}.
\newblock \bibinfo{journal}{Technometrics} \bibinfo{volume}{9}, \bibinfo{pages}{531--540}.
\bibitem[{Hui and Belkin(2021)}]{Hui.ICLR.2021}
\bibinfo{author}{Hui, L.}, \bibinfo{author}{Belkin, M.}, \bibinfo{year}{2021}.
\newblock \bibinfo{title}{Evaluation of neural architectures trained with square loss vs cross-entropy in classification tasks}, in: \bibinfo{booktitle}{Proceedings of the 9th International Conference on Learning Representations (ICLR)}.
\bibitem[{Ignatiev(2020)}]{Ignatiev.IJCAI.2020}
\bibinfo{author}{Ignatiev, A.}, \bibinfo{year}{2020}.
\newblock \bibinfo{title}{Towards trustable explainable {AI}}, in: \bibinfo{booktitle}{Proceedings of the 29th International Joint Conference on Artificial Intelligence (IJCAI)}, pp. \bibinfo{pages}{5154--5158}.
\bibitem[{Ignatiev et~al.(2020)Ignatiev, Cooper, Siala, Hebrard and Marques-Silva}]{Ignatiev.CP.2020}
\bibinfo{author}{Ignatiev, A.}, \bibinfo{author}{Cooper, M.}, \bibinfo{author}{Siala, M.}, \bibinfo{author}{Hebrard, E.}, \bibinfo{author}{Marques-Silva, J.}, \bibinfo{year}{2020}.
\newblock \bibinfo{title}{Towards formal fairness in machine learning}, in: \bibinfo{booktitle}{International Conference on Principles and Practiceof Constraint Programming (CP)}, p. \bibinfo{pages}{846–867}.
\bibitem[{Ignatiev et~al.(2019)Ignatiev, Narodytska and Marques{-}Silva}]{Ignatiev.AAAI.2019}
\bibinfo{author}{Ignatiev, A.}, \bibinfo{author}{Narodytska, N.}, \bibinfo{author}{Marques{-}Silva, J.}, \bibinfo{year}{2019}.
\newblock \bibinfo{title}{Abduction-based explanations for machine learning models}, in: \bibinfo{booktitle}{Annual {AAAI} Conference on Artificial Intelligence}, pp. \bibinfo{pages}{1511--1519}.
\bibitem[{Izza et~al.(2023)Izza, Huang, Ignatiev, Narodytska, Cooper and Marques-Silva}]{Izza.JAR.2023}
\bibinfo{author}{Izza, Y.}, \bibinfo{author}{Huang, X.}, \bibinfo{author}{Ignatiev, A.}, \bibinfo{author}{Narodytska, N.}, \bibinfo{author}{Cooper, M.}, \bibinfo{author}{Marques-Silva, J.}, \bibinfo{year}{2023}.
\newblock \bibinfo{title}{On computing probabilistic abductive explanations}.
\newblock \bibinfo{journal}{International Journal of Approximate Reasoning} \bibinfo{volume}{159}, \bibinfo{pages}{108939}.
\bibitem[{Jain et~al.(2014)Jain, Tewari and Kar}]{Jain.NeurIPS.2014}
\bibinfo{author}{Jain, P.}, \bibinfo{author}{Tewari, A.}, \bibinfo{author}{Kar, P.}, \bibinfo{year}{2014}.
\newblock \bibinfo{title}{On iterative hard thresholding methods for high-dimensional m-estimation}, in: \bibinfo{booktitle}{Advances in Neural Information Processing Systems 27: Proceedings of the Annual Conference on Neural Information Processing Systems (NeurIPS)}, pp. \bibinfo{pages}{685--693}.
\bibitem[{Jalali et~al.(2011)Jalali, Johnson and Ravikumar}]{Jalali.NeurIPS.2011}
\bibinfo{author}{Jalali, A.}, \bibinfo{author}{Johnson, C.}, \bibinfo{author}{Ravikumar, P.}, \bibinfo{year}{2011}.
\newblock \bibinfo{title}{On learning discrete graphical models using greedy methods}, in: \bibinfo{booktitle}{Advances in Neural Information Processing Systems 24: Proceedings of the Annual Conference on Neural Information Processing Systems (NeurIPS)}, pp. \bibinfo{pages}{1935--1943}.
\bibitem[{Johnson-Laird(2010)}]{JohnsonLaird.PNAS.2010}
\bibinfo{author}{Johnson-Laird, P.}, \bibinfo{year}{2010}.
\newblock \bibinfo{title}{Mental models and human reasoning}.
\newblock \bibinfo{journal}{Proceedings of the National Academy of Sciences} \bibinfo{volume}{107}, \bibinfo{pages}{18243--18250}.
\bibitem[{Kakade et~al.(2008)Kakade, Sridharan and Tewari}]{Kakade.NeurIPS.2008}
\bibinfo{author}{Kakade, S.M.}, \bibinfo{author}{Sridharan, K.}, \bibinfo{author}{Tewari, A.}, \bibinfo{year}{2008}.
\newblock \bibinfo{title}{On the complexity of linear prediction: Risk bounds, margin bounds, and regularization}, in: \bibinfo{booktitle}{Advances in Neural Information Processing Systems 21: Proceedings of the Annual Conference on Neural Information Processing Systems (NeurIPS)}, pp. \bibinfo{pages}{793--800}.
\bibitem[{Koriche et~al.(2024)Koriche, Lagniez, Mengel and Tran}]{Koriche.ECML.2024}
\bibinfo{author}{Koriche, F.}, \bibinfo{author}{Lagniez, J.M.}, \bibinfo{author}{Mengel, S.}, \bibinfo{author}{Tran, C.}, \bibinfo{year}{2024}.
\newblock \bibinfo{title}{Learning model agnostic explanations via constraint programming}, in: \bibinfo{booktitle}{Machine Learning and Knowledge Discovery in Databases. Research Track (ECML/PKDD)}, pp. \bibinfo{pages}{437--453}.
\bibitem[{Kozachinskiy(2023)}]{Kozachinskiy.arXiv.2023}
\bibinfo{author}{Kozachinskiy, A.}, \bibinfo{year}{2023}.
\newblock \bibinfo{title}{Inapproximability of sufficient reasons for decision trees}.
\newblock \bibinfo{journal}{CoRR} \bibinfo{volume}{abs/2304.02781}.
\newblock \URLprefix \url{https://doi.org/10.48550/arXiv.2304.02781}, \href{http://arxiv.org/abs/2304.02781}{\tt arXiv:2304.02781}.
\bibitem[{Kyrillidis et~al.(2013)Kyrillidis, Becker, Cevher and Koch}]{Kyrillidis.ICML.2013}
\bibinfo{author}{Kyrillidis, A.}, \bibinfo{author}{Becker, S.}, \bibinfo{author}{Cevher, V.}, \bibinfo{author}{Koch, C.}, \bibinfo{year}{2013}.
\newblock \bibinfo{title}{Sparse projections onto the simplex}, in: \bibinfo{booktitle}{Proceedings of the 30th International Conference on Machine Learning (ICML)}, pp. \bibinfo{pages}{235--243}.
\bibitem[{Lage et~al.(2019)Lage, Chen, He, Narayanan, Kim, Gershman and Doshi{-}Velez}]{Lage.HCOMP.2019}
\bibinfo{author}{Lage, I.}, \bibinfo{author}{Chen, E.}, \bibinfo{author}{He, J.}, \bibinfo{author}{Narayanan, M.}, \bibinfo{author}{Kim, B.}, \bibinfo{author}{Gershman, S.J.}, \bibinfo{author}{Doshi{-}Velez, F.}, \bibinfo{year}{2019}.
\newblock \bibinfo{title}{Human evaluation of models built for interpretability}, in: \bibinfo{booktitle}{Annual {AAAI} Conference on Human Computation and Crowdsourcing (HCOMP)}, pp. \bibinfo{pages}{59--67}.
\bibitem[{Lakkaraju et~al.(2019)Lakkaraju, Kamar, Caruana and Leskovec}]{Lakkaraju.AIES.2019}
\bibinfo{author}{Lakkaraju, H.}, \bibinfo{author}{Kamar, E.}, \bibinfo{author}{Caruana, R.}, \bibinfo{author}{Leskovec, J.}, \bibinfo{year}{2019}.
\newblock \bibinfo{title}{Faithful and customizable explanations of black box models}, in: \bibinfo{booktitle}{Proceedings of the 2019 {AAAI/ACM} Conference on AI, Ethics, and Society (AIES)}, pp. \bibinfo{pages}{131--138}.
\bibitem[{Li et~al.(2021)Li, Nagarajan, Plumb and Talwalkar}]{Li.ICLR.2021}
\bibinfo{author}{Li, J.}, \bibinfo{author}{Nagarajan, V.}, \bibinfo{author}{Plumb, G.}, \bibinfo{author}{Talwalkar, A.}, \bibinfo{year}{2021}.
\newblock \bibinfo{title}{A learning theoretic perspective on local explainability}, in: \bibinfo{booktitle}{Proceedings of the 9th International Conference on Learning Representations (ICLR)}.
\bibitem[{Littman et~al.(1998)Littman, Goldsmith and Mundhenk}]{Littman.JAIR.1998}
\bibinfo{author}{Littman, M.}, \bibinfo{author}{Goldsmith, J.}, \bibinfo{author}{Mundhenk, M.}, \bibinfo{year}{1998}.
\newblock \bibinfo{title}{The computational complexity of probabilistic planning}.
\newblock \bibinfo{journal}{Journal of Artificial Intelligence Research} \bibinfo{volume}{9}, \bibinfo{pages}{1--36}.
\bibitem[{Lundberg and Lee(2017)}]{Lundberg.NeurIPS.2017}
\bibinfo{author}{Lundberg, S.M.}, \bibinfo{author}{Lee, S.}, \bibinfo{year}{2017}.
\newblock \bibinfo{title}{A unified approach to interpreting model predictions}, in: \bibinfo{booktitle}{Advances in Neural Information Processing Systems 30: Proceedings of the Annual Conference on Neural Information Processing Systems (NeurIPS)}, pp. \bibinfo{pages}{4765--4774}.
\bibitem[{Mallows(1957)}]{Mallows.Biometrika.1957}
\bibinfo{author}{Mallows, C.L.}, \bibinfo{year}{1957}.
\newblock \bibinfo{title}{Non-null ranking models. i}.
\newblock \bibinfo{journal}{Biometrika} \bibinfo{volume}{44}, \bibinfo{pages}{114--130}.
\bibitem[{Marden(1996)}]{Marden.Book.1996}
\bibinfo{author}{Marden, J.}, \bibinfo{year}{1996}.
\newblock \bibinfo{title}{Analyzing and modeling rank data}.
\newblock \bibinfo{publisher}{CRC Press}.
\bibitem[{Marques-Silva and Ignatiev(2022)}]{MarquesSilva.AAAI.2022}
\bibinfo{author}{Marques-Silva, J.}, \bibinfo{author}{Ignatiev, A.}, \bibinfo{year}{2022}.
\newblock \bibinfo{title}{Delivering trustworthy {AI} through formal {XAI}}.
\newblock \bibinfo{journal}{Proceedings of the 36th Annual AAAI Conference on Artificial Intelligence} , \bibinfo{pages}{12342--12350}.
\bibitem[{Miller(1956)}]{Miller.PR.1956}
\bibinfo{author}{Miller, G.A.}, \bibinfo{year}{1956}.
\newblock \bibinfo{title}{The magical number seven, plus or minus two: Some limits on our capacity for processing information}.
\newblock \bibinfo{journal}{Psychological Review} \bibinfo{volume}{63}, \bibinfo{pages}{81--97}.
\bibitem[{Molnar(2025)}]{Molnar.Book.2025}
\bibinfo{author}{Molnar, C.}, \bibinfo{year}{2025}.
\newblock \bibinfo{title}{Interpretable Machine Learning: A Guide for Making Black Box Models Explainable}.
\newblock \bibinfo{edition}{3rd} ed., \bibinfo{publisher}{leanpub.com}.
\bibitem[{Mukherjee and Basu(2017)}]{Mukherjee.arXiv.2017}
\bibinfo{author}{Mukherjee, A.}, \bibinfo{author}{Basu, A.}, \bibinfo{year}{2017}.
\newblock \bibinfo{title}{Lower bounds over {B}oolean inputs for deep neural networks with {R}e{LU} gates}.
\newblock \bibinfo{journal}{CoRR} \bibinfo{volume}{abs/1711.03073}.
\newblock \href{http://arxiv.org/abs/1711.03073}{\tt arXiv:1711.03073}.
\bibitem[{Narayanan et~al.(2018)Narayanan, Chen, He, Kim, Gershman and Doshi{-}Velez}]{Narayanan.arXiv.2018}
\bibinfo{author}{Narayanan, M.}, \bibinfo{author}{Chen, E.}, \bibinfo{author}{He, J.}, \bibinfo{author}{Kim, B.}, \bibinfo{author}{Gershman, S.}, \bibinfo{author}{Doshi{-}Velez, F.}, \bibinfo{year}{2018}.
\newblock \bibinfo{title}{How do humans understand explanations from machine learning systems? {A}n evaluation of the human-interpretability of explanation}.
\newblock \bibinfo{journal}{CoRR} \bibinfo{volume}{abs/1802.00682}.
\newblock \href{http://arxiv.org/abs/1802.00682}{\tt arXiv:1802.00682}.
\bibitem[{Natarajan(1995)}]{Natarajan.SJC.1995}
\bibinfo{author}{Natarajan, B.K.}, \bibinfo{year}{1995}.
\newblock \bibinfo{title}{Sparse approximate solutions to linear systems}.
\newblock \bibinfo{journal}{{SIAM} J. Comput.} \bibinfo{volume}{24}, \bibinfo{pages}{227--234}.
\bibitem[{Negahban et~al.(2009)Negahban, Yu, Wainwright and Ravikumar}]{Negahban.NeurIPS.2009}
\bibinfo{author}{Negahban, S.}, \bibinfo{author}{Yu, B.}, \bibinfo{author}{Wainwright, M.}, \bibinfo{author}{Ravikumar, P.}, \bibinfo{year}{2009}.
\newblock \bibinfo{title}{A unified framework for high-dimensional analysis of {M}-estimators with decomposable regularizers}, in: \bibinfo{booktitle}{Advances in Neural Information Processing Systems: Proceedings of the Annual Conference on Neural Information Processing Systems (NeurIPS)}.
\bibitem[{Ordyniak et~al.(2023)Ordyniak, Paesani and Szeider}]{Ordyniak.IJCAI.2023}
\bibinfo{author}{Ordyniak, S.}, \bibinfo{author}{Paesani, G.}, \bibinfo{author}{Szeider, S.}, \bibinfo{year}{2023}.
\newblock \bibinfo{title}{The parameterized complexity of finding concise local explanations}, in: \bibinfo{booktitle}{International Joint Conference on Artificial Intelligence (IJCAI)}, pp. \bibinfo{pages}{3312--3320}.
\bibitem[{Parberry(1996)}]{Parberry.MPNN.1996}
\bibinfo{author}{Parberry, I.}, \bibinfo{year}{1996}.
\newblock \bibinfo{title}{Circuit complexity and feedforward neural networks}, in: \bibinfo{editor}{Smolensky, P.}, \bibinfo{editor}{Mozer, M.}, \bibinfo{editor}{Rumelhart, D.} (Eds.), \bibinfo{booktitle}{Mathematical Perspectives on Neural Networks}, pp. \bibinfo{pages}{85--111}.
\bibitem[{Plumb et~al.(2018)Plumb, Molitor and Talwalkar}]{Plumb.NeurIPS.2018}
\bibinfo{author}{Plumb, G.}, \bibinfo{author}{Molitor, D.}, \bibinfo{author}{Talwalkar, A.}, \bibinfo{year}{2018}.
\newblock \bibinfo{title}{Model agnostic supervised local explanations}, in: \bibinfo{booktitle}{Advances in Neural Information Processing Systems 31. Proceedings of the Annual Conference on Neural Information Processing Systems (NeurIPS)}, pp. \bibinfo{pages}{2520--2529}.
\bibitem[{Ribeiro et~al.(2016)Ribeiro, Singh and Guestrin}]{Ribeiro.KDD.2016}
\bibinfo{author}{Ribeiro, M.T.}, \bibinfo{author}{Singh, S.}, \bibinfo{author}{Guestrin, C.}, \bibinfo{year}{2016}.
\newblock \bibinfo{title}{"why should {I} trust you?": Explaining the predictions of any classifier}, in: \bibinfo{booktitle}{{ACM} {SIGKDD} International Conference on Knowledge Discovery and Data Mining}, pp. \bibinfo{pages}{1135--1144}.
\bibitem[{Ribeiro et~al.(2018)Ribeiro, Singh and Guestrin}]{Ribeiro.AAAI.2018}
\bibinfo{author}{Ribeiro, M.T.}, \bibinfo{author}{Singh, S.}, \bibinfo{author}{Guestrin, C.}, \bibinfo{year}{2018}.
\newblock \bibinfo{title}{Anchors: High-precision model-agnostic explanations}, in: \bibinfo{booktitle}{Annual {AAAI} Conference on Artificial Intelligence (AAAI)}, pp. \bibinfo{pages}{1527--1535}.
\bibitem[{Rifkin and Klautau(2004)}]{Rifkin.JMLR.2004}
\bibinfo{author}{Rifkin, R.}, \bibinfo{author}{Klautau, A.}, \bibinfo{year}{2004}.
\newblock \bibinfo{title}{In defense of one-vs-all classification}.
\newblock \bibinfo{journal}{Journal of Machine Learning Research} \bibinfo{volume}{5}, \bibinfo{pages}{101--141}.
\bibitem[{Shalev-Shwartz and Ben-David(2014)}]{ShalevShwartz.Book.2014}
\bibinfo{author}{Shalev-Shwartz, S.}, \bibinfo{author}{Ben-David, S.}, \bibinfo{year}{2014}.
\newblock \bibinfo{title}{Understanding Machine Learning: From Theory to Algorithms}.
\newblock \bibinfo{publisher}{Cambridge University Press}.
\bibitem[{Shalev{-}Shwartz et~al.(2010)Shalev{-}Shwartz, Srebro and Zhang}]{ShalevShwartz.SJO.2010}
\bibinfo{author}{Shalev{-}Shwartz, S.}, \bibinfo{author}{Srebro, N.}, \bibinfo{author}{Zhang, T.}, \bibinfo{year}{2010}.
\newblock \bibinfo{title}{Trading accuracy for sparsity in optimization problems with sparsity constraints}.
\newblock \bibinfo{journal}{{SIAM} Journal of Optimization} \bibinfo{volume}{20}, \bibinfo{pages}{2807--2832}.
\bibitem[{Slack et~al.(2020)Slack, Hilgard, Jia, Singh and Lakkaraju}]{Slack.AIES.2020}
\bibinfo{author}{Slack, D.}, \bibinfo{author}{Hilgard, S.}, \bibinfo{author}{Jia, E.}, \bibinfo{author}{Singh, S.}, \bibinfo{author}{Lakkaraju, H.}, \bibinfo{year}{2020}.
\newblock \bibinfo{title}{Fooling {LIME} and {SHAP}: Adversarial attacks on post hoc explanation methods}, in: \bibinfo{booktitle}{Proceedings of the AAAI/ACM Conference on AI, Ethics, and Society (AIES)}, p. \bibinfo{pages}{180–186}.
\bibitem[{Subercaseaux et~al.(2025)Subercaseaux, Arenas and Meel}]{Subercaseaux.AAAI.2025}
\bibinfo{author}{Subercaseaux, B.}, \bibinfo{author}{Arenas, M.}, \bibinfo{author}{Meel, K.S.}, \bibinfo{year}{2025}.
\newblock \bibinfo{title}{Probabilistic explanations for linear models}, in: \bibinfo{booktitle}{Annual AAAI Conference on Artificial Intelligence}, pp. \bibinfo{pages}{20655--20662}.
\bibitem[{Suykens and Vandewalle(1999)}]{Suykens.NPL.1999}
\bibinfo{author}{Suykens, J.A.}, \bibinfo{author}{Vandewalle, J.}, \bibinfo{year}{1999}.
\newblock \bibinfo{title}{Least squares support vector machine classifiers}.
\newblock \bibinfo{journal}{Neural processing letters} \bibinfo{volume}{9}, \bibinfo{pages}{293--300}.
\bibitem[{Wainwright(2019)}]{Wainwright.CUP.2019}
\bibinfo{author}{Wainwright, M.}, \bibinfo{year}{2019}.
\newblock \bibinfo{title}{High-Dimensional Statistics: A Non-Asymptotic Viewpoint}.
\newblock \bibinfo{publisher}{Cambridge University Press}.
\bibitem[{W{\"{a}}ldchen et~al.(2021)W{\"{a}}ldchen, MacDonald, Hauch and Kutyniok}]{Waldchen.JAIR.2021}
\bibinfo{author}{W{\"{a}}ldchen, S.}, \bibinfo{author}{MacDonald, J.}, \bibinfo{author}{Hauch, S.}, \bibinfo{author}{Kutyniok, G.}, \bibinfo{year}{2021}.
\newblock \bibinfo{title}{The computational complexity of understanding binary classifier decisions}.
\newblock \bibinfo{journal}{Journal of Artificial Intelligence Research} \bibinfo{volume}{70}, \bibinfo{pages}{351--387}.
\bibitem[{Yuan et~al.(2017)Yuan, Li and Zhang}]{Yuan.JMLR.2017}
\bibinfo{author}{Yuan, X.}, \bibinfo{author}{Li, P.}, \bibinfo{author}{Zhang, T.}, \bibinfo{year}{2017}.
\newblock \bibinfo{title}{Gradient hard thresholding pursuit}.
\newblock \bibinfo{journal}{Journal of Machince Learning Research} \bibinfo{volume}{18}, \bibinfo{pages}{166:1--166:43}.
\bibitem[{Zhao et~al.(2021)Zhao, Huang, Huang, Robu and Flynn}]{Zhao.UAI.2021}
\bibinfo{author}{Zhao, X.}, \bibinfo{author}{Huang, W.}, \bibinfo{author}{Huang, X.}, \bibinfo{author}{Robu, V.}, \bibinfo{author}{Flynn, D.}, \bibinfo{year}{2021}.
\newblock \bibinfo{title}{{BayLIME}: Bayesian local interpretable model-agnostic explanations}, in: \bibinfo{booktitle}{Proceedings of the 37th International Conference on Uncertainty in Artificial Intelligence (UAI)}, pp. \bibinfo{pages}{887--896}.

\end{thebibliography}
\end{document}